\documentclass[twoside,11pt]{article}

\usepackage[preprint]{jmlr2e}
\usepackage{ifpdf}
\usepackage[caption=false]{subfig}
\usepackage{booktabs}
\usepackage{multirow}
\usepackage{rotating}
\usepackage[titlenumbered,ruled]{algorithm2e}
\usepackage{float}
\usepackage{xcolor}

\RequirePackage{comment}
\RequirePackage{graphicx}
\RequirePackage{amsmath,amsfonts}

\newtheorem{assmpt}[theorem]{Assumption}

\newcommand{\mtc}{\mathcal}
\newcommand{\mbf}{\mathbf}

\newcommand{\wt}[1]{{\widetilde{#1}}}
\newcommand{\wh}[1]{{\widehat{#1}}}
\newcommand{\ol}[1]{\overline{#1}}

\newcommand{\tX}{\widetilde{X}}

\newcommand{\ind}[1]{{\mbf{1}_{\{#1\}}}}
\newcommand{\paren}[1]{\left(#1\right)}
\newcommand{\brac}[1]{\left[#1\right]}
\newcommand{\inner}[1]{\left\langle#1\right\rangle}
\newcommand{\norm}[1]{\left\|#1\right\|}
\newcommand{\set}[1]{\left\{#1\right\}}
\newcommand{\abs}[1]{\left\lvert #1 \right\rvert}

\newcommand{\e}[1]{\mbe\brac{#1}}
\newcommand{\ee}[2]{\mbe_{#1}\brac{#2}}
\newcommand{\prob}[1]{\mbp\brac{#1}}

\def\argmin{\mathop{\rm arg\,min}\limits}

\newcommand{\sign}{\mathop{\mathrm{sign}}}

\newcommand{\eps}{\varepsilon}

\def\cB{{\mtc{B}}}
\def\cC{{\mtc{C}}}

\def\cE{{\mtc{E}}}
\def\cF{{\mtc{F}}}

\def\cH{{\mtc{H}}}

\def\cL{{\mtc{L}}}

\def\cR{{\mtc{R}}}

\def\cX{\mathcal{X}}
\def\cY{\mathcal{Y}}

\newcommand{\mbe}{\mathbb{E}}
\newcommand{\mbr}{\mathbb{R}}
\newcommand{\mbp}{\mathbb{P}}

\newcommand{\fP}{{\mathfrak{P}}}
\newcommand{\fK}{{\mathfrak{K}}}

\def\kbound/{{\bf (K-Bounded)}}
\def\kuniv/{{\bf (K-Univ)}}
\def\khoeld/{{\bf (K-H\"{o}lder)}}

\usepackage{lastpage}
\jmlrheading{21}{2020}{1-\pageref{LastPage}}{11/17}{12/20}{17-679}{Gilles Blanchard, Aniket Anand Deshmukh, \"{U}run Dogan, Gyemin Lee and Clayton Scott}
\ShortHeadings{Domain Generalization by Marginal Transfer Learning}{Blanchard, Deshmukh, Dogan, Lee and Scott}
\firstpageno{1}

\begin{document}

\title{Domain Generalization by Marginal Transfer Learning}
%\title{Learning Marginal Predictors: Transfer \\ to an Unlabeled Task}
%\title{Multi-Task Generalization}
%\title{Learning to Classify}
%\title{Covariate Shift with Posterior Drift: A Transfer Learning Approach}
%\title{Distribution-adaptive classification}
%\title{Transfer learning without target labels}
%\title{Generalizing from Several Related Classification Tasks to a New
%Unlabeled Sample}
%\title{learning classifiers that adapt to the test distribution}
%\title{Transductive transfer learning with multiple source tasks}

\author{\name Gilles Blanchard \email blanchard@universite-paris-saclay.fr \\
       \addr Universit\'e Paris-Saclay, CNRS, Inria,
       Laboratoire de math\'ematiques d’Orsay
       %, 91405, Orsay, France.\\
       \AND
       \name Aniket Anand Deshmukh \email aniketde@umich.edu \\
       \addr Microsoft AI \& Research
       \AND
       \name \"{U}run Dogan \email urundogan@gmail.com \\
       \addr Microsoft AI \& Research
       \AND
       \name Gyemin Lee \email gyemin@seoultech.ac.kr \\
       \addr Dept. Electronic and IT Media Engineering\\
      Seoul National University of Science and Technology
      \AND
       \name Clayton Scott \email clayscot@umich.edu \\
       \addr Electrical and Computer Engineering, Statistics\\
       University of Michigan
       }

\editor{Arthur Gretton}

\maketitle

\begin{abstract}%
In the problem of domain generalization (DG), there are labeled training data sets from several related prediction problems, and the goal is to make accurate predictions on future unlabeled data sets that are not known to the learner. This problem arises in several applications where data distributions fluctuate because of environmental, technical, or other sources of variation. We introduce a formal framework for DG, and argue that it can be viewed as a kind of supervised learning problem by augmenting the original feature space with the marginal distribution of feature vectors. While our framework has several connections to conventional analysis of supervised learning algorithms, several unique aspects of DG require new methods of analysis.

This work lays the learning theoretic foundations of domain generalization, building on our earlier conference paper where the problem of DG was introduced \citep{blanchard:11:nips}. We present two formal models of data generation, corresponding notions of risk, and distribution-free generalization error analysis. By focusing our attention on kernel methods, we also provide more quantitative results and a universally consistent algorithm. An efficient implementation is provided for this algorithm, which is experimentally compared to a pooling strategy on one synthetic and three real-world data sets.

%We develop a distribution-free, kernel-based solution to domain generalization. Our approach involves identifying an appropriate reproducing kernel Hilbert space and optimizing a regularized empirical risk over the space.  We present generalization error analysis, describe universal kernels, and establish universal consistency of the proposed methodology. 
\end{abstract}


\section{Introduction}

% Is it possible to leverage the solution of one classification problem to
% solve another?  This is a question that has received increasing attention
% in recent years from the machine learning community, and has been studied
% in a variety of settings, including multi-task learning, domain adaptation, and transfer learning. In this work we study domain generalization, another setting in which
% this question arises, and one that incorporates elements of the three aforementioned
% settings and is motivated by many practical applications.

Domain generalization (DG) is a machine learning problem where the learner has access to labeled training data sets from several related prediction problems, and must generalize to a future  prediction problem for which no labeled data are available.
%to simplify the exposition, we will assume the setting
%of binary classification, $\cY = \{-1,1\}$\,, although the methodology
%and results presented here are valid for general output spaces.
%In particular, let $\cX$ be a feature space and $\cY$ a space of outcomes. 
%The notation $P_{XY}$ refers to a joint distribution on $\cX \times \cY$.
%, we refer to the $X$ marginal distribution $P_X$ as simply the marginal distribution, and the conditional $P_{XY}(Y|X)$ as the posterior distribution. 
In more detail, there are $N$ labeled training data sets
% distributions $P_{XY}^{(i)}$ on $\cX \times \cY$, $i=1,\ldots, N$. For each $i$, there is a training sample 
$S_i = (X_{ij},Y_{ij})_{1 \le j \le n_i}$, $i=1,\ldots, N$, that describe similar but possibly distinct prediction tasks. 
%of iid realizations of $P_{XY}^{(i)}$.
The objective is to learn a rule that takes as input a previously unseen {\em unlabeled} test data set $X_1^T, \ldots, X_{n_T}^T$,
%from a never-before-seen test distribution $P_{XY}^T$, 
and accurately predicts outcomes for these or possibly other unlabeled points drawn from the associated learning task.
% outputs an accurate predictor for a that test distribution.
% There is also a test distribution $P_{XY}^T$ that is similar to but again
% distinct from the training distributions $P_{XY}^{(i)}$. Finally, there is
% a test sample $(X_j^T,Y_j^T)_{1 \le j \le n_T}$ of iid realizations of
% $P_{XY}^T$, but in this case the labels $Y_j^T$ are not observed.
% The goal is to design a predictor for this new test task. 

% If one allows that the learner has access to the unlabeled test data, then the above problem statement is similar to that of multi-source, unsupervised domain adaptation. However, DG differs from domain adaptation for a subtle but important reason. In domain adaptation, $P_{XY}^T$ is view as fixed, and the learned predictor's performance is assessed with respect to the draw of a random $(X^T,Y^T)$ from $P_{XY}^T$. In contrast, in domain generalization, the test task and training tasks are all viewed as realizations from a meta-distribution (i.e., a distribution on distributions), and the goal is to optimize performance with respect to the random draw of both $P_{XY}^T$ and $(X^T,Y^T)$. Thus, DG is fundamentally concerned with generalization to a new task.\footnote{We use the terms ``domain" and ``task"  interchangeably to refer to a joint distribution $P_{XY}$.}

DG arises in several applications. One prominent example is precision medicine, where a common objective is to design a patient-specific classifier (e.g., of health status) based on clinical measurements, such as an electrocardiogram or electroencephalogram. In such measurements, patient-to-patient variation is common, arising from biological variations between patients, or technical or environmental factors influencing data acquisition. Because of patient-to-patient variation, a classifier that is trained on data from one patient may not be well matched to another patient. In this context, domain generalization enables the transfer of knowledge from historical patients (for whom labeled data are available) to a new patient without the need to acquire training labels for that patient. A detailed example in the context of flow cytometry is given below.

We view domain generalization as a conventional supervised learning problem where the original feature space is augmented to include the marginal distribution generating the features. We refer to this reframing of DG as ``marginal transfer learning," because it reflects the fact that in DG, information about the test task must be drawn from that task's marginal feature distribution. Leveraging this perspective, we formulate two statistical frameworks for analyzing DG. The first framework allows the observations within each data set to have arbitrary dependency structure, and makes connections to the literature on Campbell measures and structured prediction. The second framework is a special case of the first, assuming the data points are drawn i.i.d. within each task, and allows for a more refined risk analysis. 

We further develop a distribution-free kernel machine that employs a kernel on the aforementioned augmented feature space. Our methodology is shown to yield a
universally consistent learning procedure under both statistical frameworks,
meaning that the domain generalization risk tends to the best possible value as the relevant sample sizes tend infinity,
%to $N$ and $\{n_i\}$ and $n_T$ tend to infinity, 
with no assumptions on the data generating distributions. Although DG may be viewed as a conventional supervised learning problem on an augmented feature space, the analysis is nontrivial owing to unique aspects of the sampling plans and risks.
%, which makes the training examples $(X_{ij},Y_{ij})_{ij}$ {\em not} iid.
We offer a computationally efficient and freely available implementation of our algorithm, and present a thorough experimental study validating the proposed approach on one synthetic and three real-world data sets, including comparisons to a simple pooling approach.\footnote{Code is available at https://github.com/aniketde/DomainGeneralizationMarginal.}

To our knowledge, the problem of domain generalization was first proposed and studied by our earlier conference publication \citep{blanchard:11:nips} which this work extends in several ways. It adds (1) a new statistical framework, the agnostic generative model described below; (2) generalization error and consistency results for the new statistical model; (3) an extensive literature review; (4) an extension to the regression setting in both theory and experiments; (5) a more general statistical analysis, in particular, we no longer assume a bounded loss, and therefore accommodate common convex losses such as the hinge and logistic losses; (6) extensive experiments (the conference paper considered a single small data set); (7) a scalable implementation based on a novel extension of random Fourier features; and (8) error analysis for the random Fourier features approximation.

%We will refer to this learning problem as {\em
%learning marginal predictors}. A concrete motivating application is given
%in the next section, and others are discussed in Section \ref{sec:disc}.

%Thus, when we say that
%the training and test distributions are ``similar," we mean that there is
%some pattern making it possible to learn a mapping from marginal
%distributions to labels.

\section{Motivating Application: Automatic Gating of Flow Cytometry Data}
\label{sec:gating}

Flow cytometry is a high-throughput measurement platform that is an
important clinical tool for the diagnosis of blood-related
pathologies. This technology allows for quantitative analysis of
individual cells from a given cell population, derived for example from a
blood sample from a patient.  We may think of a flow cytometry
data set as a set of $d$-dimensional attribute vectors $(X_j)_{1 \le j \le
n}$, where $n$ is the number of cells analyzed, and $d$ is the number of
attributes recorded per cell. These attributes pertain to various
physical and chemical properties of the cell.  Thus, a flow cytometry data
set may be viewed as a random sample from a patient-specific distribution.

Now suppose a pathologist needs to analyze a new (test) patient
with data $(X_j^T)_{1 \le j \le n_T}$. Before proceeding, the
pathologist first needs the data set to be ``purified" so that only
cells of a certain type are present. For example, lymphocytes are
known to be relevant for the diagnosis of leukemia, whereas
non-lymphocytes may potentially confound the analysis. In other
words, it is necessary to determine the label $Y_j^T \in \{-1,1\}$
associated to each cell, where $Y_j^T = 1$ indicates that the $j$-th
cell is of the desired type.

In clinical practice this is accomplished through a manual process known
as ``gating."  The data are visualized through a sequence of
two-dimensional scatter plots, where at each stage a line segment or
polygon is manually drawn to eliminate a portion of the unwanted cells.
Because of the variability in flow cytometry data, this process is
difficult to quantify in terms of a small subset of simple rules. Instead,
it requires domain-specific knowledge and iterative refinement. Modern
clinical laboratories routinely see dozens of cases per day, so it is desirable to automate this process.

Since clinical laboratories maintain historical databases, we can assume
access to a number ($N$) of historical (training) patients that have already been
expert-gated.  Because of biological and technical variations in flow
cytometry data, the distributions $P_{XY}^{(i)}$ of the historical
patients will vary. To illustrate the flow cytometry gating problem,
we use the NDD data set from the FlowCap-I challenge.\footnote{We will revisit this data set in Section
\ref{Flow_Exp} where details are given.} For example, Fig. \ref{fig:pb_scatter} shows exemplary
two-dimensional scatter plots for two different patients -- see caption
for details. Despite differences in the two distributions, there are
also general trends that hold for all patients.
%For example, lymphocytes are known to exhibit
%low levels of the ``side-scatter" (SS) attribute, while expressing high levels
%of the attribute CD45 (see column 2 of Fig. \ref{fig:pb_scatter}).
Virtually every cell type of interest has a known tendency (e.g., high or
low) for most measured
attributes. Therefore, it is reasonable to assume that there is an
underlying distribution (on distributions) governing flow cytometry data
sets, that produces roughly similar distributions thereby making possible
the automation of the gating process.

% ------------------------------------------------
\begin{figure*}[th]
%
%\vskip 0.2in
\ifpdf
\centering
\includegraphics[width=\linewidth]
  {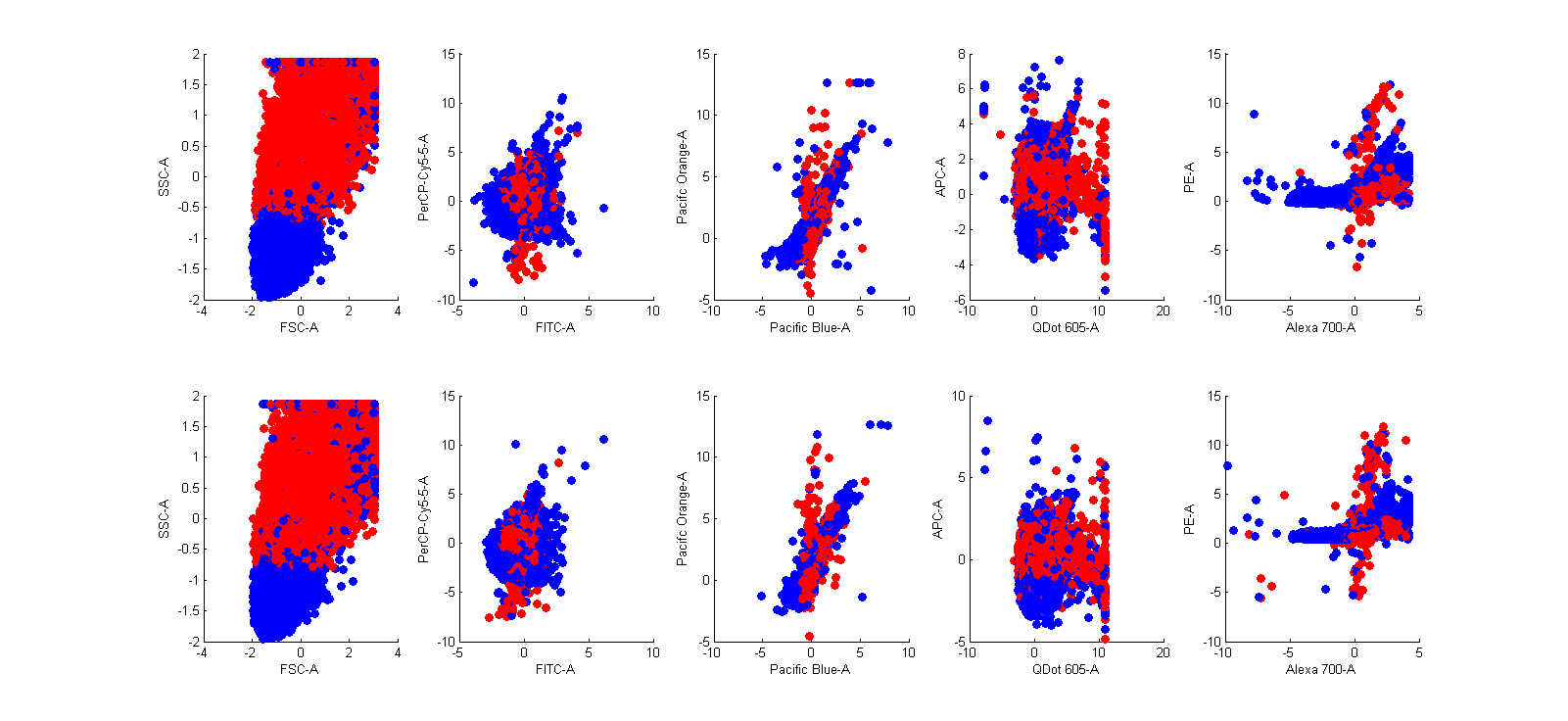}
\fi
\vskip -0.10in
\caption{
Two-dimensional projections of multi-dimensional flow cytometry data.
Each row corresponds to a single patient, and each column to a particular
two-dimensional projection.
The distribution of cells differs from patient to patient. The colors
indicate the results of gating, where a particular type of cell, marked
dark (blue), is separated from all other cells, marked bright (red).
Labels were manually selected by a domain expert.
}
\label{fig:pb_scatter}
%\vskip -0.2in
\end{figure*}
% ------------------------------------------------
\section{Formal Setting and General Results}
\label{sec:formal}

In this section we formally define domain generalization via two possible data generation models together with associated notions of risk. We also provide a basic generalization error bound for the first of these data generation models.

Let $\cX$ denote the observation space (assumed to be a Radon space)
and $\cY \subseteq \mbr$ the output space. Let $\fP_{\cX}$ and $\fP_{\cX\times\cY}$ denote the set of probability distributions on $\cX$ and $\cX \times \cY$, respectively. The spaces $\fP_{\cX}$ and $\fP_{\cX\times\cY}$ are endowed with the topology of weak convergence and the associated Borel $\sigma$-algebras. 

The disintegration theorem for joint probability distributions (see for instance \citealp{Kal02}, Theorem 6.4) tells us that (under suitable regularity properties, satisfied if $\cX$ is a Radon space)
any element $P_{XY} \in \fP_{\cX\times\cY}$ can be written as a
Markov semi-direct product $P_{XY} = P_{X} \bullet P_{Y|X}$, with $P_X \in \fP_{\cX}$,
$P_{Y|X} \in \fP_{Y|X}$, where $\fP_{Y|X}$ is the space of conditional probability distributions of $Y$ given $X$, also called Markov transition kernels from $\cX$ to $\cY$.
This specifically means that
\begin{equation}
    \label{eq:disintegration}
\ee{(X,Y)\sim P_{XY}}{h(X,Y)}  = \int \paren{ \int h(x,y) P_{Y|X}(dy|X=x)}
P_X(dx),
\end{equation}
for any integrable function $h:\cX\times\cY \rightarrow \mbr$. Following common
terminology in the statistical learning literature, we will also call $P_{Y|X}$ the
{\em posterior} distribution (of $Y$ given $X$).

We assume that $N$ training samples $S_i = (X_{ij},Y_{ij})_{1 \leq j \leq n_i}$, $i=1,\ldots,N$, are observed. To allow for possibly unequal sample
sizes $n_i$, it is convenient to formally identify each sample $S_i$ with
its associated empirical distribution $\wh{P}_{XY}^{(i)}=\frac{1}{n_i} \sum_{j=1}^{n_i} \delta_{(X_{ij},Y_{ij})} \in \fP_{\cX\times\cY}$. We assume that the ordering of the observations inside a given sample $S_i$ is
arbitrary and does not contain any relevant information.
We also denote by $\wh{P}^{(i)}_X=\frac{1}{n_i} \sum_{j=1}^{n_i} \delta_{X_{ij}} \in \fP_{\cX}$ the $i$th training sample without labels. 
Similarly, a test sample is denoted by $S^T=(X^T_j,Y^T_j)_{1 \le j \le
n_T}$, and the empirical distribution of the unlabeled data by $\wh{P}^{T}_X$.

\subsection{Data Generation Models}
\label{sec:datagen}

We propose two data generation models. The first is more general, and includes the second as a special case.
%The basic and natural assumption we make concerning the data generation model is the following.
\begin{assmpt}[AGM]
  There exists a distribution $P_S$ on $\fP_{\cX\times\cY}$
  such that $S_1,\ldots,S_N$ are i.i.d. realizations from $P_S$.
\end{assmpt}
We call this the
{\em agnostic generative model}. This is a quite general model
in which samples are assumed to be identically distributed and independent of each other, but nothing particular is assumed
about the generation mechanism for observations inside a given sample, nor for the (random) sample size.

We also introduce a more specific generative mechanism, where  observations
$(X_{ij},Y_{ij})$ inside the sample $S_i$ are themselves i.i.d. from 
$P^{(i)}_{XY}$, a latent unobserved random distribution,
as follows. The symbol $\otimes$ indicates a product measure.
% but 

\begin{assmpt}[2SGM]
  There exists a distribution $\mu$ on
  $\fP_{\cX\times\cY}$ and a distribution $\nu$ on $\mathbb{N}$, such that $(P^{(1)}_{XY},n_1), \ldots , (P^{(N)}_{XY},n_N)$ are i.i.d. realizations from $\mu \otimes \nu $, and conditional to $(P^{(i)}_{XY},n_i)$
  the sample $S_i$ is made of $n_i$ i.i.d. realizations of $(X,Y)$ following the distribution $P^{(i)}_{XY}$. 
\end{assmpt}

This model, called the {\em 2-stage generative model}, is a subcase of {\bf (AGM)}: since
the $(P^{(i)}_{XY},n_i)$ are i.i.d., the samples $S_i$ also are.
This model was the one studied in our conference paper \citep{blanchard:11:nips}. It
has been considered in the distinct but related context of
``learning to learn'' (\citealp{baxter:2000:jair}; see also a more detailed discussion below, Section~\ref{se:ltl}).
%and of distribution regression \citep{szabo2016learning}. % Removed by Clay, this is a slightly different sampling structure for DR
Many of our results will hold for the agnostic generative model, but the two-stage generative model allows for additional developments.

Since in {\bf (2SGM)} we assume that the latent random distribution
of the points in the sample $S_i$ and its size $n_i$ are independent
(which is not necessarily the case for {\bf (AGM)}), 
in this model it becomes
a formally well-defined question to ask how the learning problem evolves if we only change the size of the samples. In other words, we may study the setting where the generating distribution $\mu$ remains fixed, but
their size distribution $\nu$ changes. In particular, this work examines the following different
situations of interest in which the distribution $\mu$ is fixed:
\begin{itemize}
\item The samples all have the same fixed size $n$, i.e. $\nu=\delta_n$;
\item The training samples are subsampled (without replacement) to a fixed size $n$ in order to reduce computational complexity; this reduces to the first setting;
\item Both the training samples $N$ and their size grow. In this case the size distribution
$\nu_N$ depends on $N$ (possibly $\nu_N=\delta_{n(N)}$)
% Gilles: it seems to early to discuss triangular arrays now
%Mathematically, this can
%be understood as a ``triangular array" asymptotic regime, where 
%the distribution of a family of i.i.d. variables (the samples $S_i$)
%changes as their total number increases.
%where increasing  resources allow both sizes to grow.
\end{itemize}
We note that when the distribution of the sample sizes $n_i$ is a Poisson or a mixture of Poisson distributions, the
{\bf (2SGM)} is (a particular case of) what is known as a Cox model or doubly stochastic Poisson process in the point process literature (see, e.g., \citealp{daley2003introductionvol1}, Section~6.2), which is a Poisson process with random
(inhomogeneous) intensity.

%as will be discussed further below.  

\subsection{Decision Functions and Augmented Feature Space}
\label{sec:decfun}

In domain generalization, the learner's goal is to infer from the training data a general rule that takes an arbitrary, previously unseen, unlabeled data set corresponding to a new prediction task, and produces a classifier for that prediction task that could be applied to any $x$ (possibly outside the unlabeled data set). In other words, the learner should output a mapping $g:\fP_{\cX} \to (\cX \to \mbr)$. Equivalently, the learner should output a function $f:\fP_{\cX} \times \cX \to \mbr$,
where the two notations are related via $g(P_X)(x) = f(P_X,x)$. In the latter viewpoint, $f$
may be viewed as a standard decision function on the ``augmented" or ``extended" feature space $\fP_{\cX} \times \cX$, which facilitates connections to standard supervised learning. We refer to this view of DG as {\em marginal transfer learning}, because the information that facilitates generalization to a new task is conveyed entirely through the marginal distribution. In the next two subsections, we present two definitions of the risk of a decision function $f$, one associated to each of the two data generation models.

\subsection{Risk and Generalization Error Bound under the Agnostic Generative Model}

Consider a test sample $S^T=(X^T_j,Y^T_j)_{1 \le j \le n_T}$, whose labels are not observed. 
%A decision function is a function $f: \fP_{\cX} \times \cX \mapsto {\mathbb R}$ that takes into account the full information of the observed test $X$-sample and predicts a label for any given test point $x\in \cX$ (which can belong to the test sample or not).
If $\ell: \mbr \times \cY \mapsto
\mbr_+$ is a loss function for a single prediction,
and predictions of a fixed decision function $f$ on the test sample are given by
$\wh{Y}_j^T = f(\wh{P}^{T}_X,X_j^{T})$,
then the empirical average loss incurred on
the test sample is 
\[
\cL(S^T,f) := \frac{1}{n_T} \sum_{j=1}^{n_T}
\ell(\wh{Y}^T_j,Y^T_j)\,.
\]
Thus, we define the \emph{risk} of a decision function
%of size~$n_T$,
%over test samples,
as the average of the above quantity when test samples
are drawn according to the same mechanism as the training
samples:
\begin{equation*}
%\label{eq:err_agm}
\cE(f) := \mbe_{S^T \sim P_S}\brac{\cL(S^T,f)} = \mbe_{S^T \sim P_S}
\brac{\frac{1}{n_T} \sum_{j=1}^{n_T} \ell(f(\wh{P}^T_X,X_j^T),Y_j^T)}.
\end{equation*}
In a similar way, we define the \emph{empirical risk} of a decision function as its average prediction
error over the training samples:
\begin{equation}
  \label{eq:train_err}
  \wh{\cE}(f,N) := \frac{1}{N}\sum_{i=1}^N \cL(S_i,f) = \frac{1}{N} \sum_{i=1}^N
  \frac{1}{n_i} \sum_{j=1}^{n_i} \ell(f(\wh{P}^{(i)}_X,X_{ij}),Y_{ij}).
\end{equation}
\begin{remark} It is possible to understand the above setting as a particular instance
  of a {\em structured output learning problem} \citep{tsochantaridis2005large,bakir2007predicting}, in which the input variable $X^*$ is
  $\wh{P}^T_X$, and the ``structured output'' $Y^*$ is the collection of labels $(Y_i^T)_{1\leq i \leq n_T}$  (matched to their respective input points). As is generally the case for structured output learning, the nature of the problem and the ``structure'' of the outputs
is very much encoded in the particular form of the loss function. In our setting the loss
function is additive over the labels forming the collection $Y^*$, and we will exploit
this particular form for our method and analysis.
\end{remark}

\begin{remark} \label{rem:campbell} The risk $\cE(f)$ defined above can be described in the following way:
  consider the random variable $\xi:=(\wh{P}_{XY}; (X,Y))$ obtained by first drawing $\wh{P}_{XY}$ according to $P_S$, then, conditional to this, drawing $(X,Y)$ according to $\wh{P}_{XY}$.
  The risk is then the expectation of a certain function of $\xi$
  (namely $F_f(\xi)=\ell(f(\wh{P}_X,X),Y)$). In probability theory literature, the distribution
  of the variable $\xi$ is known as the {\em Campbell measure} associated to the 
  distribution $P_S$ over the measure space $\fP_{\cX \times \cY}$; this object is in particular
  of fundamental use in point process theory (see, e.g., \citealp{daley2008introductionvol2}, Section 13.1). We will denote it by $\cC(P_S)$ here.
  This intriguing connection suggests that more elaborate tools of
  point process literature may find their use to analyze DG when various classical point processes are considered for the generating distribution.
  The Campbell measure will also appear in the Rademacher analysis below.
  %under the
  %generic {\bf (AGM)}, see Section~\ref{toadd}.
%  In this paper, beyond the agnostic generative model, we will only discuss
%  the specific case where $P_S$ is a simple two-stage generative process, introduced next.
\end{remark}

The next result establishes an analogue of classical Rademacher analysis
under the agnostic generative model.

\begin{theorem}[Uniform estimation error control under {\bf (AGM)}]
\label{th:radgenbound}
Let $\cF$ be a class of decision functions $\fP_\cX \times \cX \rightarrow \mbr$.
Assume the following boundedness condition holds:
\begin{equation}
\label{eqn:agmbnd}
  \sup_{f \in \cF} \sup_{P_X \in \fP_\cX} \sup_{(x,y) \in \cX \times \cY} \ell(f(P_X,x),y) \leq B_\ell.
\end{equation}
Under {\bf (AGM)}, if $S_1,\ldots,S_N$ are i.i.d. realizations from $P_S$, then %for any $R>0$,
with probability at least $1-\delta$ with respect to the draws of the training samples: %$\widehat{P}_{XY}^{(i)}$,
\begin{multline}
\label{eq:rademmainbound}
\sup_{f \in \cF} \abs{ \wh{\cE}(f,N) - \cE(f)}\\
\leq \frac{2}{N} \mbe_{(\widehat{P}^{(i)}_{XY};(X_i,Y_i))\sim \cC(P_S)^{\otimes{N}}}
\ee{(\eps_i)_{1\leq i \leq N}}{\sup_{f \in \cF} \abs{\sum_{i=1}^N
\eps_i \ell(f(\widehat{P}^{(i)}_X,X_i),Y_i)}} + B_\ell \sqrt{\frac{\log(\delta^{-1})}{2N}},
%\frac{1}{N} \sum_{i=1}^N \frac1{n_i} \sum_{j=1}^{n_i} \ell(f(\widehat{P}_X^{(i)},X_{ij}),Y_{ij}) -
%\cE(f,n_T)} \\
%\leq \frac{2 R L_\ell B_k B_\fK}{ \sqrt{N}} + B_{\ell} \sqrt{\frac{\log (4/\delta)}{2N}}.
\end{multline}
where $(\eps_i)_{1\leq i \leq N}$ are i.i.d. Rademacher variables, independent from
$(\wh{P}^{(i)}_{XY},(X_i,Y_i))_{1\leq i \leq N}$, and $\cC(P_S)$ is
the Campbell measure on $\fP_{\cX \times \cY} \times (\cX\times\cY)$ associated to $P_S$ (see Remark~\ref{rem:campbell}).
\end{theorem}

\begin{proof}
  Since the $(S_i)_{1 \leq i \leq N}$ are i.i.d., $\sup_{f\in\cF}\abs{\wh{\cE}(f,N) - \cE(f)}$
takes the form of a uniform deviation between average and expected loss over
the function class $\cF$. We can therefore apply standard analysis
(Azuma-McDiarmid inequality followed by Rademacher complexity analysis for a nonnegative bounded loss;
see, e.g., \citealp{kolt01,BarMen02}, Theorem~8) to obtain that with probability at least $1-\delta$
with respect to the draw of the training samples $(S_i)_{1 \leq i \leq N}$:
\[
  \sup_{f \in \cF} \abs{ \wh{\cE}(f,N) - \cE(f)}
  \leq \frac{2}{N} \mbe_{(S_i)_{1\leq i \leq N}}\ee{(\eps_i)_{1\leq i \leq N}}{\sup_{f \in \cF}
    \abs{\sum_{i=1}^N \eps_i \cL(S_i,f)}} + B_\ell \sqrt{\frac{\log(\delta^{-1})}{2N}} \,,
\]
where $(\eps_i)_{1\leq i \leq N}$ are i.i.d. Rademacher variables, independent of $(S_i)_{1 \leq i \leq N}$.

%   (\widehat{P}^{(i)}_{XY};(X_i,Y_i))\sim \cC(P_S)^{\otimes{N}}}
% \ee{(\eps_i)_{1\leq i \leq N}}{\sup_{f \in \cF} \abs{\sum_{i=1}^N
% \eps_i \ell(f(\widehat{P}^{(i)}_X,X_i),Y_i)}} + B_\ell \sqrt{\frac{\log(2/\delta)}{2N}},

%   \]
% we can apply the Azuma-McDiarmid
% inequality to the function
% \begin{align*}
%   G(S_1,\ldots,S_N) = \sup_{f\in\cF} \abs{\wh{\cE}(f,N) - \cE(f)} = 
%   \sup_{f\in\cF} \abs{\frac{1}{N} \sum_{i=1}^N \cL(S_i,f)
%   - \ee{S^T\sim P_S}{\cL(S^T,f)}},
%&\xi((\widehat{P}_{XY}^{(i)})_{1\leq i \leq N}) := \\ 
%& \sup_{f \in \cB_{\ol{k}}(R)}
%\frac{1}{N} \sum_{i=1}^N \paren{ 
%\ee{(X,Y)\sim \widehat{P}_{XY}^{(i)}}{\ell(f(\widehat{P}_X^{(i)},X),Y)}} -
%\mbe_{\widehat{P}_{XY} \sim \widehat{\mu}}\ee{(X,Y)\sim \widehat{P}_{XY}}{\ell(f(\widehat{P}_X,X),Y)}
%\cE(f,n|\widehat{P}_{XY}^{(i)}) - \cE(f,n)}\,,
% \end{align*}
% %(using the same notation as in the original proof of IIb)
% obtaining that with probability $1-\delta$ over the draw of
% $(S_i)_{1\leq i \leq N}$, it holds
% \[
% G-\e{G}
% \leq B_{\ell} \sqrt{\frac{\log (1/\delta)}{2N}}.
% \]
% %{\bf Question:} The original analysis has a $1/\delta$ but I think it's $2/\delta$ because of the absolute value.
% We use Rademacher complexity analysis for bounding $\e{G}$.
We may write
\[
\cL(S_i,f) = \frac1{n_i} \sum_{j=1}^{n_i} \ell(f(\widehat{P}_X^{(i)},X_{ij}),Y_{ij}) = \ee{(X,Y)\sim \widehat{P}_{XY}^{(i)}}{\ell(f(\widehat{P}_X^{(i)},X),Y)};
\]
%and similarly for $\cL(S^T,f)$.
%\[
%\cE(f,n_T) = \mbe_{\widehat{P}_{XY} \sim \widehat{\mu}}\ee{(X,Y)\sim %\widehat{P}_{XY}}{\ell(f(\widehat{P}_X,X),Y)}.
%\]
%Below, we will denote by $(X_i,Y_i)$ a (single) draw from the empirical distribution $\widehat{P}_{XY}^{(i)}$ (and these draws are independent). 
%We denote by $(\eps_i)_{1\leq i \leq N}$ iid Rademacher variables (independent from everything else).
thus, we have
\begin{multline*}
%  \e{G} & = \mbe_{(S_i)_{1\leq i \leq N}\sim P_S^{\otimes N}}\bigg[ \sup_{f \in
%\cF} \bigg|\frac{1}{N} \sum_{i=1}^N \cL(S_i,f) - \ee{S^T\sim P_S}{\cL(S^T,f)}\bigg| \bigg]\\
%&= \mbe_{(\widehat{P}_{XY}^{(i)})_{1\leq i \leq N}\sim P_S^{\otimes N}}\Bigg[\sup_{f \in
%\cF} \bigg(\frac{1}{N} \sum_{i=1}^N
%\ee{(X,Y) \sim \widehat{P}^{(i)}_{XY}}
%{\ell(f(\widehat{P}_X^{(i)},X),Y)} \\
%&\qquad \qquad \qquad \qquad \qquad \qquad - \mbe_{\widehat{P}^T_{XY} \sim P_S} \ee{(X^T,Y^T) \sim \widehat{P}^T_{XY}}
%{\ell(f(\widehat{P}^T_X,X^T),Y^T)}\bigg)\Bigg]\\
\mbe_{(S_i)_{1\leq i \leq N}}\ee{(\eps_i)_{1\leq i \leq N}}{\sup_{f \in \cF}
  \abs{\sum_{i=1}^N \eps_i \cL(S_i,f)}}\\
\begin{aligned}
& =
\mbe_{(\widehat{P}_{XY}^{(i)})_{1\leq i \leq N}} \ee{(\eps_i)_{1\leq i \leq
N}}{\sup_{f \in \cF} \abs{\sum_{i=1}^N \eps_i
\ee{(X_i,Y_i) \sim \widehat{P}^{(i)}_{XY}}
{\ell(f(\widehat{P}^{(i)}_X,X_i),Y_i)}}} \\
& \leq 
\mbe_{(\widehat{P}_{XY}^{(i)})_{1\leq i \leq N}} \mbe_{(X_1,Y_1) \sim \widehat{P}_{XY}^{(1)}, \ldots,
  (X_N,Y_N) \sim \widehat{P}_{XY}^{(N)}}
\ee{(\eps_i)_{1\leq i \leq N}}{\sup_{f \in \cF} \abs{\sum_{i=1}^N
\eps_i
\ell(f(\widehat{P}^{(i)}_X,X_i),Y_i)}}.% \\
%& \leq \frac{2 R L_\ell B_k B_\fK}{ \sqrt{N}}\,.
\end{aligned}
\end{multline*}
%The first inequality uses a standard argument from Rademacher complexity analysis, namely, the introduction of a ghost sample followed by symmetrization.
In the above inequality, the
inner expectation  on the $(X_i,Y_i)$ is pulled outwards by Jensen's inequality and convexity of the supremum.

To obtain the announced estimate, notice that the above expectation is the same as the expectation with respect to the $N$-fold Campbell measure $\cC(P_S)$.
\end{proof}  
%% Gilles: I moved this to the statement itself.
%(see Remark~\ref{rem:campbell}).
%We have obtained a bound for $\sup_{f\in \cF} (\wh{\cE}(f) - \cE{f})$ and can obtain a
%  similar one when exchanging the role of $\wh{\cE}(f)$ and $\cE(f)$, thus replacing
%  $\delta$ by $\delta/2$ we obtain the bound with the absolute value.
%For the last inequality, the innermost expectation is bounded by the standard bound for the Rademacher complexity of a Lipschitz loss function
%on the ball of radius $R$ of $\cH_{\ol{k}}$, the kernel $\ol{k}$ being bounded by
%$B_k^2 B_\fK^2$.

\begin{remark} \label{rem:camprad} The main term in the theorem is just the conventional
Rademacher complexity for the augmented feature space $\fP_\cX \times \cX$ endowed with the Campbell measure
$\cC(P_S)$. It could also be thought of as the Rademacher complexity for the meta-distribution $P_S$.
\end{remark}

%Since kernels have not been referenced up to this point in the proof, we could introduce this concept separately in the main text and give this intermediate result its own theorem -- a result that is agnostic to the particular learning model. It basically says that despite the non-iid nature of the data, Rademacher complexity on the augmented space still bounds the generalization error. Or at least highlight it in the proof sketch. I also think this technique could be useful in other contexts where you have multiple training samples and you are minimizing an average of their empirical risks.

\subsection{Idealized Risk under the 2-stage Generative Model}

\label{se:2gsm_gen}

The additional structure of {\bf (2SGM)} allows us to define a different notion of risk under this model.
% \gillesb{While we will state more results holding under {\bf (AGM)} below, % as we discussed previously one
% advantage of the more specific {\bf (2SGM)} is to allow us to study the effect of the sample sizes $n_i$ (more precisely of their distribution $\nu$)
% while keeping $\mu$ fixed.}
Toward this end, 
% we now formalize a notion of risk associated to {\bf (2SGM)}. Let
let $P_{XY}^T$ denote the testing data distribution, $P_X^T$ the marginal $X$-distribution of $P_{XY}^T$, $n_T$ the test sample size, $S^T=(X_i^T,Y_i^T)_{1 \le i \le n_T}$ the testing sample, and $\widehat{P}_X^T$ the empirical $X$-distribution. Parallel to the training data generating mechanism under {\bf (2SGM)}, we assume that $P^T_{XY}$ is drawn according to $\mu$.

We first define the risk of any $f:\fP_{\cX} \times \cX \to \mbr$, conditioned on a test sample of size $n_T$, to be
\begin{equation}
  \label{eq:err_nt}
 \cE(f|n_T) := \mbe_{P^T_{XY} \sim \mu} \ee{(X_i^T,Y_i^T)_{1\leq i \leq n_T} \sim (P^T_{XY})^{\otimes n_T}}{\frac{1}{n_T} \sum_{i=1}^{n_T} \ell(f(\wh{P}_X^T,X_i^T),Y_i^T)}.
\end{equation}
In this definition, the test sample $S^T$ consist of $n_T$ iid draws from and random $P_{XY}^T$ drawn from $\mu$. This conditional risk may be viewed as the previously defined risk for {\bf (AGM)} specialized to {\bf (2SGM)}, where $\nu = \delta_{n_T}$.
%\gillesb{(Formally, %under {\bf (2SGM)},
%to say that a test sample size $n_T$ is fixed can be
%interpreted either as taking $\nu = \delta_{n_T}$, or to %conditioning on the test sample size being $n_T$.)}

%Toward this end, we now introduce a notion of risk under {\bf (2SGM)}}.
%% Gilles: I removed the sentence below since it became redundant with p.5 
%In particular, for the purpose of reducing computational complexity,  we can
%analyze the effect of subsampling observations inside a given sample. Such an analysis would not be possible under {\bf (AGM)}.

%\gillesb{Let $P_{XY}^T$ denote the testing data distribution, $P_X^T$ the marginal $X$-distribution of $P_{XY}^T$, $n_T$ the test sample size, $S^T=(X_i^T,Y_i^T)_{1 \le i \le n_T}$ the testing sample, and $\widehat{P}_X^T$ the empirical $X$-distribution. Parallel to the training data generating mechanism under {\bf (2SGM)}, we assume that $(P^T_{XY},n_T)$ is drawn according to $\mu \otimes \nu$, and conditional to this the test sample $S^T$ consists of $n_T$ i.i.d. realizations of $P^T_{XY}$.} 

% Clay: "previously" we discussed evolution with training sample size, not testing.
% As discussed previously, it is meaningful to consider how the test error evolves
% if the test sample size $n_T$ grows  while the distribution
% $\mu$ is fixed. Formally, to say that a test sample size $n_T$ is fixed corresponds to taking $\nu = \delta_{n_T}$, or to conditioning on the test sample size being $n_T$.}

%Note that under {\bf (2SGM)}, the marginal distribution $P^{T}_{X}$ is not %observed and is therefore a latent variable at test time. 

We are particularly interested in the idealized situation where the test sample size $n_T$ grows to infinity.
By the law of large numbers, as $n_T$ grows, $\wh{P}^T_X$  converges to $P^T_X$ (in the sense of weak convergence). This motivates the introduction of the following \emph{idealized risk} which assumes access to an infinite test sample, and thus to the true marginal $P^T_X$:
\begin{equation}
\label{eq:err_infty}
\cE^\infty(f) := \mbe_{P^T_{XY} \sim \mu} \ee{(X^T,Y^T) \sim
P^T_{XY}}
{\ell(f(P^T_X,X^T),Y^T)}.
\end{equation}
Note that both notions of risk in \eqref{eq:err_nt} and \eqref{eq:err_infty} depends on $\mu$, but not on $\nu$.
%and thus only makes sense under {\bf (2SGM)}. 

The following proposition more precisely motivates viewing $\cE^\infty(f)$ as a limiting case of $\cE(f|n_T)$.
%First, under {\bf (2SGM)}, introduce the following risk conditional to a finite test sample size $n_T$:
%\begin{equation}
% \label{eq:err_nt}
% \cE(f|n_T) = \mbe_{P^T_{XY} \sim \mu} \ee{(X_i^T,Y_i^T)_{1\leq i \leq n_T} %\sim (P^T_{XY})^{\otimes n_T}}{\frac{1}{n_T} \sum_{i=1}^{n_T} %\ell(f(\wh{P}_X^T,X_i^T),Y_i^T)}.
%end{equation}
\begin{proposition}
  \label{prop:conv_err_2gsm}
  Assume $\ell$ is a bounded, $L$-Lipschitz loss function and $f: \fP_{\cX} \times \cX \rightarrow \mbr$ is a fixed decision function which is continuous with
  respect to both its arguments (recalling $\fP_{\cX}$ is endowed with the weak convergence topology). Then
  it holds under {\bf (2SGM)}:
  \[
    \lim_{n_T \rightarrow \infty} \cE(f|n_T) = \cE^{\infty}(f).
  \]
\end{proposition}

\begin{remark}
This result provides one setting where the risk $\cE^\infty$ is clearly motivated as the goal of asymptotic analysis when $n_T \to \infty$. Although Proposition \ref{prop:conv_err_2gsm} is not used elsewhere in this work, a more quantitative version of this result is stated below for kernels (see Theorem \ref{th:tworisks}), where convergence holds uniformly and the assumption of a bounded loss is dropped.
\end{remark}

%The minimum possible value of this risk is the Bayes risk for domain generalization. Our goal is a learning rule that asymptotically predicts as well as the
%global minimizer of \eqref{eq:err_infty}, for a {\em general} loss $\ell$.

To gain more insight into the idealized risk $\cE^\infty$,
% Let $\fP_{\cX\times\cY}$ denote the set of probability
% distributions on $\cX\times\cY$, $\fP_{\cX}$ the
% set of probability distributions on $\cX$ (which we call ``marginals''), and $\fP_{\cY|\cX}$
% the set of conditional probabilities of $Y$ given $X$
% (also known as Markov transition kernels from $X$ to $Y$,
% which we also call posteriors\footnote{This term comes from viewing $Y$ as an unknown parameter to be estimated given the observation $X$.}).
%%%%%%%%%
%recall that the disintegration theorem (see for instance \citet{Kal02}, Theorem 6.4) tells us that (under
%suitable regularity properties, e.g., $\cX$ is a Radon space)
%any element $P_{XY} \in \fP_{\cX\times\cY}$ can be written as a
%product $P_{XY} = P_{X} \bullet P_{Y|X}$, with $P_X \in \fP_{\cX}$,
% $P_{Y|X} \in \fP_{Y|X}$, that is to say,
%\[
%\ee{(X,Y)\sim P_{XY}}{h(X,Y)}  = \int \paren{ \int h(x,y) P_{Y|X}(dy|X=x)}
%P_X(dx),
%\]
%for any integrable function $h:\cX\times\cY \rightarrow \mbr$.
%The space $\fP_{\cX\times\cY}$ is endowed with the topology of weak
%convergence and the associated Borel $\sigma$-algebra.
%
%\footnote{The manipulations in this paragraph and the next are not used in our subsequent analysis, so we do not fully justify them in the interest of readability.}
recalling the standard decomposition~\eqref{eq:disintegration} of
$P_{XY}$ into the marginal $P_X$ and the posterior $P_{Y|X}$, we observe that we can apply the disintegration theorem not only to any $P_{XY}$, but also to $\mu$, and thus
decompose it into two parts, $\mu_X$ which
generates the marginal distribution $P_X$, and $\mu_{Y|X}$ which,
conditioned on $P_X$, generates the posterior $P_{Y|X}$. (More precise notation might be $\mu_{P_X}$ instead of $\mu_X$ and $\mu_{P_{Y|X}|P_X}$ instead of $\mu_{Y|X}$, but this is rather cumbersome.) Denote $\tX =
(P_X,X)$. We then have, using Fubini's theorem,
\begin{align*}
\cE^\infty(f) & = \mbe_{P_{X} \sim \mu_X} \mbe_{P_{Y|X} \sim \mu_{Y|X}}
\mbe_{X \sim P_X} \ee{Y \sim P_{Y|X}}{\ell(f(\tX),Y)} \\
& = \mbe_{P_{X} \sim \mu_X} \mbe_{X \sim P_X} \mbe_{P_{Y|X} \sim
\mu_{Y|X}} \ee{Y \sim P_{Y|X}}{\ell(f(\tX),Y)} \\
& = \ee{(\tX,Y) \sim Q^\mu}{\ell(f(\tX),Y)}.
\end{align*}
% \begin{align*}
% \cE(f,\infty) & = \mbe_{P_{X} \sim \mu_X} \mbe_{P_{Y|X} \sim \mu_{Y|X}}
% \mbe_{X \sim P_X} \ee{Y|X \sim P_{Y|X}}{\ell(f(\tX),Y)} \\
% & = \mbe_{P_{X} \sim \mu_X} \mbe_{X \sim P_X} \mbe_{P_{Y|X} \sim
% \mu_{Y|X}} \ee{Y|X \sim P_{Y|X}}{\ell(f(\tX),Y)} \\
% & = \ee{(\tX,Y) \sim Q^\mu}{\ell(f(\tX),Y)}.
% \end{align*}
Here $Q^\mu$ is the distribution that generates $\tX$ by first drawing
$P_X$ according to $\mu_X$, and then drawing $X$ according to $P_X$;
similarly, $Y$ is generated, conditioned on $\tX$, by first drawing
$P_{Y|X}$ according to $\mu_{Y|X}$, and then drawing $Y$ from $P_{Y|X}$.
(The distribution of $\wt{X}$ again takes the form of a Campbell measure, see Remark~\ref{rem:campbell}.)

From the previous expression, we see that the risk $\cE^\infty$ is like a standard
supervised learning risk based on $(\tX,Y) \sim Q^\mu$.  Thus, we can
deduce properties that are known to hold for supervised learning risks.
For example, in the binary classification setting, if the loss is the 0/1
loss, then
$f^*(\tX) = 2 \tilde{\eta}(\tX) - 1$ is an optimal predictor, where
$\tilde{\eta}(\tX) = \ee{Y \sim Q^\mu_{Y|\tX}}{\ind{Y=1}}$,
and
\[
\cE^\infty(f) - \cE^\infty(f^*) = \ee{\tX \sim
Q^\mu_{\tX}}{\ind{\sign(f(\tX)) \ne \sign(f^*(\tX))}|2 \tilde{\eta}(\tX)
- 1|}.
\]
Furthermore, consistency in the sense of $\cE^\infty$ with respect to a general loss $\ell$
(thought of as a surrogate) will imply consistency for the 0/1 loss,
provided $\ell$ is classification calibrated \citep{bartlett06}.

%Despite the similarity to standard supervised learning in the infinite sample
%case, we emphasize that the learning task here is different, because the
%realizations $(\wt{X}_{ij},Y_{ij})$ are neither independent nor
%identically distributed. This is because for fixed $i$, the $\wt{X}_{ij}$ have the same  marginal $P_X^{(i)}$. As discussed later in the paper, this complicates theoretical analysis and necessitates new proof techniques.

%The main ``covariate shift'' assumption is the following.
%Let $P_{XY} \in \fP_{\cX \times \cY}$ be drawn according to $\mu$ and
%$P_{XY} = P_X \bullet P_{Y|X}$ its decomposition. There exists a
%deterministic function $F: \fP_{\cX} \rightarrow \fP_{Y|X}$ so that
%$P_{Y|X} = F(P_X)$ $\mu$-a.s.
%This formalizes the idea that the conditional distribution
%of the labels is entirely determined by the marginal distribution of the
%$X's$.

For a given loss $\ell$, the optimal or Bayes $\cE^\infty$-risk in DG is in general larger than the expected Bayes risk under the (random) test sample generating distribution $P_{XY}^T$,
%for an arbitrary target distribution,
because it is typically not possible to fully determine the Bayes-optimal predictor from only the marginal $P_X^T$. There is, however, a
condition where for $\mu$-almost all test
distributions $P_{XY}^T$, the decision function $f^*(P^T_X,.)$
(where $f^*$ is a
global minimizer of Equation \ref{eq:err_infty}) coincides with an optimal Bayes
decision function for $P_{XY}^T$. This condition is simply that the posterior
$P_{Y|X}$ is ($\mu$-almost surely) a function of $P_X$ (in other words, 
with the notation introduced above, $\mu_{Y|X}(P_X)$ is a Dirac
measure for $\mu$-almost all $P_X$). Although we will {\em not} be assuming this
condition throughout the paper under {\bf (2SGM)}, observe that it is implicitly assumed in
the motivating application presented in Section~\ref{sec:gating}, where an
expert labels the data points by just looking at their marginal
distribution.

\begin{lemma}
\label{lem:determ}
For a fixed distribution $P_{XY}$, and a decision function
$g:\cX\rightarrow \mbr$, let us denote $\cR(g,P_{XY}) =
\ee{(X,Y)\sim P_{XY}}{\ell(g(X),Y)}$ and
\[
\cR^*(P_{XY}) := \min_{g: \cX \rightarrow \mbr} \cR(g,P_{XY}) =
\min_{g: \cX \rightarrow \mbr} \ee{(X,Y)\sim P_{XY}}{\ell(g(X),Y)}
\]
the corresponding optimal (Bayes) risk for the loss function $\ell$ under data distribution $P_{XY}$.
Then
under ~{\bf (2SGM)}:
\[
\cE^\infty(f^*) \ge \ee{P_{XY} \sim \mu}{\cR^*(P_{XY})},
\]
where $f^*:\fP_\cX \times \cX \rightarrow \mbr$ is a minimizer of the idealized DG risk $\cE^\infty$
defined in~\eqref{eq:err_infty}. 

Furthermore, if $\mu$ is a distribution on $\fP_{\cX \times \cY}$ such that
$\mu$-a.s. it holds $P_{Y|X}= F(P_X)$ for some deterministic mapping $F$, then for $\mu$-almost all $P_{XY}$:
\[
\cR(f^*(P_X,.),P_{XY}) = \cR^*(P_{XY})
\]
and
\[
\cE^\infty(f^*) %:= \min_{f: \fP_\cX \times \cX \rightarrow \mbr} \cE(f,\infty)
= \ee{P_{XY} \sim \mu}{\cR^*(P_{XY})}.
\]
\end{lemma}

\begin{proof}
For any $f:\fP_{\cX}\times \cX \rightarrow \mbr$,
one has for all $P_{XY}$: $\cR(f(P_X,.),P_{XY}) \geq \cR^*(P_{XY})$. Taking expectation with respect to $P_{XY}$ establishes the first claim.
Now for any fixed $P_X \in \fP_{\cX}$,
consider $P_{XY} := P_X \bullet F(P_X)$ and $g^*(P_X)$ a Bayes decision
function
for this joint distribution. Pose $f(P_X,x) := g^*(P_X)(x)$. Then $f$ coincides
for $\mu$-almost all $P_{XY}$ with a Bayes decision function for $P_{XY}$,
achieving
equality in the above inequality. The second equality follows by taking
expectation over $P_{XY}\sim\mu$.
\end{proof}

Under {\bf (2SGM)}, we will establish that our proposed learning algorithm is
$\cE^\infty$-consistent, provided the average sample size grows to infinity as well as the total number of samples. Thus, the above result provides a condition on $\mu$ under which it is possible to asymptotically attain the (classical, single-task) Bayes risk on any test distribution although {\em no labels from this test
distribution are observed}.

More generally, and speaking informally, if $\mu$ is such that $ P_{Y|X} $ is close to being a function of $P_X$ in some sense, we can expect the Bayes $\cE^\infty$-risk for domain generalization to be close to the expected 
(classical single-task) Bayes risk for a random test distribution. We reiterate, however, that we make no assumptions on $\mu$ in this work so that the two quantities may be far apart. In the worst case, the posterior may be independent of the marginal, in which case a method for domain generalization will do no better than the na\"ive pooling strategy. For further discussion, see the comparison of domain adaptation and domain generalization in the next section.

\section{Related Work}

Since at least the 1990s, machine learning researchers have investigated the possibility of solving one learning problem by leveraging data from one or more related problems. In this section, we provide an overview of such problems and their relation to domain generalization, while also reviewing prior work on DG.

Two critical terms are {\em domain} and {\em task}. Use of these terms is not consistent throughout the literature, but at a minimum, the domain of a learning problem describes the input (feature) space $\cX$ and marginal distribution of $X$, while the task describes the output space $\cY$ and the conditional distribution of $Y$ given $X$ (also called posterior). In many settings, however, the sets $\cX$ and $\cY$ are the same for all learning problems, and the terms ``domain" and ``task" are used interchangeably to refer to a joint distribution $P_{XY}$ on $\cX \times \cY$. This is the perspective adopted in this work, as well as in much of the work on multi-task learning, domain adaptation (DA), and domain generalization. 

Multi-task learning is similar to DG, except only the training tasks are of interest, and the goal is to leverage the similarity among distributions to improve the learning of individual predictors for each task \citep{caruana:97:ml,evgeniou05,yang2009heterogeneous}.  In contrast, in DG, we are concerned with generalization to a new task.

% keep this paragraph?
%Transfer learning is typically thought of as a generalization of domain adaptation in which the distinct meanings of ``domain" and ``task" (mentioned above) are now taken \citep{pan:10:kde}. For example, the target task may involve classes that were not seen in the source data, and so the set $\cY$ (and hence, the task) changes from one problem to another. Alternatively, the source and target learning problems may have distinct feature spaces (domains). It is in the latter sense that both the fundamental learning problem arising in DG, and our solution to it, can be described as ``marginal transfer learning." As discussed in the previous section, DG can be viewed as a conventional supervised learning problem on a new space where the original feature is augmented with its governing marginal distribution. Our kernel-based algorithm is based on this perspective.

\subsection{Domain Generalization vs. Domain Adaptation}

Domain adaptation refers to the setting in which there is a specific target task and one or more source tasks. The goal is to design a predictor for the target task, for which there are typically few to no labeled training examples, by leveraging labeled training data from the source task(s).

Formulations of domain adaptation may take several forms, depending on the number of sources and whether there are any labeled examples from the target to supplement the unlabeled examples. In multi-source, unsupervised domain adaptation, the learner is presented with labeled training data from several source distributions, and unlabeled data from a target marginal distribution (see \citet{zhang2015multi} and references therein). Thus, the available data are the same as in domain generalization, and algorithms for one of these problems may be applied to the other. 

In all forms of DA, the goal is to attain optimal performance with respect to the joint distribution of the target domain. For example, if the performance measure is a risk, the goal is to attain the Bayes risk for the target domain. To achieve this goal, it is necessary to make assumptions about how the source and target distributions are related \citep{Quionero-Candela:2009:DSM:1462129}. For example, several works adopt the covariate shift assumption, which requires the source and target domains to have the same posterior, allowing the marginals to differ arbitrarily \citep{zadrozny04, huang07, cortes08, sugiyama08kliep, bickel09, kanamori09, yu12, bendavid12alt}. Another common assumption is target shift, which stipulates that the source and target have the same class-conditional distributions, allowing the prior class probability to change \citep{hall81, titterington83, saerens01, storkey09, plessis12, sanderson14, azizzadenesheli2018regularized}. \citet{mansour2009domain,zhang2015multi} assume that the target posterior is a weighted combination of source posteriors, while \citet{zhang2013domain,gong2016domain} extend target shift by also allowing the class-conditional distributions to undergo a location-scale shift, and \citet{tasche17jmlr} assumes the ratio of class-conditional distributions is unchanged. Work on classification with label noise assumes the source data are obtained from the target distribution but the labels have been corrupted in either a label-dependent \citep{blanchard16ejs, natarajan18jmlr, rooyen18jmlr} or feature-dependent \citep{menon18,cannings18,scott19alt} way. Finally, there are several works that assume the existence of a predictor that achieves good performance on both source and target domains \citep{bendavid07nips,bendavid10ml,blitzer08nips,mansour09colt,cortes15kdd,germain16icml}.

The key difference between DG and DA may be found in the performance measures optimized. In DG, the goal is to design a single predictor $f(P_X,x)$ that can apply to any future task, and risk is assessed with respect to the draw of both a new task, and (under {\bf 2SGM}) a new data point from that task. This is in contrast to DA, where the target distribution is typically considered fixed, and the goal is to design a predictor $f(x)$ where, in assessing the risk, the only randomness is in the draw of a new sample from the target task. This difference in performance measures for DG and DA has an interesting consequence for analysis. As we will show, it is possible to attain optimal risk (asymptotically) in DG without making any distributional assumptions like those described above for DA. Of course, this optimal risk is typically larger than the Bayes risk for any particular target domain (see Lemma \ref{lem:determ}). An interesting question for future research is whether it is possible to close or eliminate this gap (between DG and expected DA risks) by imposing distributional assumptions like those for DA.

Another difference between DA and DG lies in whether the learning algorithm must be rerun for each new test data set. Most unsupervised DA methods employ the unlabeled target data for training and thus, when a new unlabeled target data set is presented, the learning algorithm must be rerun. In contrast, most existing DG methods do not assume access to the unlabeled test data at learning time, and are capable of making predictions as new unlabeled data sets arrive without any further training. %That said, one could imagine a setting where a DG learning algorithm has access to the unlabeled test data, in which case this distinction vanishes.

%\subsection{Inductive Bias Learning}
\subsection{Domain Generalization vs. Learning to Learn}
\label{se:ltl}

In the problem of learning to learn (LTL, \citealp{thrun:96:nips}), which has also been called bias learning, meta-learning, and (typically in an online setting) lifelong learning, there are labeled data sets for several tasks, as in DG. There is also a given family of learning algorithms, and the objective is to design a meta-learner that selects the learning algorithm that will perform best on future tasks. The learning theoretic study of LTL traces to the work of \citet{baxter:2000:jair}, who was the first to propose a distribution on tasks, which he calls an ``environment," and which coincides with our $\mu$. Given this setting, the performance of the learning algorithm selected by a meta-learner is obtained by drawing a new task at random, drawing a labeled training data set from that task, running the selected algorithm, drawing a test point, and evaluating the expected loss, where the expectation is with respect to all sources of randomness (new task, training data from new task, test point from new task). 

Baxter analyzes learning algorithms given by usual empirical risk minimization over a hypothesis (prediction function) class, and the goal of the meta-learner is then to select a hypothesis class from a family of such classes. He shows that it is possible to find a good trade-off between the complexity of a hypothesis class and its approximation capabilities for tasks sampled from $\mu$, in an average sense. In particular, the information gained by finding a well-adapted hypothesis class can lead to significantly improved sample efficiency when learning a new task. See \citet{maurer:2009:ml} for further discussion of the results of \citet{baxter:2000:jair}.

Later work on LTL establishes similar results that quantify the ability of a meta-learner to transfer knowledge to a new task. These meta-learners all optimize a particular structure that defines a learning algorithm, such as a feature representation \citep{maurer:2009:ml, Maurer:2016:jmlr,pontil18}, a prior on predictors in a PAC-Bayesian setting \citep{lampert:2014:icml}, a dictionary \citep{maurer:2013:icml}, the bias of a regularizer \citep{denevi18}, and a pretrained neural network \citep{finn17}. It is also worth noting that some algorithms on multi-task learning extract structures that characterize an environment and can be applied to LTL. 

%Recent references on LTL include \cite{finn17, pontil19}.

Although DG and LTL both involve generalization to a new task, they are clearly different problems because LTL assumes access to labeled data from the new task, whereas DG only sees unlabeled data and requires no additional learning. In LTL, the learner can achieve the Bayes risk for the new task, the only issue is the sample complexity. DG is thus a more challenging problem, but also potentially more useful since in many transfer learning settings, labeled data for the new task are unavailable.

%{\bf Work in a way to cite Maurer's 2005 JMLR paper if possible}

% The primary distinction with DG once again stems from a difference in how generalization error is defined. In bias learning, the goal is to minimize the expected Bayes risk of a new task, whereas in DG, it is to minimize the expected risk of any predictor that knows the novel task only through its $X$-marginal. The optimal generalization error in DG is typically larger than the expected Bayes risk for a random target distribution (see proof of Lemma \ref{lem:determ}), because it is typically not possible to fully know the Bayes-optimal predictor from the marginal.

% the learner may restrict their attention to hypotheses in the chosen class. This greatly reduces the sample complexity of learning compared to the conventional setting (without training tasks) in which the learner must consider the entire family of hypothesis classes. The training tasks are thus used to learn the ``inductive bias," that is, the appropriate hypothesis class for the given distribution on distributions.

% The differences with respect to DG are that Baxter neither assumes access to an unlabeled data set, nor considers the DG generalization error. Instead, 

\subsection{Prior Work on Domain Generalization}

To our knowledge, the first paper to consider domain generalization (as formulated in Section \ref{sec:decfun}) was our earlier conference paper \citep{blanchard:11:nips}. The term ``domain generalization" was coined by \citet{muandet:2013:icml}, who study the same setting and build upon our work by extracting features that facilitate DG. \citet{carbonell:2013:ml} study an active learning variant of DG in the realizable setting, and directly learn the task sampling distribution. 
%Very recently, our kernel-based approach was extended to the setting of multiclass classification \cite{deshmukh19}.

Other methods for DG were studied by
\citet{khosla12, xu2014exploiting, grubinger:2015:iwann, ghifary15, Gan_2016_CVPR, ghifary:2017:pami, motiian2017unified, li2017deeper, li2018learning, li18adversarial, li18aaai, li2018deep, balaji18neurips, Ding2018DeepDG, shankar2018generalizing, hu19uai, dou19neurips, Carlucci2019DomainGB, wang2019learning,  Akuzawa2019AdversarialIF}.
Many of these methods learn a common feature space for all tasks. Such methods are complementary to the method that we study. Indeed, our kernel-based learning algorithm may be applied after having learned a feature representation by another method, as was done by \citet{muandet:2013:icml}. 
%In many applications, learning a feature representation will lead to improved empirical performance compared to applying our kernel based-approach on the original input space, as the effective feature space $\fP_{\cX} \times \cX$ is typically infinite-dimensional. 
Since our interest is primarily theoretical, we restrict our experimental comparison to another algorithm that also operates directly on the original input space, namely, a simple pooling algorithm that lumps all training tasks into a single data set and trains a single support vector machine.

\section{Learning Algorithm}
\label{sec:alg}

%\subsection{General formulation}

%The goal is to learn from the samples $S_1,\ldots,S_N$ either the function
%$F$
%or some relevant decision function depending on $P_{Y|X}=F(P_X)$.

In this section, we introduce a concrete algorithm to tackle the learning problem
exposed in Section~\ref{sec:formal}, using an approach based on kernels.  The function $k:\Omega \times
\Omega \to \mbr$ is called a {\em kernel} on $\Omega$ if the matrix
$(k(x_i,x_j))_{1 \le i,j \le n}$ is symmetric and  positive semi-definite for all
positive integers $n$ and all $x_1, \ldots, x_n \in \Omega$. It is well known that every kernel $k$ on $\Omega$ is associated to a space of functions $f:\Omega \to \mathbb{R}$ called the reproducing kernel Hilbert space (RKHS) ${\cal H}_k$ with kernel $k$. One way to envision ${\cal H}_k$ is as follows. Define $\Phi(x) := k(\cdot,x)$, which is called the {\em
canonical feature map} associated with $k$. Then the span of $\{\Phi(x) \, : \, x \in \Omega\}$, endowed
with the inner product $\inner{\Phi(x),\Phi(x')}=k(x,x')$, is dense in ${\cal H}_k$. We also recall the
{\em reproducing property}, which states that $\langle f, \Phi(x) \rangle = f(x)$ for all $f \in {\cal H}_k$ and $x \in \Omega$. 

% It is
% well-known that if $k$ is a kernel on $\Omega$, then there exists a
% Hilbert space $\tilde{{\cal H}}$ and $\tilde{\Phi}: \Omega \to
% \tilde{{\cal H}}$ such that $k(x,x') = \langle \tilde{\Phi}(x),
% \tilde{\Phi}(x') \rangle_{\tilde{{\cal H}}}$. While $\tilde{{\cal H}}$ and
% $\tilde{\Phi}$ are not uniquely determined by $k$, the Hilbert space of
% functions (from $\Omega$ to $\mbr$) ${\cal H}_k = \{\langle v, \tilde{\Phi}( \cdot )
% \rangle_{\tilde{{\cal H}}} : v \in \tilde{{\cal H}}\}$ is uniquely
% determined by $k$, and is called the reproducing kernel Hilbert space
% (RKHS) of $k$.

For later use, we introduce the notion of a {\em universal} kernel.
A kernel $k$ on a compact metric space $\Omega$ is said to be {\em
universal} when its RKHS is dense in $\cC(\Omega)$, the set of continuous
functions on $\Omega$, with respect to the supremum norm. Universal
kernels are important for establishing universal consistency of many
learning algorithms.  See \citet{steinwart08} for background on kernels and reproducing kernel Hilbert spaces.

Several well-known learning algorithms, such as support vector machines
and kernel ridge regression, may be viewed as minimizers of a
norm-regularized empirical risk over the RKHS of a kernel. A similar
development has also been made for multi-task learning \citep{evgeniou05}.
Inspired by this framework, we consider a general kernel-based algorithm as
follows.

Consider the loss function $\ell:\mbr\times\cY \rightarrow \mbr_+$.
Let $\ol{k}$ be a kernel on $\fP_{\cX} \times \cX$, and let $\cH_{\ol{k}}$
be the associated RKHS. For the
sample $S_i$, recall that $\wh{P}^{(i)}_X=\frac{1}{n_i} \sum_{j=1}^{n_i} \delta_{X_{ij}}$
denotes the corresponding empirical
$X$ distribution. Also consider the extended input
space $\fP_{\cX} \times \cX$ and the extended data $\wt{X}_{ij} =
(\wh{P}^{(i)}_{X}, X_{ij})$. Note that $\wh{P}^{(i)}_{X}$ plays a role
analogous to the task index in multi-task learning. Now define
\begin{equation}
\wh{f}_\lambda =
\argmin_{f \in \cH_{\ol{k}}} \frac{1}{N} \sum_{i=1}^N
\frac{1}{n_i}\sum_{j=1}^{n_i}
\ell(f(\wt{X}_{ij}),Y_{ij}) + \lambda \norm{f}^2.
\label{eq:est}
\end{equation}
Algorithms for solving \eqref{eq:est} will be discussed in Section \ref{sec:imp}.

%In the remainder of this section, we elaborate on the definition of the
%kernel, and on the implementation of the algorithm.

%% (Note: one might want to discuss the relevance of the factor $1/n_i$
%% between the two above sum signs, as we might want to give more weight
%% to larger samples. However, the above is more natural as an empirical
%% version of an idealized test error, see analysis below.
%% Anyway, to start with we may consider the case
%% where all $n_i$s are equal to shunt this issue)

\subsection{Specifying the Kernels}

In the rest of the paper we will consider a kernel $\ol{k}$ on
$\fP_{\cX} \times \cX$ of the product form
\begin{equation}
\label{eq:prodkern}
\ol{k}( (P_1,x_1) , (P_2,x_2)) = k_{P}(P_1,P_2)k_X(x_1,x_2),
\end{equation}
where $k_{P}$ is a kernel on $\fP_\cX$ and $k_X$ a kernel on $\cX$.

%% {\bf Note}: formally, if $\cH_{\ol{k}},\cH_{k_P},\cH_{k_X}$ are the RKHSs
%% associated to
%% the respective kernels, we have an isometric isomorphism
%% $\cH_{\ol{k}}=\cH_{k_P}\otimes\cH_{k_X}$.

%% {\bf Conjecture:} if both $k_P$ and $k_X$ are universal, then $\ol{k}$ is.
%% (To check in the kernel literature).

Furthermore, we will consider kernels on $\fP_{\cX}$ of a particular form.
Let $k_X'$ denote a kernel on $\cX$ (which might be different
from $k_X$)
that is measurable and bounded. We define the {\em kernel mean embedding}
$\Psi: \fP_{\cX} \rightarrow \cH_{k'_X}$:
\begin{equation}
\label{eq:Psidef}
P_X \mapsto \Psi(P_X) := \int_{\cX} k'_X(x,\cdot) dP_X(x).
\end{equation}
This mapping has been studied in the framework of ``characteristic
kernels''
\citep{Greetal07a}, and it has been proved that
universality of $k'_X$ implies injectivity of $\Psi$
\citep{Greetal07b,Srietal10}.

Note that the mapping $\Psi$ is linear. Therefore, if we consider the
kernel $k_P(P_X,P_X') = \inner{\Psi(P_X),\Psi(P_X')}$, it is a linear
kernel on $\fP_{\cX}$ and cannot be a universal kernel.
For this reason, we introduce yet another kernel $\fK$ on
$\cH_{k'_X}$
and consider the kernel on $\fP_\cX$ given by
\begin{equation}
\label{eq:psikern}
k_P(P_X,P_X') = \fK\paren{\Psi(P_X),\Psi(P_X')}\,.
\end{equation}
Note that particular kernels inspired
by the finite dimensional case are of the form
\begin{equation}
\label{eq:Fkernel}
\fK(v,v') = F(\norm{v-v'}),
\end{equation}
or
\begin{equation}
\label{eq:Gkernel}
\fK(v,v') = G(\inner{v,v'}),
\end{equation}
where $F,G$ are real functions of a real variable such that they define
a kernel. For example, $F(t) = \exp(-t^2/(2\sigma^2))$ yields a
Gaussian-like kernel,
while $G(t) = (1+t)^d$ yields a polynomial-like kernel.
Kernels of the above form on the space of probability distributions
over a compact space $\cX$ have been introduced and studied in
\citet{CriSte10}. Below we apply their results to deduce that $\ol{k}$ is a
universal kernel for certain choices of $k_X, k_X'$, and
$\fK$.

\subsection{Relation to Other Kernel Methods}
\label{sec:relkern}

By choosing $\ol{k}$ differently, one can recover other existing kernel
methods. In particular, consider the class of kernels of the same product
form as above, but where
\[
k_P(P_X, P_X') = \left\{
\begin{array}{ll}
1 & P_X = P_X' \\
\tau & P_X \ne P_X'
\end{array}
\right.
\]
If $\tau = 0$, the algorithm \eqref{eq:est} corresponds to training $N$
kernel machines $f(\widehat{P}_X^{(i)}, \cdot)$ using kernel $k_X$ (e.g., support vector
machines in the case of the hinge
loss) on each training data set, independently of the others (note that this does not offer any
generalization ability to a new data set).
If $\tau = 1$, we have a
``pooling" strategy that, in the case of equal sample sizes $n_i$, is equivalent to pooling
all training data sets
together in a single data set, and running a conventional
supervised learning algorithm with kernel $k_X$ ({\em i.e.}, this corresponds to trying to find
a single ``one-fits-all'' prediction function which does not depend on the marginal).
In the intermediate case $0 < \tau <
1$, the resulting kernel is a ``multi-task kernel," and the algorithm
recovers a multitask learning algorithm like that of \citet{evgeniou05}. We
compare to the pooling strategy below in our experiments. We also examined the multi-task kernel with
$\tau < 1$, but found that, as far as generalization to a new unlabeled task is concerned,
it was always outperformed by pooling, and so those results are not reported. This
fits the observation that the choice $\tau=0$ does not provide any generalization to a new task,
while $\tau=1$ at least offers some form of generalization, if only by fitting the same predictor to all data sets.

In the special case where all labels $Y_{ij}$ are the same value for a given task, and
$k_X$ is taken to be the constant kernel, the problem we consider reduces to
``distributional" classification or regression, which is essentially standard supervised
learning where a distribution (observed through a sample) plays the role of the feature
vector. Many of our analysis techniques specialize to this setting.

\section{Learning Theoretic Study}

This section presents generalization error and consistency analysis for the proposed kernel method under the agnostic and 2-stage generative models. Although the regularized estimation formula \eqref{eq:est} defining
$\wh{f}_\lambda$ is standard, the generalization error analysis is not, owing to the particular sampling structures and risks under {\bf (AGM)} and {\bf (2SGM)}.
%since the $\wt{X}_{ij}$ are neither identically distributed nor independent \citep{szabo2016learning}.

\subsection{Universal Consistency under the Agnostic Generative Model}

We will consider the following assumptions on the loss function and kernels:

\begin{description}
\item[] {\bf (LB)}
The loss function $\ell:\mbr \times \cY \rightarrow \mbr_+$
is
$L_\ell$-Lipschitz in its first variable and satisfies
$B_0 := \sup_{y \in \cY} \ell(0,y) < \infty.$
% bounded by $B_\ell$.

\item[] \kbound/
%The kernel $\ol{k}$ on $\fP_{\cX}\times\cX$ is given by
%\[
%\ol{k}( (P_1,x_1) , (P_2,x_2)) = %k_{P}(P_1,P_2)
%\fK\paren{\Psi(P_X),\Psi(P_X')} k_X(x_1,x_2),
%\]
%where the mapping $\Psi$ is defined as
%\[
%P_X \mapsto \Psi(P_X) := \int_{\cX} k'_X(x,\cdot) dP_X(x)\,.
%\]
%Above, $k_X,k_X'$ are two kernels on $\cX$, and $\fK$ is
%a kernel on $\cH_{k'_X}$ (the RKHS associated to $k'_X$),
%such that:
The kernels $k_X,k_X'$ and $\fK$ are bounded respectively by
constants $B_k^2,B_{k'}^2\geq 1$, and $B_\fK^2$\,. 
%\end{itemize}
%that the kernel $k_P$ on $\fP_\cX$ has the
%form
%\[
%k_p(P_X,P_X') = \fK\paren{\Psi(P_X),\Psi(P_X')}\,,
%\]
 \end{description}

The condition $B_0 < \infty$ always holds for classification, as well as certain regression settings.
The boundedness assumptions are clearly satisfied for Gaussian kernels, and can be enforced by normalizing the kernel (discussed further below).

We begin with a generalization error
 bound that establishes uniform estimation error control over functions belonging to a ball of
 $\cH_{\ol{k}}$\,. We then discuss universal kernels, and finally deduce
 universal consistency of the algorithm.

 % To simplify somewhat the presentation, we assume below that all training
% samples have the same size $n_i = n$. Also 

Let $\cB_{k}(r)$ denote the
closed ball of radius $r$, centered at the origin, in the RKHS of the
kernel $k$.
We start with the following simple result allowing us to bound the loss on a RKHS ball.
\begin{lemma}
\label{le:boundloss}
  Suppose $k$ is a kernel on a set $\Omega$, bounded by $B^2$. Let $\ell:\mathbb{R} \times \cY \to [0,\infty)$ be a loss satisfying {\bf (LB)}.
%  wih Lipschitz constant $L_\ell$ and such that
%$$
%B_0 := \sup_{y \in \cY} \ell(0,y) < \infty.
%$$
Then for any $R > 0$ and $f\in B_k(R)$, and any $z \in \Omega$ and $y \in \cY$,
\begin{equation}
\label{eqn:kernelbnd}
\big|\ell(f(z),y)\big| \leq B_0 + L_\ell R B 
\end{equation}
\end{lemma}
\begin{proof}
By the Lipschitz continuity of $\ell$, the reproducing property, and Cauchy-Schwarz, we have
\begin{align*}
  \big|\ell(f(z),y)\big| & \leq \ell(0,y)+\big|\ell(f(z),y)-\ell(0,y)\big| \\
   & \leq B_0 + L_\ell |f(z)-0| \\
   &= B_0+L_\ell\big|\langle f,k(z,\cdot)\rangle \big| \\
   & \leq B_0+L_\ell \|f\|_{{\cal H}_k} B \\
   &\leq B_0+L_\ell R B.
\end{align*} 
\end{proof}

The expression in \eqref{eqn:kernelbnd} serves to replace the boundedness assumption \eqref{eqn:agmbnd} in Theorem \ref{th:radgenbound}. We now state the following, which is a specialization of Theorem \ref{th:radgenbound} to the kernel setting.

\begin{theorem}[Uniform estimation error control over RKHS balls]
\label{th:main}
Assume {\bf (LB)} and \kbound/ hold, and data generation
follows {\bf (AGM)}.
%If $P_{XY}^{(1)},\ldots,P_{XY}^{(N)}$ are i.i.d. realizations from $\mu$, and
%for each $i=1,\ldots,N$, the sample $S_i=(X_{ij},Y_{ij})_{1\leq j \leq n}$ is made of
%i.i.d. realizations from $P_{XY}^{(i)}$,
Then for any $R>0$,
with probability at least $1-\delta$ (with respect to the draws of the samples $S_i, i=1,\ldots,N$)
%
%tasks $P_{XY}^{(i)}$ and the training data $(X_{ij},Y_{ij})$):
\begin{equation}
  \label{eq:mainbound}
  \sup_{f \in \cB_{\ol{k}}(R)} \abs{     \wh{\cE}(f,N) - \cE(f)}
  \leq
%  \frac{2 L_\ell R B_k B_\fK}{\sqrt{N}}
%  + (B_0 + L_\ell R B_\fK B_k) \sqrt{\frac{\log \delta^{-1}}{2N}}.
(B_0 + L_\ell R B_\fK B_k) \frac{(\sqrt{\log \delta^{-1}}+2)}{\sqrt{N}}.
\end{equation}
%\begin{multline}
%\sup_{f \in \cB_{\ol{k}}(R)} \abs{ 
  %\frac{1}{N}\sum_{i=1}^N \frac{1}{n} \sum_{j=1}^{n}
  %\ell(f(\wt{X}_{ij}),Y_{ij})
%   \wh{\cE}(f) - \cE(f)}\\
%\leq
%c \paren{ RB_kL_\ell \paren{B_{k'} L_\fK \paren{\frac{\log N + \log \delta^{-1}}{n}}^{\frac{\alpha}{2}}
%+ B_\fK \frac{1}{\sqrt{N}}}
%+ B_\ell
%\sqrt{\frac{\log \delta^{-1}}{Nn}}  + L_\ell R B_k B_\fK \frac{1}{\sqrt{Nn}}
%\frac{2 L_\ell R B_k B_\fK \frac{1}{\sqrt{N}
%+ B_\ell \sqrt{\frac{\log \delta^{-1}}{2N}}},
%\end{multline}
%where $c$ is a numerical constant.
%, and $\cB_{\ol{k}}(R)$ denotes the ball
%of radius $R$ of $\cH_{\ol{k}}$\,.
\end{theorem}

\begin{proof}
  This is a direct consequence of Theorem~\ref{th:radgenbound} and of
  Lemma~\ref{le:boundloss}, the kernel $\ol{k}$ on $\fP_\cX\times\cX$ being bounded by
$B_k^2 B_\fK^2$.
  As noted there, the main term in the upper bound~\eqref{eq:rademmainbound} is
  a standard Rademacher complexity on the augmented input space $\fP \times \cX$,
  endowed with the Campbell measure $\cC(P_S)$.

  In the kernel learning context, we can bound the Rademacher complexity term using a
  standard bound for the Rademacher complexity of a Lipschitz loss function
  on the ball of radius $R$ of $\cH_{\ol{k}}$  (\citealp{kolt01,BarMen02}, e.g., Theorems~8, 12 and Lemma~22 there),
  using again the bound $B_k^2 B_\fK^2$ on the kernel $\ol{k}$, giving the conclusion.
\end{proof}

Next, we turn our attention to universal kernels (see Section
\ref{sec:alg} for the definition).
A relevant notion for our purposes is that of a
normalized kernel.  If $k$ is a kernel on $\Omega$, then
\[
k^*(x,x') := \frac{k(x,x')}{\sqrt{k(x,x)k(x',x')}}
\]
is the associated {\em normalized} kernel. If a kernel is universal, then
so is its associated normalized kernel. For example, the exponential
kernel $k(x,x') = \exp(\kappa \langle x, x' \rangle_{\mbr^d})$, $\kappa >
0$, can be shown to be
universal on $\mbr^d$ through a Taylor series argument. Consequently, the
Gaussian kernel
\[
k_\sigma(x,x') := \frac{\exp(\frac1{\sigma^2}\langle x, x'
\rangle)}{\exp(\frac1{2\sigma^2} \|x\|^2)\exp(\frac1{2\sigma^2} \|x'\|^2)}
\]
is universal, being the normalized kernel associated with the exponential
kernel with $\kappa = 1/\sigma^2$.  See \citet{steinwart08} for
additional details and discussion.

To establish that $\ol{k}$ is universal on $\fP_{\cX} \times \cX$, the
following lemma is useful.
\begin{lemma}
\label{lem:produniv}
Let $\Omega,\Omega'$ be two compact spaces and $k,k'$ be kernels on
$\Omega,\Omega'$, respectively. If $k,k'$ are both universal, then the product kernel
\[
\ol{k}((x,x'),(y,y')) := k(x,y)k'(x',y')
\]
is universal on $\Omega\times\Omega'$.
\end{lemma}

Several examples of universal kernels are known on Euclidean space. For
our purposes, we also need universal kernels on $\fP_{\cX}$.
Fortunately, this was studied by \citet{CriSte10}. Some
additional assumptions on the kernels and feature space are required:

\begin{description}
\item[] \kuniv/ $k_X$, $k_X'$, $\fK$, and $\cX$ satisfy the
following:
\begin{itemize}
\item $\cX$ is a compact metric space
\item $k_X$ is universal on $\cX$
\item $k'_X$ is continuous and universal on $\cX$
\item $\fK$ is universal on any compact subset of $\cH_{k_X'}$.
\end{itemize}
\end{description}

Adapting the results of \citet{CriSte10}, we have the following.

\begin{theorem}[Universal kernel]
\label{th:univkbar}
Assume condition \kuniv/ holds.
Then, for $k_P$ defined as in \eqref{eq:psikern}, the product kernel
$\ol{k}$ in \eqref{eq:prodkern} is universal on $\fP_{\cX} \times \cX$.

Furthermore, the assumption on $\fK$ is fulfilled if $\fK$ is of the form
\eqref{eq:Gkernel}, where $G$ is an analytical function with positive
Taylor series coefficients, or if $\fK$ is the normalized kernel
associated to such a kernel.
\end{theorem}

\begin{proof}
By Lemma \ref{lem:produniv}, it suffices to show $\fP_{\cX}$ is a compact
metric space, and that $k_P(P_X,P_X')$ is
universal on $\fP_{\cX}$. The former statement follows from Theorem 6.4 of
\citet{Par67}, where the metric is the Prohorov metric. We will deduce the
latter statement from Theorem~2.2 of \citet{CriSte10}. The statement
of Theorem~2.2 there is apparently restricted to kernels of the
form \eqref{eq:Gkernel}, but the proof actually only uses that the
kernel $\fK$ is universal on any compact set of $\cH_{k_X'}$. To apply
Theorem 2.2,
it remains to show that $\cH_{k_X'}$ is a separable Hilbert space, and
that $\Psi$ is injective and continuous. Injectivity of
$\Psi$ is equivalent to $k_X'$ being a characteristic kernel, and
follows from the assumed universality of $k_X'$ \citep{Srietal10}.
The continuity of $k_X'$ implies separability of
$\cH_{k_X'}$ (\citet{steinwart08}, Lemma 4.33)
as well as continuity of $\Psi$ (\citet{CriSte10}, Lemma 2.3
and preceding discussion). Now Theorem 2.2 of
\citet{CriSte10} may be applied, and the results follows.

The fact that kernels of the form \eqref{eq:Gkernel}, where $G$ is
analytic with positive Taylor coefficients, are
universal on any compact set of $\cH_{k_X'}$ was established in the proof
of
Theorem 2.2 of the same work \citep{CriSte10}.
\end{proof}

As an example, suppose that $\cX$ is a compact subset of $\mbr^d$. Let
$k_X$ and $k_X'$ be Gaussian kernels on $\cX$.
Taking $G(t) = \exp(t)$, it follows that $\fK(P_X,P_X') = \exp(\langle
\Psi(P_X), \Psi(P_X') \rangle_{\cH_{k_X'}})$ is universal on $\fP_{\cX}$.
By similar reasoning as in the finite dimensional case, the Gaussian-like
kernel $\fK(P_X,P_X') = \exp(-\frac1{2\sigma^2}\| \Psi(P_X) - \Psi(P_X')
\|^2_{\cH_{k_X'}})$ is also universal on $\fP_{\cX}$. Thus the product
kernel is universal on $\fP_{\cX} \times \cX$.

% in particular kernels of the form \eqref{eq:Gkernel}
%have been shown to be universal on $\fP_{\cX}$ when $k'_X$ is a universal
%kernel on $\cX$ and $G$ is an analytic function with strictly positive
%Taylor expansion coefficients.

From Theorems \ref{th:main} and \ref{th:univkbar}, we may deduce
universal consistency of the learning algorithm. 

\begin{corollary}[Universal consistency]
\label{cor:const}
%Let $\ell$ be Lipschitz in its first variable such that
%\begin{equation}
%  \label{eqn:zeroloss}
%\sup_{y \in \cY} \ell(0,y) < \infty.
%\end{equation}
Assume that conditions {\bf (LB)}, \kbound/ and \kuniv/ are satisfied.
%Assume that as
%$\min(N,n)\to \infty$, $N = \cO(n^\gamma)$ for some
%$\gamma > 0$, and $\frac{N}{\log n} \to \infty$. Also
Let $\lambda =\lambda(N)$ be a sequence such that
%as $\min(N,n) \to
%\infty$,
as $N\rightarrow \infty$: $\lambda(N) \to 0$  and
$\lambda(N) N/ \log N \rightarrow \infty$.
%\begin{equation}
%\label{eq:condlamconst}
%\lambda \min \left(\frac{N}{\log n}, \left(\frac{n}{\log n} \right)^\alpha
%\right) \to
%\infty.
%\end{equation}
Then
\[
\cE(\wh{f}_{\lambda(N)}) \rightarrow
\inf_{f:\fP_{\cX} \times \cX \rightarrow \mbr} \cE(f) \text{ a.s., as } N \rightarrow \infty.
\]
%Assume that $N,n$ grow to infinity in such a way that $N = \cO(
%n^\gamma)$
%for some $\gamma>0$.
%Then, if $\lambda_j$ is a sequence such that $\lambda_j \rightarrow 0$
%and
%$\lambda_j \sqrt{\frac{j}{\log j}} \rightarrow \infty$, it holds that
%\[
%\cE(\wh{f}_{\lambda_{\min(N,n^\alpha)}},\infty) \rightarrow
%\inf_{f:\fP_{\cX} \times \cX \rightarrow \mbr} \cE(f,\infty)
%\]
%in probability.
\end{corollary}
The proof of the corollary relies on the bound established in Theorem
\ref{th:main}, the universality of $\ol{k}$ established in Theorem
\ref{th:univkbar}, and otherwise relatively standard arguments.
%The
%assumption~{\bf (LB)} always holds for classification, and it
%holds for regression, for example, when $\cY$ is compact and $\ell(0,y)$
%is continuous as a function of $y$.

One notable feature of this result is that we have established consistency where only $N$ is required to diverge. In particular, the training sample sizes $n_i$ may remain bounded. In the next subsection, we consider the role of the $n_i$ under the 2-stage generative model.

\subsection{Role of the Individual Sample Sizes under the 2-Stage Generative Model
%{\bf (2SGM)}
}
\label{sec:theory2sgm}

In this section, we are concerned with the role of the individual sample sizes $(n_i)_{1 \leq i \leq N}$, more precisely, of their distribution $\nu$ under {\bf (2SGM)}, see Section~\ref{sec:datagen}. 
%(Recall that a particular case of {\bf (2SGM)} is when all the samples have the same size $n$, i.e., $\nu=\delta_n$.) 
A particular motivation for investigating this point is that in some applications the number of training points per task is large, which can give rise
to a high computational burden at the learning stage (and also for storing the learned model in computer memory).
A practical way to alleviate this issue is to reduce the number of training
points per task by random subsampling, which in effect modifies the sample size distribution $\nu$  while keeping the generating distribution $\mu$ for the tasks' point distributions unchanged.
%We investigate to which extent reducing the number of training
%points per task (by random subsampling without replacement) in order to reduce computational %urden can be done without
%suffering a significant loss in statistical performance. 
%\gillesb{Under {\bf (2SGM)}, the effect of subsampling without
%replacement is transparent: it amounts to changing the original individual sample size
%distribution $\nu$ by $\nu'=\delta_n$, while keeping the generating distribution $\mu$ for the
%tasks' point distributions fixed. Here $n$ is the common fixed size of the training
%subsamples, and we assume implicitly that the original sample sizes are a.s. larger than $n$, %i.e. $\nu$ has support $[n,\infty)$.
%If this is not the case, a subsample size depending on the
%original sample size may be considered,
%n this case $\nu'$ can also be a general distribution;
%t any rate we are still under {\bf (2SGM)} for the subsampled tasks.
Observe that under
{\bf (AGM)} the sample size and the sample point distribution may be dependent
in general, and subsampling would then affect that relationship in an
unknown manner. This is why we assume {\bf (2SGM)} in the present
section.
%For this we need a more specific model
%for the generating model of points in each task, and we therefore assume here that
%the {\bf (2SGM)}, introduced in Section~\ref{sec:datagen}, holds.

We will consider the following additional assumption.

\begin{description}
  \item[] \khoeld/ The
canonical feature map $\Phi_\fK: \cH_{k'_X} \rightarrow \cH_{\fK}$
associated to $\fK$ satisfies a H\"{o}lder condition of order
$\alpha \in (0,1]$\, with constant $L_\fK$, on $\cB_{k'_X}(B_{k'})$\,:
%the ball of radius $B_{k'}$
%of $\cH_{k'_X}$, denoted $\cB_{k'_X}(B_{k'})$\,,
\begin{equation} \label{eq:kernhold}
\forall v,w \in  \cB_{k'_X}(B_{k'}): \qquad
\norm{\Phi_{\fK}(v) - \Phi_{\fK}(w)} \leq L_{\fK} \norm{v-w}^{\alpha}.
\end{equation}
\end{description}
Sufficient conditions for \eqref{eq:kernhold} are described in Section
\ref{sec:reg}. As an example, the condition is shown to hold with $\alpha
= 1$ when $\fK$ is the Gaussian-like kernel on ${\cal H}_{k_X'}$.

Since we are interested in the influence of the number of training points per task, it is helpful to introduce notations for the {\bf (2SGM)} risks that are conditioned on a fixed task $P_{XY}$.
%we will first focus here
%on a single task, i.e., a sample $S=(X_{j},Y_{j})_{1\leq j \leq n}$ of size $n$ (and the associated empirical
%loss $\cL(S,f)$), or a fixed generating sample distribution $P_{XY}$.
Thus, we introduce
the following notation, in analogy to~\eqref{eq:err_nt}--\eqref{eq:err_infty} introduced
in Section~\ref{se:2gsm_gen}, for 
risk at sample size $n$, and  risk at infinite sample size, conditional to $P_{XY}$:
\begin{align}
%   \wh{\cE}(f,S_n) & := \frac{1}{n} \sum_{i=1}^{n} \ell(f(\wh{P}_X,X_i),Y_i);\\
%    \label{eq:err_nt_pxy}
 \cE(f|P_{XY},n) & := \ee{S^T \sim (P_{XY})^{\otimes n}}{\frac{1}{n} \sum_{i=1}^{n} \ell(f(\wh{P}_X,X_i),Y_i)};\\
\label{eq:err_infty_pxy}
\cE^\infty(f|P_{XY}) & := \ee{(X,Y) \sim P_{XY}} {\ell(f(P_X,X),Y)}.
\end{align}
%above %$S_n=(X_i,Y_i)_{1 \leq i \leq n}$ is a fixed sample of size $n$, and
%$\wh{P}_X=\frac{1}{n}\sum_{i=1}^n \delta_{X_i}$ is the $X$-marginal empirical distribution associated to the sample $S$.

The following proposition gives an upper bound on the discrepancy between these risks.
It can be seen as a quantitative version of Proposition~\ref{prop:conv_err_2gsm}
in the kernel setting, which is furthermore uniform over an RKHS ball.
\begin{theorem}
\label{th:tworisks}
Assume conditions {\bf (LB)}, \kbound/, and \khoeld/ hold.
If the sample $S=(X_{j},Y_{j})_{1\leq j \leq n}$ is made of $n$
i.i.d. realizations from $P_{XY}$, with $P_{XY}$ and $n$ fixed, then for any $R>0$, with probability
at least $1-\delta$:
\begin{multline}
\label{eq:bound2gsm_emp}
\sup_{f \in \cB_{\ol{k}}(R)} \abs{\cL(S,f) - \cE^\infty(f|P_{XY})}
\leq (B_0 + 3 L_\ell R B_k ( B_{k'}^\alpha L_\fK +B_{\fK}) )\paren{\frac{\log (3\delta^{-1})}{n}}^{-\frac{\alpha}{2}}.
\end{multline}
%For the averaged risk over the draw of $S^{(n)}$, it 
Averaging over the draw of $S$, again with $P_{XY}$ and $n$ fixed, it holds for any $R>0$:
\begin{equation}
\label{eq:bound2gsm_ave}
\sup_{f \in \cB_{\ol{k}}(R)} \abs{\cE(f|P_{XY},n) - \cE^\infty(f|P_{XY})}
\leq 2 L_\ell R B_k L_\fK B_{k'}^\alpha n^{-\alpha/2}.
\end{equation}
As a consequence, for the unconditional risks when $(P_{XY}, n)$ is drawn from $\mu \otimes \nu$ 
under {\bf (2SGM)}, for any $R>0$:
\begin{equation}
\label{eq:bound2gsm_gen}
\sup_{f \in \cB_{\ol{k}}(R)} \abs{\cE(f) - \cE^\infty(f)}
\leq 2 R L_\ell B_k L_\fK B_{k'}^\alpha
\ee{\nu}{n^{-\frac{\alpha}{2}}}.
\end{equation}
\end{theorem}
The above results are useful in a number of ways. First, under {\bf (2SGM)},
we can consider the goal of  asymptotically achieving the idealized optimal risk $\inf_{f} \cE^\infty(f)$, where we recall that $\cE^\infty(f)$ is the expected loss of a decision function $f$ over a random test task $P_{XY}^T$ in the case where $P_X^T$ would be
perfectly observed (this can be thought of as observing an infinite sample from the marginal). Equation \eqref{eq:bound2gsm_gen} bounds the risk under {\bf (2SGM)} in terms of the risk under {\bf (AGM)}, for which we have already established consistency. Thus, consistency 
to the idealized risk under {\bf (2SGM)} will be possible
if the number of examples $n_i$ per training task also grows together with the number
of training tasks $N$. The following result formalizes
this intuition.
\begin{corollary}
\label{th:altgen}
Assume {\bf (LB)}, \kbound/, and \khoeld/, and 
assume {\bf (2SGM)}.
%let $n$ be fixed.
%If $P_{XY}^{(1)}, \ldots, P_{XY}^{(N)}$ are drawn i.i.d. from $\mu$, and the corresponding samples $S_1, \ldots, S_N$ are i.i.d. all of size $n$ from their respective distributions, 
Then for any $R>0$, with probability at least $1-\delta$ with respect to the draws of the training tasks and training samples
\begin{multline}
\label{eq:altgen}
\sup_{f \in \cB_{\ol{k}}(R)} \abs{
\frac{1}{N} \sum_{i=1}^N \cL(S_i,f) -
\cE^\infty(f)} \\
\leq (B_0 + L_\ell R B_\fK B_k)  \frac{(\sqrt{\log \delta^{-1}}+2)}{\sqrt{N}} + 
2 R L_\ell B_k L_\fK B_{k'}^\alpha \ee{\nu}{n^{-\frac{\alpha}{2}}}.
%\frac{2 R L_\ell B_k B_\fK}{ \sqrt{N}} + R L_\ell B_k L_\fK B_{k'}^\alpha %n^{-\frac{\alpha}{2}} + B_{\ell} \sqrt{\frac{\log (4/\delta)}{2N}}.
\end{multline}
Consider an asymptotic setting under {\bf (2SGM)} in which, as the number of training tasks $N\rightarrow\infty$, the distribution $\mu$ remains fixed but the sample
size distribution $\nu_N$ depends on $N$. Denote $\kappa_N^{-1} := \ee{\nu_N}{n^{-\alpha/2}} $.
Assuming \kuniv/ is satisfied, and the regularization parameter
$\lambda(N)$ is such that $\lambda(N) \rightarrow 0$ and $\lambda(N)\min(N,\kappa_N^2) \rightarrow \infty$, then
\[
  \cE(\wh{f}_{\lambda(N)}) \rightarrow \inf_{f:\fP_{\cX} \times \cX \rightarrow \mbr} \cE^\infty(f) \text{ in probability, as } N \rightarrow \infty.
  \]
\end{corollary}
\begin{proof}
  The setting is that of the {\bf (2SGM)} model.
  %where the sample size is fixed at
  %$n$ (i.e. the sample size distribution $\nu$ is the Dirac $\delta_n$).
  This is a particular case of {\bf (AGM)}, so we can apply Theorem~\ref{th:main}
  and combine with~\eqref{eq:bound2gsm_gen} 
  %(wherein the test sample is also of size $n$)
  to get the announced bound.
  The consistency statement follows the same argument as in the proof of Corollary~\ref{cor:const}, with $\cE(f)$ replaced by $\cE^\infty(f)$, and $\eps(N)$ there replaced
  by the RHS in~\eqref{eq:altgen}. 
\end{proof}

\begin{remark} The bound~\eqref{eq:altgen} is non-asymptotic and can be used
as such to assess the respective role of number of tasks $N$ and of the
sample size distribution $\nu$ when the objective is the idealized risk (see below).
The result of consistency to that risk on the other hand, is formalized as a ``triangular array" type of asymptotics where the distribution of the sizes $n_i$ of the i.i.d. 
training samples $S_i$ changes with their number $N$.
%% Gilles: The following discussion is rather verbose, maybe simply remove it
% Although almost sure convergence is still valid, we have 
% almost sure consistency appears less relevant since it is harder
% to imagine in practice a genuinely sequential setting in which both
% $N$ and $n_i$ grow by experimental design. The ``triangular array" asymptotics rather represents a situation where the same learning
% problem would be addressed anew from the top with growing resources
% invested both in sample number and sample size.
\end{remark}
\begin{remark}
Our conference paper \citep{blanchard:11:nips} also established a generalization error bound and consistency for $\cE^\infty$ under {\bf (2SGM)}. That bound had a different form for two main reasons. First, it assumed the loss to be bounded, whereas the present analysis avoids that assumption via Lemma \ref{le:boundloss}. Second, that analysis did not leverage a connection to {\bf (AGM)}, which led to a $\log N$ in the second term. This required the two sample sizes to be coupled asymptotically to achieve consistency. In the present analysis, the two sample sizes $N$ and $n$ may diverge at arbitrary rates.
\end{remark}
\begin{remark} It is possible to obtain a result similar to Corollary~\ref{th:altgen} when
the training task sample sizes $(n_i)_{1\leq i \leq N}$ are fixed
(considered as deterministic), unequal and possibly arbitrary. In this case we would follow a slightly different argument,
leveraging~~\eqref{eq:bound2gsm_emp} for each single training task together with a union bound,
and applying Theorem~\ref{th:main} to the idealized situation with an infinite number
of samples per training task. This way, the term $\ee{\nu}{n^{-\frac{\alpha}{2}}}$ is replaced by $\log(N) N^{-1}\sum_{i=1}^N n_i^{-\frac{\alpha}{2}}$,
the additional logarithmic factor being the price of the union bound. We eschew an exact statement for brevity.
\end{remark}

We now come back to our initial motivation of possibly reducing computational burden by subsampling and analyze to what extent this affects statistical error. Under {\bf (2SGM)}
 the effect of subsampling (without
replacement) is transparent: it amounts to changing the original individual sample size
distribution $\nu$ by $\nu'=\delta_{n'}$, while keeping the generating distribution $\mu$ for the
tasks' point distributions fixed. Here $n'$ is the common fixed size of the training
subsamples, and we must assume implicitly that the original sample sizes are a.s. larger than $n'$, i.e. that their distribution $\nu$ has support $[n',\infty)$.
%If this is not the case, a random subsample size $n'$ depending on the original sample size $n$ may be considered, then $\nu'$ can also be a general distribution; at any rate we are still under {\bf (2SGM)} for the subsampled tasks.
For simplicity, for the rest of the discussion we only consider the case of equal, 
deterministic sizes of sample ($n$) and subsample ($n'<n$).
Using~\eqref{eq:altgen} we compare the two settings 
to a common reference, namely the idealized risk $\cE^\infty$.
%the two settings where we have the same task generating distribution $\mu$ but different individual
%training task sample sizes $n$ versus $n'<n$ . Under the {\bf (2SGM)} model, the
%setting $n'<n$ can be obtained by simple random subsampling of the original data.
We see that the statistical risk bound in~\eqref{eq:altgen} is unchanged up to a small factor if
$n' \geq \min(N^\alpha,n)$. Assuming $\alpha=1$ to simplify, in the case where the original sample sizes $n$ are much larger than the number of training tasks $N$, this suggests that
we can subsample to $n'\approx N$ without taking a significant hit to performance.
This applies equally well to subsampling the tasks used for prediction or testing. The most
precise statement in this regard is~\eqref{eq:bound2gsm_emp}, since it bounds the
deviations of the observed prediction loss for a fixed task $P_{XY}$ and i.i.d. sample from that task.

%\begin{remark}
  % From the results of Theorem~\ref{th:main} and Corollary~\ref{cor:const}, the limiting factor for the estimation error bound is given by the minimum appearing in~\eqref{eq:condlamconst}. Balancing the two terms (disregarding logarithmic factors) suggests to choose, if possible, a number of examples per task $n=\Omega(N^{1/\alpha})$ as function of the number of training tasks $N$ (where we recall that $\alpha \in (0,1]$ is the (known) H\"older smoothness of the kernel feature mapping).
  The minimal subsampling size $n'$ can be interpreted as an optimal efficiency/accuracy tradeoff, since it reduces computational complexity as much as possible without sacrificing statistical accuracy. Similar considerations appear in the context of distribution regression~\citep[Remark 6]{szabo2016learning}. In that reference, a sharp analysis giving rise to {\em fast convergence rates} is presented, resulting in a more involved optimal balance between $N$ and $n$. In the present work, we have focused on {\em slow rates} based on a uniform control of the estimation error over RKHS balls; we leave for future work sharper convergence bounds (under additional regularity conditions), which would also give rise to more refined balancing conditions between $n$ and $N$.

\section{Implementation}
\label{sec:imp}

Implementation of the algorithm in \eqref{eq:est}
relies on techniques that are similar to those used for other kernel methods, but with some
variations.\footnote{Code is available at
https://github.com/aniketde/DomainGeneralizationMarginal} The first subsection illustrates how, for the case of hinge loss, the optimization
problem corresponds to a certain cost-sensitive support vector machine. The second subsection
focuses on more scalable implementations based on approximate feature mappings.

\subsection{Representer Theorem and Hinge Loss}

For a particular loss $\ell$, existing algorithms for optimizing an
empirical
risk based on that loss can be adapted to the setting of marginal transfer learning.
We now illustrate this idea for the case of the hinge loss,
$\ell(t,y) = \max(0,1-yt)$. To make the presentation more concise, we will
employ the extended feature representation $\tX_{ij}= (\wh{P}^{(i)}_X,X_{ij})$, and we
will also ``vectorize'' the indices $(i,j)$ so as to employ a single index on these variables and on the labels. Thus
the training data are $(\tX_i,Y_i)_{1 \le i \le M},$ where $M =
\sum_{i=1}^N n_i$, and we seek a solution to
\[
\min_{f \in \cH_{\ol{k}}} \sum_{i=1}^M c_i \max(0,1-Y_i
f(\wt{X}_i)) + \frac12 \norm{f}^2.
\]
Here $c_i = \frac1{\lambda N n_m}$, where $m$ is the smallest
positive
integer such that $i \le n_1 + \cdots + n_m$. By the representer theorem
\citep{steinwart08}, the solution of \eqref{eq:est} has the form
\[
\wh{f}_\lambda = \sum_{i=1}^M r_i \ol{k}(\tX_i, \cdot)
\]
for real numbers $r_i$. Plugging this expression into the objective
function of \eqref{eq:est}, and introducing the auxiliary variables
$\xi_i$, we have the quadratic program
\begin{align*}
\min_{r,\xi} & \ \frac12 r^T \ol{K} r + \sum_{i=1}^M c_i \xi_i \\
{\rm s.t.} & \ Y_i \sum_{j=1}^M r_j
\ol{k}(\tX_i,\tX_j) \ge 1 - \xi_i, \ \forall i \\
& \ \xi_i \ge 0, \ \forall i,
\end{align*}
where $\ol{K} := (\ol{k}(\tX_i,\tX_j))_{1 \le i,j \le M}$.
Using Lagrange multiplier theory,
%and provided $\ol{K}$ is positive definite,
the dual quadratic program is
\begin{align*}
\max_{\alpha} & \ - \frac12 \sum_{i,j=1}^M \alpha_i \alpha_j Y_i Y_j
\ol{k}(\tX_i, \tX_j) + \sum_{i=1}^M \alpha_i \\
{\rm s.t.} & \ 0 \le \alpha_i \le c_i \ \forall i,
\end{align*}
and the optimal function is
\[
\wh{f}_\lambda = \sum_{i=1}^M \alpha_i Y_i \ol{k}(\tX_i, \cdot).
\]
This is equivalent to the dual of
a cost-sensitive support vector machine, without offset, where the costs
are given by
$c_i$. Therefore we can learn the weights $\alpha_i$ using any existing
software package for SVMs that accepts example-dependent costs and a
user-specified kernel matrix, and allows for no offset. Returning to the
original notation, the
final predictor given a test $X$-sample $S^T$ has the form
\[
\wh{f}_\lambda (\wh{P}^T_X,x) = \sum_{i=1}^N \sum_{j=1}^{n_i}
\alpha_{ij} Y_{ij} \ol{k}((\wh{P}_X^{(i)},X_{ij}),(\wh{P}^T_X,x))
\]
where the $\alpha_{ij}$ are nonnegative. Like the SVM, the solution is
often sparse, meaning most $\alpha_{ij}$ are zero.

Finally, we remark on the computation of $k_P(\wh{P}_X, \wh{P}_X')$. When
$\fK$ has the form of \eqref{eq:Fkernel} or \eqref{eq:Gkernel},
the calculation of $k_P$ may be reduced to computations of the form
$\inner{\Psi(\wh{P}_X), \Psi(\wh{P}_X')}$. If $\wh{P}_X$ and $\wh{P}_X'$
are empirical distributions based on the samples $X_1, \ldots, X_n$ and $X_1', \ldots, X_{n'}'$,
then
\begin{align*}
\inner{\Psi(\wh{P}_X), \Psi(\wh{P}_X')} &=
\inner{\frac1{n} \sum_{i=1}^n k_X'(X_i, \cdot),
\frac1{n'} \sum_{j=1}^{n'} k_X'(X_j', \cdot)} \\
&= \frac1{n n'} \sum_{i=1}^n \sum_{j=1}^{n'} k_X'(X_i,X_j').
\end{align*}
Note that when $k_X'$ is a (normalized) Gaussian kernel,
$\Psi(\wh{P}_X)$ coincides (as a function) with a smoothing
kernel density estimate for $P_X$.

\subsection{Approximate Feature Mapping for Scalable Implementation}

Assuming $n_i = n, $ for all $i$, the computational complexity of a
nonlinear SVM solver (in our context) is between $O(N^2n^2)$ and $O(N^3n^3)$
\citep{joachims:99:svmlight, chang2011libsvm}. Thus, standard nonlinear SVM solvers may be insufficient when $N$ or $n$ are very large.

One approach to scaling up kernel methods is to employ approximate feature
mappings together with linear solvers. This is based on the idea that
kernel methods are solving for a linear predictor after first nonlinearly
transforming the data. Since this nonlinear transformation can have an
extremely high- or even infinite-dimensional output, classical kernel
 methods avoid computing it explicitly. However, if the feature mapping can
be approximated by a finite dimensional transformation with a relatively
low-dimensional output, one can directly solve for the linear predictor,
which can be accomplished in $O(Nn) $ time \citep{hsieh2008dual}.

In particular, given a kernel $ \ol{k}$, the goal is to find an
approximate feature mapping $\ol{z}(\tilde{x})$ such that
$\ol{k}(\tilde{x},\tilde{x}') \approx
\ol{z}(\tilde{x})^T\ol{z}(\tilde{x}')$. Given such a mapping $\ol{z}$, one
then applies an efficient linear solver, such as
Liblinear \citep{fan2008liblinear}, to the training data
$(\ol{z}(\tilde{X}_{ij}), Y_{ij})_{ij}$ to obtain a weight vector $w$. The final
prediction on a test point $\tilde{x}$ is then $\sign\{w^T \ol{z}(\tilde{x})\}$. As
described in the previous subsection, the linear solver may need to be
tweaked, as in the case of unequal sample sizes $n_i$, but this is
usually straightforward.

Recently, such low-dimensional approximate future mappings $z(x)$ have
been developed for several kernels. We examine two such techniques in the
context of marginal transfer learning, the Nystr{\"o}m approximation
\citep{williams2001using, drineas2005nystrom} and random Fourier features.
The Nystr{\"o}m approximation applies to any kernel method, and therefore
extends to the marginal transfer setting without additional work.  On the
other hand, we give a novel extension of random Fourier features to the
marginal transfer learning setting (for the case of all Gaussian kernels),
together with performance analysis. Our approach is similar to the one in \citet{jitkrittum2015kernel} which proposes a two-stage approximation for the mean embedding. Note that \citet{jitkrittum2015kernel} does not give an error bound. We describe our novel extension of random Fourier features to the marginal transfer learning setting, with error bounds, in the appendix, where we also review the Nystr\"om method.

The Nystr{\"o}m approximation holds for any positive definite kernel, but
random Fourier features can be used only for shift invariant kernels. On the other hand, random Fourier features are very easy to implement and the Nystr{\"o}m method has additional time complexity due to an 
eigenvalue decomposition. Moreover, the Nystr{\"o}m method is useful only when the kernel matrix has low rank. For additional comparison of various kernel approximation approaches we refer the reader to \citet{le2014fastfood}. In our experiments, we use random Fourier features when all kernels are Gaussian and the Nystr{\"o}m method otherwise.

%  \begin{algorithm}
%    \SetKwInOut{Input}{input}
%    \SetKwInOut{Output}{output}
%    \SetKwInOut{Compute}{compute}
%    \Input{ Training sample $ S_i =
%(X_{ij},Y_{ij})_{1 \le j \le n_i}, i = 1,...,N$, testing sample
%$(X_j^T)_{1 \le j \le n_T}$  and kernels $k_X,k_X^{\prime}, k_P$. }
%       \Compute{$\bar{z}(\tilde{x}))$ using \ref{alg:RFF_Approx} or
%\ref{alg:Nystro_Approx}. \newline Use any linear kernel method solver
%applied to $\{ \bar{z}(\tilde{X}_{ij}),Y_{ij}) \}_{i=1,j=1}^{N,n_i} $ to
%get labels $(Y_j^T)_{1 \le j \le n_T}$ for testing sample.}
%	\Output{$(Y_j^T)_{1 \le j \le n_T}$. \newline (More generally
%output is a $ w \in \mathbb{R}^D$ s.t. prediction is $ f(\tilde{x}_T) =
%w^T\bar{z}(\tilde{x}_T)$, where $ \tilde{x}_T = (\widehat{P}_X^T, X^T),
%\bar{z}(\tilde{x}_T) \in \mathbb{R}^D $)}
%	\label{alg:marginal_learning_algo}
%    \caption{Marginal Transfer Learning}
%\end{algorithm}

\section{Experiments}

This section empirically compares our marginal transfer learning method with
pooling.\footnote{Code available at https://github.com/aniketde/DomainGeneralizationMarginal}
One implementation of the pooling algorithm was mentioned in Section \ref{sec:relkern}, where
$k_P$ is taken to be a constant kernel. Another implementation is to put all the training data
sets together and train a single conventional kernel method. The only difference between the
two implementations is that in the former, weights of $1/n_i$ are used for examples from
training task $i$. In almost all of our experiments below, the various training tasks have the
same sample sizes, in which case the two implementations coincide. The only exception is the
fourth experiment when we use all training data, in which case we use the second of the two
implementations mentioned above.

We consider three classification problems $(\cY = \{-1,1\})$, for which the hinge loss is employed, and
one regression problem $(\cY \subset \mbr)$, where the $\epsilon$-insensitive loss is employed. Thus, the
algorithms implemented are natural extensions of support vector classification and regression to domain generalization. Performance of a learning strategy is assessed by holding out several data sets $S^T_1, \ldots, S^T_{N_T}$, learning a decision function $\wh{f}$ on the remaining data sets, and reporting the average empirical risk $\frac1{N_T} \sum_{i=1}^{N_T} \cL(S_i^T,\wh{f})$. In some cases, this value is again averaged over several randomized versions of the experiment.

%  \item Multi-task SVMs\citep{evgeniou05}:
%  These kernels have the form $k_P(P_1, P_2) = 1$ if $P_1 = P_2$, and $k_P(P_1, P_2) = \tau$ otherwise. When $\tau = 1$, the method is equivalent to pooling. When $0 < \tau < 1$, we obtain a kernel like what was used for multi-task learning (MTL) by \citep{evgeniou05}. Note that these kernels have the property that if $P_1$ is a training data set, and $P_2$ a test data set, then $P_1 \ne P_2$ and so $k_P(P_1, P_2)$ is simply a constant.  This implies that the learning rules produced by these kernels do not adapt to the test distribution, unlike the proposed kernel which depends on the similarity between $P_1$ and $P_2$.
%\end{itemize}

\subsection{Model Selection}
The various experiments use different combinations of kernels.
In all experiments, linear kernels $k({x_1},{x_2})= x_1^{T}{x_2}$ and
Gaussian kernels $k_{\sigma}({x_1},{x_2})=\exp\big(-\frac{||{x_1}-{x_2}||^2}{2\sigma^2}\big)$ were
used.

The bandwidth $\sigma$ of each Gaussian kernel and the regularization parameter $\lambda$
of the machines were selected by grid search. For model selection, five-fold
cross-validation was used. In order to stabilize the cross-validation procedure, it
was repeated $5$ times over independent random splits into folds \citep{kohavi1995study}.
Thus, candidate parameter values were evaluated on the $5 \times 5$ validation sets and the
configuration yielding the best average performance was selected. If any of the chosen
hyper-parameters was at the grid boundary, the grid was extended accordingly, i.e., the
same grid size has been used, however, the center of grid has been assigned to the
previously selected point. The grid used for kernels was $\sigma\in \big( 10^{-2},
10^{4}\big) $ with logarithmic spacing, and the grid used for the regularization parameter
was $\lambda\in\big( 10^{-1}, 10^{1}\big)$ with logarithmic spacing.

%For MTL method we find the best $\tau \in \{0.01, 0.05, 0.1, 0.25, 0.5, 0.75, 0.9, 0.95, 0.99, 1 \}$.

%\begin{table}[htbp]
%  \centering
%  \caption{The kernels and hyper-parameters that are used by each method are listed. Further, the used kernel in the experiments are also defined.}
%    \begin{tabular}{rcccccc}
%    \toprule
%          & \textbf{$k_X(X_i,X^{'}_j)$} & \textbf{$k^{'}_X(X_i,X^{'}_j)$} &
%          \textbf{$\fK\paren{\Psi(\wh{P}_X),\Psi(\wh{P}_X')}$} & \textbf{$\sigma$} &
%          \textbf{$\lambda$} \\
%    \midrule
%    \textbf{Marginal} & Gaussian & Gaussian & Gaussian &  \textbf{-} & Yes \\
%    \textbf{Pooling} & Gaussian & \textbf{-} & \textbf{-} & \textbf{-} & Yes\\
%    \textbf{MTL} & Gaussian & \textbf{-} & \textbf{-} & Yes  & Yes \\
%    \bottomrule
%    \end{tabular}%
%  \label{tab:kernelParamTable}%
%\end{table}%

\subsection{Synthetic Data Experiment}

To illustrate the proposed method, a synthetic problem was constructed. The synthetic data generation
algorithm is given in Algorithm~\ref{alg:ToyDataGeneration}. In brief, for each classification task, the
data are uniformly supported on an ellipse, with the major axis determining the labels, and the rotation
of the major axis randomly generated in a 90 degree range for each task. One random realization of this
synthetic data is shown in Figure \ref{fig:ToyProblem}. This synthetic data set is an ideal candidate for marginal transfer learning, because the Bayes classifier for a task is uniquely determined
by the marginal distribution of the features, i.e. Lemma~\ref{lem:determ} applies (and the optimal error
$\inf_f \cE^\infty(f)$ is zero). On the other hand, observe that the expectation of each $X$ distribution is
the same regardless of the task and thus does not provide any relevant information, so that taking into
account at least second order information is needed to perform domain generalization.

To analyse the effects of number of examples per task ($n$) and number of tasks ($N$), we
constructed $12$ synthetic data sets by taking combinations $N \times n$ where $N \in
\{16,64, 256\} $ and $n\in\{8,16,32,256\}$. For each synthetic data set, the
test set contains 10 tasks and each task contains one million data points. All kernels are taken to be Gaussian, and the random Fourier features speedup is used. The results are shown in Figure \ref{fig:synthetic_results} and Tables \ref{table:synthetic_marginal} and \ref{table:synthetic_pooling} (see appendix). The marginal transfer learning (MTL) method
significantly outperforms the baseline pooling method. Furthermore, the performance of
MTL improves as $N$ and $n$ increase, as expected. The pooling
method, however, does no better than random guessing regardless of $N$ and $n$.

In the remaining experiments, the marginal distribution does not perfectly characterize the
optimal decision function, but still provides some information to offer improvements over
pooling.

\begin{algorithm}
    \SetKwInOut{Input}{input}
    \SetKwInOut{Output}{output}
    \Input{$N$: Number of tasks, $n$: Number of training examples per task}
    \Output{Realization of synthetic data set for $N$ tasks}
%       Sample $N$ random rotation angles uniformly $\in
%	  \Big[\dfrac{\pi}{4} ,\dfrac{3 \pi}{4}\Big]$\;	
		 \For{$i = 1$ to $N$}
		 {
		 	\begin{itemize}
		 	  \setlength\itemsep{0em}
                          \item sample rotation $\alpha_i$ uniformly in $\Big[\dfrac{\pi}{4} ,\dfrac{3 \pi}{4}\Big]$;
		 		\item Take an ellipse whose major axis is aligned with the horizontal axis, and \\ rotate  it by an angle of $\alpha_i$ about its center;
			  	\item Sample $n$ points $X_{ij}$, $j=1,\ldots,n$
                                  uniformly at random from the rotated \\ ellipse;
				\item Label the points according to their position with respect to the major axis, \\ i.e. the points that are on
		  	the right of the major axis are considered as class \\ $1$ and the points on the left
		  	of the major axis are considered as class $-1$.
			\end{itemize}
		 }
    \caption{Synthetic Data Generation 	\label{alg:ToyDataGeneration}}
\end{algorithm}

\begin{figure}[htb!]
	\hspace{1.5 cm}\begin{minipage}{0.18\textwidth}
	\subfloat[]{\label{main:ab}\includegraphics[width=6cm]{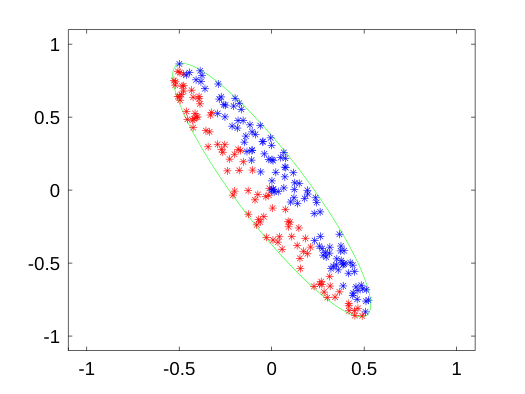}}
	\end{minipage}\hspace{3.75 cm}
	\begin{minipage}{0.18\textwidth}
	\subfloat[]{\label{main:bb}\includegraphics[width=6cm]{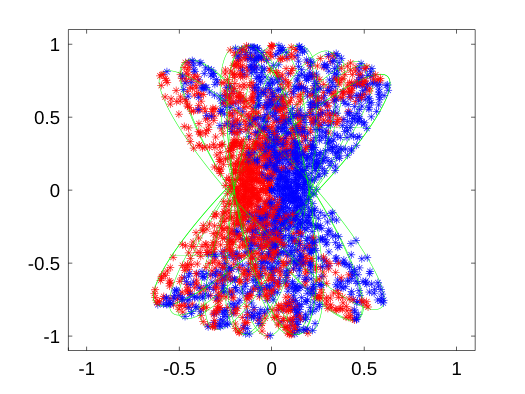}}
	\end{minipage}\\
	
	\hspace{1.5 cm}\begin{minipage}{0.18\textwidth}
	\subfloat[]{\label{main:cb}\includegraphics[width=6cm]{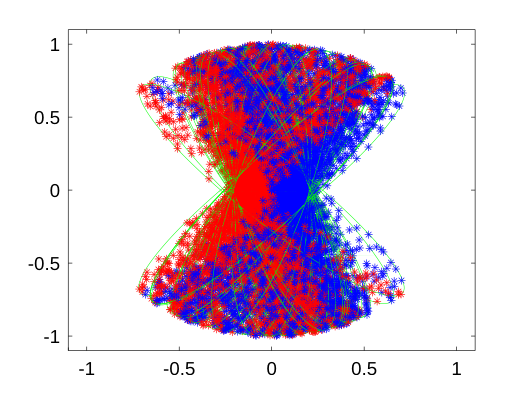}}
	\end{minipage}\hspace{3.75 cm}
	\begin{minipage}{0.18\textwidth}
	\subfloat[]{\label{main:db}\includegraphics[width=6cm]{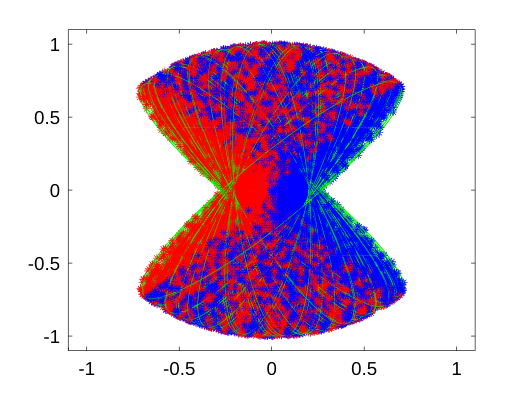}}
	\end{minipage}
	 \caption{Plots of synthetic data sets (red and blue points represent negative and positive classes) for different
settings: (a) Random realization of a single task with 256 training examples per task. Plots (b), (c) and(d)  are random
realizations of synthetic data with 256 training examples for 16, 64 and 256 tasks.}
	 \label{fig:ToyProblem}%
\end{figure}

\begin{figure}[htb!]
\centering
	\includegraphics[width=0.9\columnwidth]{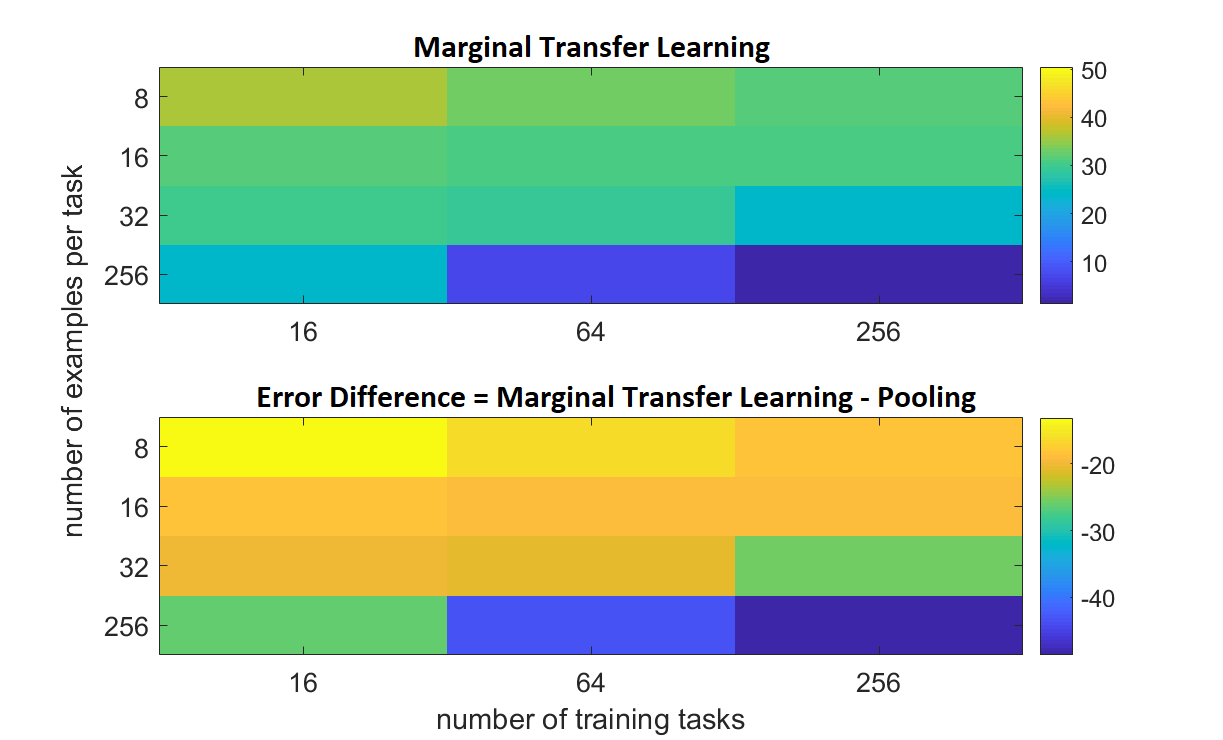}	
	\caption{Synthetic data set: Classification error rates for proposed method and difference with baseline for different experimental settings, i.e., number of examples per task and number of tasks. }
	\label{fig:synthetic_results}
\end{figure}

\subsection{Parkinson's Disease Telemonitoring}\label{subSubSection_Parkinson}
We test our method in the regression setting
using the Parkinson's disease telemonitoring data set, which is composed of a range of biomedical voice measurements using a
telemonitoring device from 42 people with early-stage Parkinson's. The recordings were automatically captured in the
patients' homes. The aim is to predict the clinician's Parkinson's disease symptom score for each recording
on the unified Parkinson's disease
rating scale (UPDRS) \citep{tsanas2010accurate}. Thus we are in a regression setting, and employ the $\epsilon$-insensitive loss from support vector
regression. All kernels are taken to be Gaussian, and the random Fourier features speedup is used.

There are around 200 recordings per patient. We randomly select 7 test users and then vary the number of training users $N$
from 10 to 35 in steps of 5, and we also vary the number of training examples $n$ per user from 20 to 100. We repeat this
process several times to get the average errors which are shown in Fig~\ref{fig:parkinson} and Tables
\ref{table:parkinson_marginal} and \ref{table:parkinson_pooling} (see appendix). The marginal transfer learning method clearly outperforms pooling,
especially as $N$ and $n$ increase.

%When number of training tasks are less and number of examples per training task are low, then estimation of marginal
%probability distribution is erroneous. So, we expect proposed method to perform well as number of training tasks and number
%of examples per training task increase. Fig \ref{fig:parkinson} shows that proposed method starts performing better when
%the
%number of training tasks are 15. When the number of training tasks is 35, the proposed method performs significantly better
%than pooling for any number of training examples per task used.
\begin{figure}[htb!]
\centering
	\includegraphics[width=0.8\columnwidth]{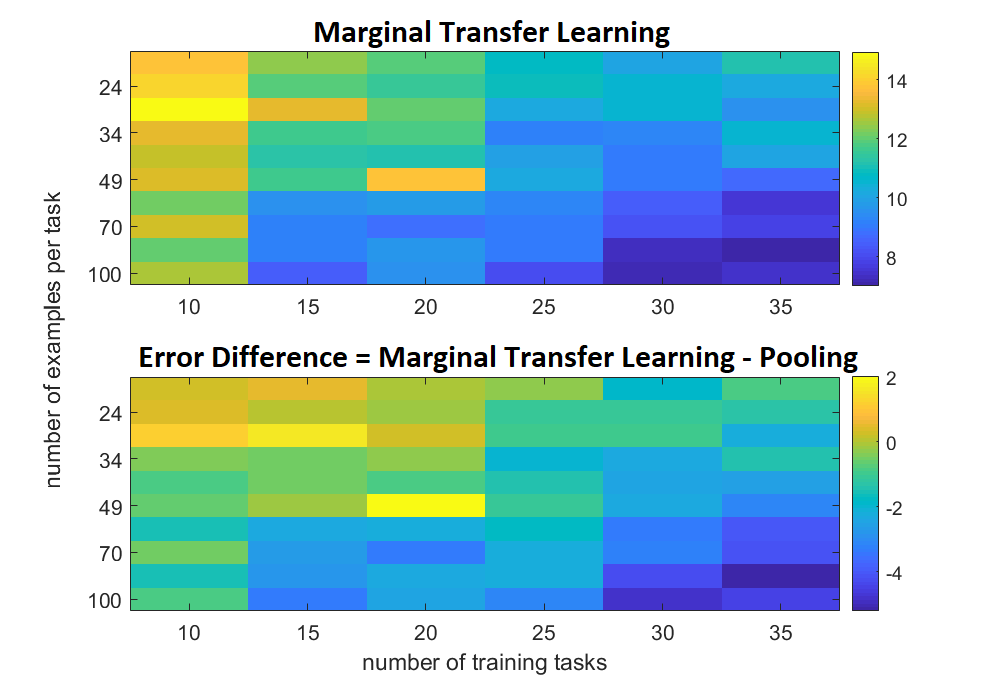}
	\caption{Parkinson's disease telemonitoring data set: Root mean square error rates for proposed method and difference with baseline for different experimental settings, i.e., number of examples per task and number of tasks. 	\label{fig:parkinson}}
\end{figure}

\subsection{Satellite Classification}\label{Satellite_classification}

Microsatellites are increasingly deployed in space missions for a variety of scientific and technological
purposes. Because of randomness in the launch process, the orbit of a microsatellite is random, and must
be determined after the launch. One recently proposed approach is to estimate the orbit of a satellite
based on radiofrequency (RF) signals as measured in a ground sensor network. However, microsatellites are
often launched in bunches, and for this approach to be successful, it is necessary to associate each RF
measurement vector with a particular satellite. Furthermore, the ground antennae are not able to decode
unique identifier signals transmitted by the microsatellites, because (a) of constraints on the
satellite/ground antennae links, including transmission power, atmospheric attenuation, scattering, and
thermal noise, and (b) ground antennae must have low gain and low directional specificity owing to
uncertainty in satellite position and dynamics. To address this problem, recent work has proposed to apply
our marginal transfer learning methodology \citep{sharma2015robust}.

As a concrete instance of this problem, suppose two microsatellites are launched together. Each launch is
a random phenomenon and may be viewed as a task in our framework. For each launch $i$, training data
$(X_{ij}, Y_{ij})$, $j = 1, \ldots, n_i$, are generated using a highly realistic simulation model, where
$X_{ij}$ is a feature vector of RF measurements across a particular sensor network and at a particular
time, and $Y_{ij}$ is a binary label identifying which of the two microsatellites produced a given
measurement. By applying our methodology, we can classify unlabeled measurements $X^T_j$ from a new launch
with high accuracy. Given these labels, orbits can subsequently be estimated using the observed RF
measurements. We thank Srinagesh Sharma and James Cutler for providing us with their simulated data, and
refer the reader to their paper for more details on the application \citep{sharma2015robust}.

To demonstrate this idea, we analyzed the data from \citet{sharma2015robust} for $T=50$ launches, viewing up to
40 as training data and 10 as testing. We use Gaussian kernels and the RFF kernel approximation technique
to speed up the algorithm. Results are shown in Fig~\ref{fig:satellite} (tables given in the appendix). As expected, the error for the
proposed method is much lower than for pooling, especially as $N$ and $n$ increase.

\begin{figure}[htb!]
\centering
	\includegraphics[width=0.8\columnwidth]{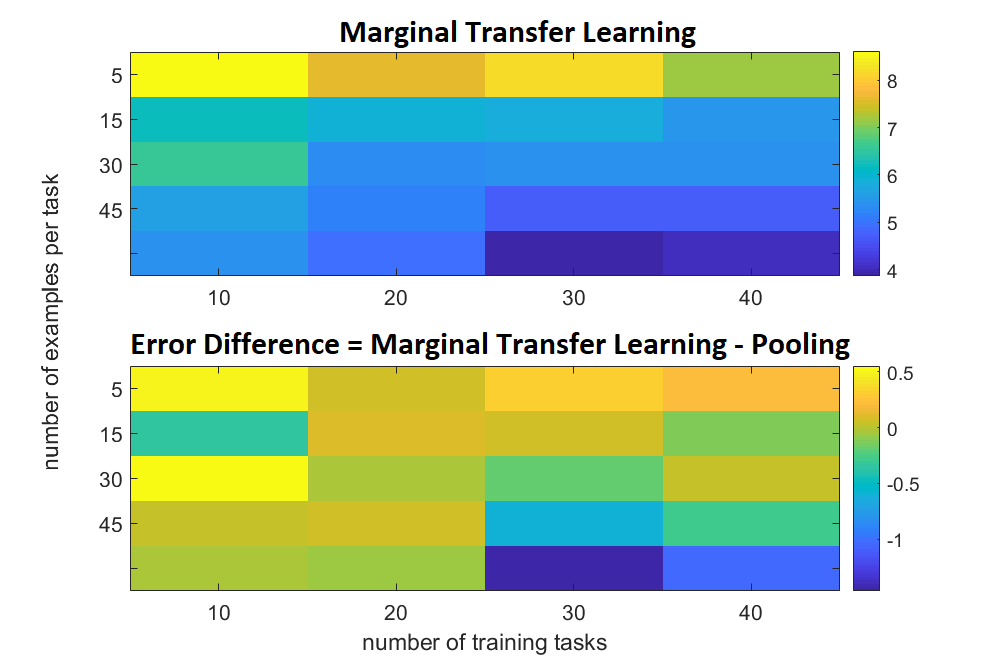}	
	\caption{Satellite data set: Classification error rates for proposed method and difference with baseline for different experimental settings, i.e., number of examples per task and number of tasks. \label{fig:satellite}}
\end{figure}

\subsection{Flow Cytometry Experiments}\label{Flow_Exp}

We demonstrate the proposed methodology for the flow cytometry auto-gating problem, described in Sec. \ref{sec:gating}. The
pooling approach has been previously investigated in this context by \citet{toedling06bioinfo}. We used a data set that is a
part of the FlowCAP Challenges where the ground truth labels have been supplied by human experts \citep{aghaeepour2013critical}.
We used the so-called ``Normal Donors" data set. The data set contains 8 different classes and 30 subjects. Only two classes
(0 and 2) have consistent class ratios, so we have restricted our attention to these two.

The corresponding flow cytometry data sets have sample sizes ranging from 18,641 to 59,411, and the proportion of class 0 in
each data set ranges from 25.59 to 38.44\%. 
%We randomly selected 10 tasks as test tasks and used exactly the same tasks over all experiments. 
We  randomly  selected  10  tasks to serve as the test tasks. These tasks were removed from the pool of eligible training tasks.
We varied the number of training tasks from 5 to 20 with an additive step size of 5, and the number
of training examples per task from 1024 to 16384 with a multiplicative step size of 2. We repeated this process 10 times to
get the average classification errors which are shown in Fig. \ref{fig:flow_error} and Tables \ref{table:flow_marginal} and
\ref{table:flow_pooling} (see appendix). The kernel $k_P$ was Gaussian, and the other two were linear. The Nystr{\"o}m approximation was used to achieve an efficient implementation.

For nearly all settings the proposed method has a smaller error rate than the baseline. Furthermore, for
the marginal transfer learning method, when one fixes the number of training examples and increases the number of
tasks then the classification error rate drops.

On the other hand, we observe on Table~\ref{table:flow_marginal} that the number $n$ of training points per task hardly affects the final performance when $n\geq 10^3$. This is in contrast with the previous experimental examples (synthetic, Parkinson's disease telemonitoring, and satellite classification), for which increasing $n$ led to better performance, but where the values of $n$ remained somewhat modest ($n\leq 256$). This is qualitatively in line with the theoretical results under {\bf (2SGM)} in Section~\ref{sec:theory2sgm} (see in particular the concluding discussion there), suggesting that the influence of increasing $n$ on the performance should eventually taper off, in particular if $n\gg N$.

\begin{figure}[htb!]
\centering
	\includegraphics[width=0.8\columnwidth]{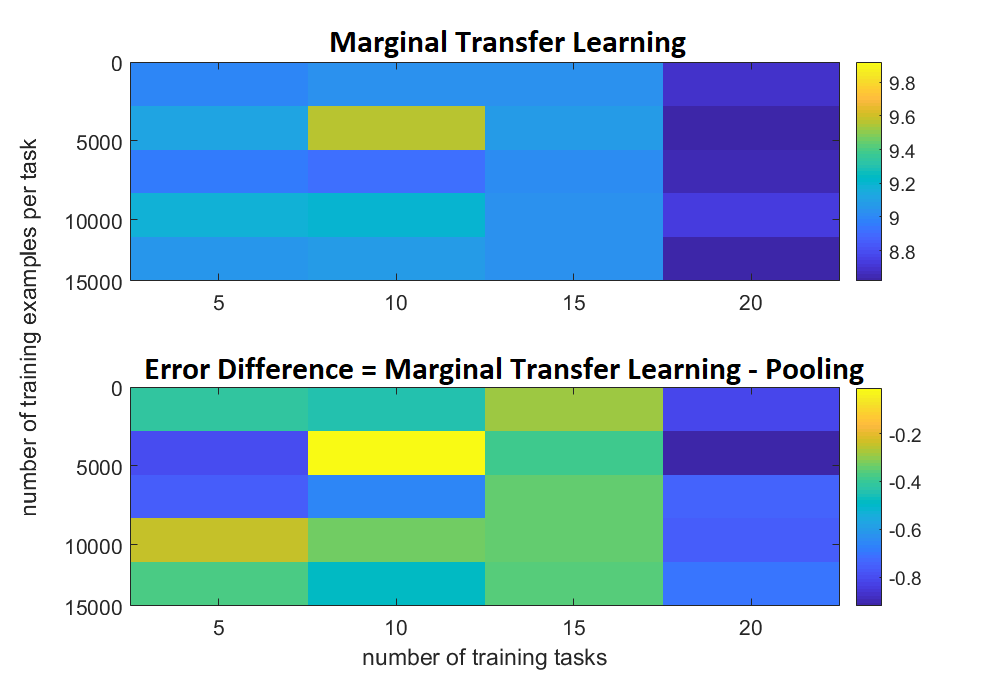}
	\caption{Flow Cytometry Data set: Percentage Classification error rates for proposed method and difference with baseline for different experimental settings, i.e., number of examples per task and number of tasks.
        	\label{fig:flow_error}}
\end{figure}

\section{Discussion}
\label{sec:disc}

Our approach to domain generalization relies on the extended input
pattern $\tX = (P_X,X)$. Thus, we study the natural algorithm of
minimizing a regularized empirical loss over a reproducing kernel Hilbert
space associated with the extended input domain $\fP_{\cX} \times \cX$. We
also establish universal consistency under two sampling plans.  To achieve this, we present novel
generalization error analyses,
and construct a universal kernel on $\fP_{\cX} \times \cX$. A detailed implementation based on novel approximate feature
mappings is also presented. 

On one synthetic and three real-world data sets, the marginal transfer learning approach consistently outperforms a pooling baseline. On some data sets, however, the difference between the two methods is small. This is because the utility of  transfer learning varies from one DG problem to another. As an extreme example, if all of the task are the same, then pooling should do just as well as our method. 

Several future directions exist.  From an application perspective, the
need for adaptive classifiers arises in many applications, especially in
biomedical applications involving biological and/or technical variation in
patient data. Examples include brain computer interfaces and patient
monitoring.  For example, when electrocardiograms are used to continuously
monitor cardiac patients, it is desirable to classify each heartbeat as
irregular or not.  Given the extraordinary amount of data involved,
automation of this process is essential.  However, irregularities in a
test patient's heartbeat will differ from
irregularities of historical patients, hence the need to adapt to the test
distribution \citep{wiens10}.

From a theoretical and methodological perspective, several questions are of interest.  We would like to specify conditions on the meta-distributions $P_S$ or $\mu$ under which the DG risk is close to the expected Bayes risk of the test distribution (beyond the simple condition discussed in Lemma~\ref{lem:determ}). We would also like to develop fast learning rates under suitable distributional assumptions. Furthermore, given the close connections with supervised learning, many common variants of supervised learning can also be investigated in the DG context, including multiclass classification, class probability estimation, and robustness to various forms of noise.

 We can also ask how the methodology and analysis can be
 extended to the context where a
 small number of labels are available for the test distribution
 (additionally to a larger number of unlabeled data from the same distribution); this situation appears to be common in practice, and can be seen as intermediary between the DG and learning to learn (LTL, see Section~\ref{se:ltl}) settings (one could dub it ``semi-supervised domain generalization"). 
 In this setting, two
 approaches appear promising to take advantage of the labeled data. The simplest one is to use the same optimization
 problem \eqref{eq:est}, where we include additionally the labeled
examples of the test distribution. However, if several test samples
 are to be treated in succession, and we want to avoid a full,
 resource-consuming re-training using all the training samples each time,
an interesting alternative is the following:
 learn once a function $f_0(P_X,x)$ using the available training samples via
 \eqref{eq:est}; then, given a partially labeled test sample, learn
 a decision function on this sample only via the usual kernel
 ($k_X$) norm
 regularized empirical loss minimization
 method, but replace the usual regularizer term $\norm{f}^2_\cH$ by
 $\norm{f-f_0(\wh{P}_X^T,.)}_\cH^2$ (note that $f_0(\wh{P}_X^T,.) \in \cH_{k_X}$). In this
 sense,
 the marginal-adaptive decision function learned from the training samples
 would serve as a ``prior'' or ``informed guess" for learning on the test data. This can be also interpreted as learning an adequate complexity penalty to improve learning on new samples, thus connecting to the general principles of LTL (see Section~\ref{se:ltl}).
 An interesting difference with underlying existing LTL approaches is that those tend to adapt the hypothesis class or the``shape" of the regularization penalty to the problem at hand, while the approach delineated above would modify the ``origin" of the penalty, using the marginal distribution information. These two principles could also be combined.

%Future work may also consider the multiclass setting, and other asymptotic regimes, e.g., where $\{n_i\}, n_T$ do not tend to infinity, or they tend to infinity much slower than $N$.

\acks{C. Scott and A. Deshmukh were supported in part by NSF Grants No. 1422157, 1217880, 1047871, 1838179, and 2008074.
G. Blanchard was supported in part by the Deutsche Forschungsgemeinschaft (DFG) via Research Unit FG1735/2 {\it Structural Inference in Statistics}, and via
SFB1294/1 - 318763901DFG {\it Data Assimilation}; and by the Agence Nationale de la Recherche (ANR) via the IA Chair {\it BiSCottE}. G. Lee was supported by the National Research Foundation of Korea
(NRF-2014R1A1A1003458). }

\appendix
\section{Proofs, Technical Details, and Experimental Details}

\label{sec:proofs}

This appendix contains the remaining proofs, as well as additional technical and experimental details.

\subsection{Proof of Proposition~\ref{prop:conv_err_2gsm}}

Let $P^T_{XY}$ be a fixed probability distribution on $\cX \times \mbr$, and $\eps>0$ a fixed number.
Since $\cX$ is a Radon space, by definition any Borel probability measure on it, in  particular $P^T_{X}$ (the $X$-marginal of $P^T_{XY}$), is inner regular, so that there exists a compact set $K \subset \cX$ such that $P^T_{X}(K^c) \leq \eps.$

For all $x\in K$, by the assumed continuity of the decision function $f$ at point $(P^T_X,x)$ there exists an open neighborhood $U_x \times V_x \subset \fP_{\cX} \times \cX$ of this point such that $\abs{f(u,v) - f(P^T_X,x)} \leq \eps$ for all $(u,v) \in U_x\times V_x$. Since the family $(V_{x})_{x\in K}$ is an open covering of the
compact $K$, there exists a finite subfamily $(V_{x_i})_{i \in I}$ covering $K$.
Denoting $U_0 := \bigcap_{i \in I} U_{x_i}$ which is an open neighborhood of $P^T_X$ in
$\fP_{\cX}$, it therefore holds for any $P \in U_0$ and uniformly over
$x \in K$ that $\abs{f(P,x) - f(P_X^T,x)} \leq \abs{f(P,x) - f(P,x_{i_0})} + \abs{f(P_X^T,x) - f(P,x_{i_0})} \leq 2 \eps$, where $i_0 \in I $ is such that $x\in V_{x_{i_0}}$.

Denote $S^T=(X^T_i,Y^T_i)_{1\leq i \leq n_T}$ a sample of size $n_T$ drawn i.i.d. from $P^T_{XY}$,
%let $A$ be the event $\set{X_i^T \in K, i=1,\ldots,n_T }$, which is of probability
%at least $(1-\eps)^{n_T}$.
and $A$ the event $\set{\wh{P}_X^T \in U_0}$. By the law of large numbers, $\wh{P}_X^T$ weakly
converges to $P_X^T$ in probability, so that $\prob{A^c}\leq \eps$ holds for $n_T$ large enough.
We have (denoting $B$ a bound on the loss function):
\begin{multline*}
  \ee{S_T}{\frac{1}{n_T} \sum_{i=1}^{n_T} \ell(f(\wh{P}_X^T,X_i^T),Y_i^T)}
  % & \leq B (1 - (1-\eps)^{n_T}) +  \ee{S_T}{\ind{A} \frac{1}{n_T} \sum_{i=1}^{n_T} \ell(f(\wh{P}_X^T,X_i^T),Y_i^T)}\\
     \leq B \eps  + \ee{S_T}{\frac{1}{n_T} \sum_{i=1}^{n_T} \ell(f(\wh{P}_X^T,X_i^T),Y_i^T)\ind{X_i^T \in K}}\\
\begin{aligned}
    & \leq B (\eps + \prob{A^c})  + \ee{S_T}{\ind{\wh{P}_X^T \in U_0} \frac{1}{n_T} \sum_{i=1}^{n_T} \ell(f(\wh{P}_X^T,X_i^T),Y_i^T)\ind{X_i^T \in K}}\\
    & \leq 2B\eps + 2L\eps + \ee{S_T}{\frac{1}{n_T} \sum_{i=1}^{n_T} \ell(f(P_X^T,X_i^T),Y_i^T)}\\
    & \leq 2(B+L) \eps + \ee{(X^T,Y^T) \sim P_{XY}^T}{\ell(f(P_X^T,X^T),Y^T)}.
\end{aligned}
\end{multline*}
Conversely,
\begin{align*}
    \ee{S_T}{\frac{1}{n_T} \sum_{i=1}^{n_T} \ell(f(\wh{P}_X^T,X_i^T),Y_i^T)}
  % & \leq B (1 - (1-\eps)^{n_T}) +  \ee{S_T}{\ind{A} \frac{1}{n_T} \sum_{i=1}^{n_T} \ell(f(\wh{P}_X^T,X_i^T),Y_i^T)}\\
%    & \geq \ee{S_T}{\frac{1}{n_T} \sum_{i=1}^{n_T} \ell(f(\wh{P}_X^T,X_i^T),Y_i^T)\ind{X_i^T \in K}}\\
    & \geq \ee{S_T}{\ind{\wh{P}_X^T \in U_0} \frac{1}{n_T} \sum_{i=1}^{n_T} \ell(f(\wh{P}_X^T,X_i^T),Y_i^T)\ind{X_i^T \in K}}\\
    & \geq \ee{S_T}{\frac{1}{n_T} \sum_{i=1}^{n_T} \ell(f(P_X^T,X_i^T),Y_i^T)} - 2 B\eps -2 L\eps\\
   & \geq   \ee{(X^T,Y^T) \sim P_{XY}^T}{\ell(f(P_X^T,X^T),Y^T)} - 2(B+L) \eps.    
\end{align*}
Since the above inequalities hold for any $\eps>0$ provided $n_T$ is large enough, this yields that
for any fixed $P_{XY}^T$, we have \[\lim_{n_T \rightarrow \infty}  \ee{S_T \sim (P_{XY}^T) ^{\otimes n_T}}{\frac{1}{n_T} \sum_{i=1}^{n_T} \ell(f(\wh{P}_X^T,X_i^T),Y_i^T)} = \ee{(X^T,Y^T) \sim P_{XY}^T}{\ell(f(P_X^T,X^T),Y^T)}.\]
Finally, since the above right-hand side is bounded by $B$, applying dominated convergence to
integrate over $P_{XY}^T \sim \mu$ yields the desired conclusion.

\subsection{Proof of Corollary \ref{cor:const}}

\begin{proof}
Denote $\cE^* = \inf_{f: {\fP}_X \times \cX \to \mbr} \cE(f)$. Let
$\eps > 0$. Since $\ol{k}$ is a universal kernel on $\fP_{\cX} \times \cX$ and
$\ell$ is Lipschitz, there exists $f_0 \in \cH_{\ol{k}}$ such that
$\cE(f_0) \le \cE^* + \frac{\eps}2$ \citep{steinwart08}.

% Let us introduce the shorthand notation
% $$
% \wh{\cE}_\ell(f,N,n) = \frac1{N}\sum_{i=1}^N \frac1{n_i} \sum_{j=1}^{n_i}
% \ell(f(\tilde{X}_{ij}),Y_{ij}).
% $$
%We will also write $\cE(f,\infty) = \cE_\ell(f,\infty)$ to indicate the
%dependence of the risk on the loss.

By comparing the objective function in~\eqref{eq:est} at the
minimizer
$\wh{f}_\lambda$ and at
the null function, using assumption {\bf (LB)} we deduce that we must have $\| \wh{f}_{\lambda} \|
\leq \sqrt{B_0/\lambda}$.
%Let us denote $R_{\lambda} = \sqrt{\frac{B_0}{\lambda}}$.
%and $M_\lambda =  B_{\fK} B_k R_{\lambda}$.
% Define the truncated loss
% $$
% \ell_\lambda(t,y) := \left\{
% \begin{array}{ll}
% \ell(M_\lambda,y), & t > M_\lambda, \\
% \ell(t,y), & -M_\lambda \le t \le M_\lambda, \\
% \ell(-M_\lambda,y), & t < -M_\lambda.
% \end{array}
% \right.
% $$
% We will write $\cE_{\ell_\lambda}$ and $\wh{\cE}_{\ell_\lambda}(f,N,n)$ to
% denote the modified true and empirical risks when the loss is
% $\ell_\lambda$. This loss is easily seen to be $L_\ell$-Lipschitz. It is
% also bounded, which can be seen by noting that it suffices to bound
% $\ell_\lambda(t,y)$ for $(t,y) \in [-M_\lambda,M_\lambda] \times \cY$, in
% which case we have $\ell_\lambda(t,y)
% = \ell(t,y) \le B_0 + M_\lambda L_\ell$ by the definition of $B_0$
% and the fact that $\ell$ is $L_\ell$-Lipschitz.
Applying Theorem~\ref{th:main} for $R = R_\lambda =
\sqrt{B_0/\lambda}$,
%$B_{\ell_\lambda} = B_0 + L_\ell
%B_{\fK} B_k R_\lambda$,
and $\delta  = 1/N^2$,
gives that with probability at least $1-1/N^2$,
\[
  \sup_{f \in \cB_{\ol{k}}(R)} \abs{     \wh{\cE}(f,N) - \cE(f)}
  \leq \eps(N):=
  (B_0 + L_\ell B_\fK B_k \sqrt{{B_0}/{\lambda}}) \frac{(\sqrt{\log N}+2)}{\sqrt{N}}.
\]
%  \sup_{f \in B_{\ol{k}(R_\lambda)}} \abs{     \wh{\cE}(f,N) - \cE(f)}
%  \leq
%  \frac{2 L_\ell R_\lambda B_k B_\fK}{\sqrt{N}}
%  + (B_0 + L_\ell R_\lambda B \sqrt{\frac{\log \delta^{-1}}{2N}}.
%\]
% \sup_{f \in B_{\ol{k}(R_\lambda)}} \left| \wh{\cE}_{\ell_\lambda}(f,N,n) -
% \cE_{\ell_\lambda}(f,\infty) \right| \le \abs{     \wh{\cE}(f,N) - \cE(f)}
%   \leq
%   \frac{2 L_\ell R_\lambda B_k B_\fK}{\sqrt{N}}
% + (B_0 + L_\ell R_\lambda B \sqrt{\frac{\log \delta^{-1}}{2N}}.
% where
% $$
% \eps(N,n):=
% \frac{C_1}{\sqrt{\lambda}} \left(\frac{\log N + \log
% n}{n} \right)^{\frac{\alpha}2} + \frac{C_2}{\sqrt{\lambda N}} +
% \left(C_3 + \frac{C_4}{\sqrt{\lambda}} \right) \sqrt{\frac{\log N + \log
% n}{N}}.
% $$

% The following result captures the key property that on the ball
% $\cB_{\ol{k}}(R_\lambda)$, $\ell$ and $\ell_\lambda$ yield the same true
% and empirical risks.
% \begin{lemma}
% $\forall f \in \cB_{\ol{k}}(R_\lambda)$, $\cE_\ell(f,\infty) =
% \cE_{\ell_\lambda}(f,\infty)$ and
% $\wh{\cE}_{\ell}(f,N,n) =
% \wh{\cE}_{\ell_\lambda}(f,N,n)$.
% \end{lemma}
% \begin{proof}
% If $f \in \cB_{\ol{k}}(R_\lambda)$, then the reproducing property and
% Cauchy-Schwarz imply that for an arbitrary $\tilde{x}$, $| f(\tilde{x})
% | \le B_{\fK} B_k R_\lambda = M_\lambda$. The result now follows from the
% definitions of $\ell_\lambda$ and of the true and empirical risks.
% \end{proof}

Let $N$ be large enough so that $\| f_0 \| \le R_\lambda$.
We can now deduce that with probability at least $1-1/N^2$,
\begin{align*}
  \cE(\wh{f}_\lambda)
  %&=
%\cE_{\ell_\lambda}(\wh{f}_\lambda,\infty) \\
&\le \wh{\cE}(\wh{f}_\lambda,N) + \eps(N) \\
%&= \wh{\cE}_{\ell}(\wh{f}_\lambda,N,n) + \eps(N,n) \\
&=\wh{\cE}(\wh{f}_\lambda,N) + \lambda \| \wh{f}_\lambda \|^2
- \lambda \| \wh{f}_\lambda \|^2 + \eps(N) \\
&\le \wh{\cE}(f_0,N) + \lambda \| {f_0}\|^2
- \lambda \| \wh{f}_\lambda \|^2 + \eps(N) \\
&\le \wh{\cE}(f_0,N) + \lambda \| {f_0}\|^2 + \eps(N) \\
%&= \wh{\cE}_{\ell_\lambda}(f_0,N,n) + \lambda \| {f_0}\|^2 + \eps(N,n)
%\\
&\le \cE(f_0) + \lambda \| {f_0}\|^2 + 2\eps(N) \\
%&= \cE_{\ell}(f_0,\infty) + \lambda \| {f_0}\|^2 + 2\eps(N,n) \\
&\le \cE^* + \frac{\eps}2 + \lambda \| {f_0}\|^2 + 2\eps(N). \\
\end{align*}
The last two terms become less than $\frac{\eps}2$ for $N$ sufficiently
large by the assumptions on the growth of %$N$, $n$, and
$\lambda = \lambda(N)$. This establishes that for any $\eps > 0$,
there exists $N_0$ such that
\[
\sum_{N \ge N_0}  \Pr(\cE(\wh{f}_\lambda) \ge
\cE^* + \eps) \le \sum_{N \ge N_0}  \frac1{N^2} <
\infty,
\]
and so the result follows by the Borel-Cantelli lemma.
\end{proof}

\subsection{Proof of Theorem~\ref{th:tworisks}}

We control the difference between the training loss
and the conditional risk at infinite sample size via
the following decomposition:
\begin{align}
  \sup_{f \in \cB_{\ol{k}}(R)} |\cL(S,f) &- \cE^\infty(f|P_{XY})|
   = \sup_{f \in \cB_{\ol{k}}(R)} \abs{ \frac{1}{n} \sum_{i=1}^{n}
\ell(f(\wh{P}_X ,X_i),Y_{i}) - \cE^\infty(f|P_{XY})} \notag \\
%\begin{aligned}
& \leq \sup_{f \in \cB_{\ol{k}}(R)}
 \abs{ \frac{1}{n} \sum_{i=1}^{n}
\paren{ \ell(f(\wh{P}_X,X_{i}),Y_{i}) -
\ell(f(P_X,X_{i}),Y_{i})}} \notag
\\
& \;\;\;\;
+ \sup_{f \in \cB_{\ol{k}}(R)} \abs{ \frac{1}{n}
\sum_{i=1}^{n} \ell(f(P_X,X_{i}),Y_{i})
- \cE^\infty(f|P_{XY})} \notag\\
& =: (I) + (II). \label{eq:term1-2}
%\end{aligned}
\end{align}

\subsubsection{Control of Term (I)}

Using the assumption that the loss $\ell$ is $L_\ell$-Lipschitz
in its first coordinate,
we can bound the first term as follows:
\begin{equation}
\label{eq:term1}
(I)  \leq L_\ell \sup_{f \in \cB_{\ol{k}}(R)}  
\frac{1}{n} \sum_{i=1}^{n} \abs{f(\wh{P}_X,X_{i}) - f(P_X,X_{i})}
 \leq L_\ell  \sup_{f \in \cB_{\ol{k}}(R)} 
\norm{f(\wh{P}_X,\cdot) - f(P_X,\cdot)}_\infty.
\end{equation}
This can now be controlled using the first part of the
following result:
\begin{lemma}
Assume \kbound/ holds.
Let $P_X$ be an arbitrary
distribution on $\cX$ and
$\wh{P}_X$ denote an empirical distribution on $\cX$ based on an iid
sample
of size $n$ from $P_X$. Then with probability at least $1-\delta$ over the
draw of this sample, it holds that
\begin{equation}
  \label{eq:supf-dev}
\sup_{f \in \cB_{\ol{k}}(R)} \norm{f(\wh{P}_X,\cdot) -
f(P_X,\cdot)}_\infty
\leq 3 R B_k L_\fK B_{k'}^\alpha  \paren{\frac{\log 2\delta^{-1}}{n}}^{\frac{\alpha}{2}}\,.
\end{equation}
In expectation, it holds
\begin{equation}
  \label{eq:supf-exp}
  \e{\sup_{f \in \cB_{\ol{k}}(R)} \norm{f(\wh{P}_X,\cdot) -
f(P_X,\cdot)}_\infty}
\leq 2 R B_k L_\fK B_{k'}^\alpha n^{-\alpha/2}.
\end{equation}
\end{lemma}
\begin{proof}
Let $X_1,\ldots,X_n$ denote the $n$-sample from $P_X$.
Let us denote by $\Phi'_X$ the canonical feature mapping $x \mapsto k'_X(x,\cdot)$ from
$\cX$ into $\cH_{k'_\cX}$.
We have for all $x \in \cX$, $\norm{\Phi_X'(x)} \leq B_{k'}$, and so, as a
consequence of
Hoeffding's inequality in a Hilbert space %\citep{HoeffdingHilbert},
(see, e.g., \citealp{PinSak85}), it holds with
probability at least $1-\delta$:
\begin{equation}
\label{eq:Hoeffdingstyle}
\norm{ \Psi(P_X) - \Psi(\wh{P}_X)}
= \norm{ \frac{1}{n} \sum_{i=1}^n \Phi_X'(X_i) - \ee{X\sim
P_X}{\Phi_X'(X)}}
\leq  3 B_{k'} \sqrt{\frac{\log 2\delta^{-1}}{n}}.
\end{equation}
%(TODO: this must appear in the original Gretton et al. paper; check this).
Furthermore, using the reproducing property of the kernel
$\ol{k}$, we have for any $x\in \cX$ and $f\in \cB_{\ol{k}}(R)$:
\begin{align*}
|f(\wh{P}_X,x) - f(P_X,x)|
& =
\left| \inner{\ol{k}((\wh{P}_X,x),\cdot) - \ol{k}((P_X,x),\cdot) , f} \right| \\
& \leq \norm{f} \norm{\ol{k}((\wh{P}_X,x),\cdot) - \ol{k}((P_X,x),\cdot)}\\
& \leq R k_X(x,x)^{\frac{1}{2}} \Big(\fK(\Psi(P_X),\Psi(P_X)) \\
& \qquad \qquad \qquad  +
\fK(\Psi(\wh{P}_X),\Psi(\wh{P}_X))
- 2 \fK(\Psi(P_X),\Psi(\wh{P}_X))\Big)^{\frac{1}{2}} \\
& \leq R B_k  \norm{\Phi_{\fK}(\Psi(P_X)) - \Phi_{\fK}(\Psi(\wh{P}_X))} \\
& \leq R B_k L_\fK \norm{\Psi(P_X) - \Psi(\wh{P}_X)}^{\alpha},
\end{align*}
where in the last step we have used property \khoeld/ together with the fact that for all $P \in \fP_\cX$,
$\norm{\Psi(P)} \leq \int_{\cX} \norm{ k'_X(x,\cdot)} dP_X(x) \leq B_{k'}$,
so that $\Psi(P) \in \cB_{k'_X}(B_{k'})$.
Combining with \eqref{eq:Hoeffdingstyle} gives~\eqref{eq:supf-dev}.

For the bound in expectation, we use the inequality above,
and can bound further (using Jensen's inequality, since $\alpha \leq 1$)
\begin{multline*}
\e{ \norm{\Psi(P_X) - \Psi(\wh{P}_X)}^{\alpha}}
 \leq \e{ \norm{\Psi(P_X) - \Psi(\wh{P}_X)}^2}^{\alpha/2}\\
\begin{aligned}
& = \paren{\frac{1}{n^2}\sum_{i,j=1}^{n}\e{\inner{\Phi'_X(X_i)-\e{\Phi'_X(X)},\Phi'_X(X_j)-\e{\Phi'_X(X)}}}}^{\alpha/2}\\
& = \paren{\frac{1}{n^2}\sum_{i=1}^{n}\e{\norm{\Phi'_X(X_i)-\e{\Phi'_X(X)}}^2}}^{\alpha/2}\\
& \leq \paren{\frac{4B_{k'}^2}{n}}^{\frac{\alpha}{2}},
\end{aligned}
\end{multline*}
which yields~\eqref{eq:supf-exp} in combination with the above.
\end{proof}
% Conditionally to the draw of
% $(P_X^{(i)})_{1 \leq i \leq N}$,
% we can now apply this lemma
% to each $(P_X^{(i)},\wh{P}_X^{(i)})$
% then the union bound over $i=1,\ldots,N$ to get that with probability
% $1-\delta$ (conditionally to $(P_X^{(i)})_{1 \leq i \leq N}$, and thus also
% unconditionally):
% \[
% (I) \leq 3 R B_k B_{k'} L_\ell L_\fK
% \paren{\frac{\log \delta^{-1} + \log 2N}{n}}^{\frac{\alpha}{2}}\,.
% \]

\subsubsection{Control of Term (II)}

Term (II) takes the form of a uniform deviation over a RKHS ball of
an empirical loss for the data $(\wt{X}_i,Y_i)$, where $\wt{X}_i := (P_X,X_i)$. Since
$P_X$ is fixed (in contrast with term (I)
where $\wh{P}_X$ depended on the whole sample), these data are i.i.d.
Similar to the proofs of Theorems~\ref{th:radgenbound}
and~\ref{th:main}, we can therefore apply again standard Rademacher analysis, this time at the level of one specific task (Azuma-McDiarmid inequality followed by Rademacher complexity analysis for a Lipschitz, bounded
loss over a RKHS ball; see \citealp{kolt01,BarMen02}, Theorems~8, 12 and Lemma~22 there).
The kernel $\ol{k}$ is bounded by $B_k^2B_{\fK}^2$ by assumption \kbound/; by Lemma~\ref{le:boundloss}
and assumption {\bf (LB)}, the loss is bounded by $B_0+L_\ell R B_k B_{\fK}$, and is $L_\ell$-Lipschitz.
Therefore, with probability at least $1-\delta$ we get
\begin{align}
  (II) &\leq (B_0 + 3L_\ell R B_k B_{\fK})\min\paren{\sqrt{\frac{\log (\delta^{-1})}{2n}},1} \notag \\
  &\leq (B_0 + 3L_\ell R B_k B_{\fK})\min\paren{\paren{\frac{\log (\delta^{-1})}{2n}}^{\frac{\alpha}{2}},1}.  \label{eq:term2} 
\end{align}
Observe that we can cap the second factor at 1 since (II) is upper bounded by the
bound on the loss in all cases; the second inequality then uses $\alpha \leq 1$.
Combining with a union bound  the probabilistic controls~\eqref{eq:term1}, \eqref{eq:supf-dev} of term (I) and~\eqref{eq:term2} of (II) yields~\eqref{eq:bound2gsm_emp}.

To establish the bound~\eqref{eq:bound2gsm_ave} we use a similar argument.
We use the
decomposition
\begin{multline*}
\sup_{f \in \cB_{\ol{k}}(R)} \abs{\cE(f|P_{XY},n) - \cE^\infty(f|P_{XY})}\\
\begin{aligned}
& \leq \sup_{f \in \cB_{\ol{k}}(R)}
 \abs{\ee{S_n \sim (P_{XY})^{\otimes n}}{\frac1{n}\sum_{i=1}^{n}
\paren{ \ell(f(\wh{P}_X,X_{i}^T),Y_{i}) - \ell(f(P_X,X_{i}),Y_{i})}}}
\\
& \;\;
+ \sup_{f \in \cB_{\ol{k}}(R)}
 \abs{\ee{S \sim (P_{XY})^{\otimes n}}{\frac1{n}\sum_{i=1}^{n}
\paren{ \ell(f(P_X,X_{i}^T),Y_{i})}} - \ee{(X,Y)\sim P_{XY}}{\ell(f(P_X,X),Y)}}
\\
& =: (I') + (II').
\end{aligned}
\end{multline*}
It is easily seen that the second term vanishes: for any fixed
$f \in \cB_{\ol{k}}(R)$ and $P_{XY}$, the difference of the expectations is zero.
For the first term, using Lipschitzness of the loss, then~\eqref{eq:supf-exp}, we obtain
\begin{align*}
  (I')  \leq L_\ell \e{ \sup_{f \in \cB_{\ol{k}}(R)}  
  \norm{f(\wh{P}_X,.) - f(P_X,.)}_\infty}
  \leq 2 L_\ell R B_k L_\fK B_{k'}^\alpha n^{-\alpha/2},
\end{align*}
yielding~\eqref{eq:bound2gsm_ave}. The bound~\eqref{eq:bound2gsm_gen} is obtained as
a direct consequence by taking expectation over $P_{XY} \sim \mu$ and using Jensen's
inequality to pull out the absolute value.

\subsection{Regularity Conditions for the Kernel on Distributions}
\label{sec:reg}

We investigate sufficient conditions on the kernel $\fK$ to
ensure the regularity condition \khoeld/ \eqref{eq:kernhold}. Roughly speaking,
the regularity of the feature mapping of a reproducing kernel is
``one half'' of the regularity of the kernel in each of its variables.
The next result considers the situation where $\fK$ is itself simply
a H\"older continuous function of its variables.
\begin{lemma}
Let $\alpha\in(0,\frac{1}{2}]$.
Assume that the kernel $\fK$ is H\"older continuous of order~$2\alpha$
and constant $L^2_{\fK}/2$ in each of its two variables on
$\cB_{k'_X}(B_{k'})$. Then \khoeld/  is satisfied.
\end{lemma}
\begin{proof}
For any $v,w \in \cB_{k'_X}(B_{k'})$:
\begin{align*}
\norm{\Phi_{\fK}(v) - \Phi_{\fK}(w)} &
= \paren{ \fK(v,v) + \fK(w,w) - 2 \fK(v,w)}^{\frac{1}{2}}
\leq L_{\fK} \norm{v-w}^{\alpha}.
\end{align*}
\end{proof}
The above type of regularity only leads to a H\"older feature
mapping of order at most~$\frac{1}{2}$ (when the kernel function is Lipschitz
continuous in each variable). Since this order plays an important role
in the rate of convergence of the upper bound in the main error control theorem,
it is desirable to study conditions ensuring more regularity, in
particular a feature mapping which has
at least Lipschitz continuity.
For this, we consider the following stronger condition, namely that
the kernel function is twice differentiable in a specific sense:
\begin{lemma}
\label{le:Lipk}
Assume that, for any $u,v \in \cB_{k'_X}(B_{k'})$ and
unit norm vector $e$ of $\cH_{k'_X}$, the function
$h_{u,v,e}: (\lambda,\mu) \in \mbr^2 \mapsto \fK(u + \lambda e, v + \mu e)$ admits
a mixed partial derivative $\partial_{1} \partial_{2} h_{u,v,e}$ at the point $(\lambda,\mu)=(0,0)$
which is %continuous in $(u,v)$ and %%??
bounded in absolute value by a constant $C^2_\fK$ independent of $(u,v,e)$.
Then \eqref{eq:kernhold} is satisfied with $\alpha=1$ and $L_\fK=C_{\fK}$, that is,
the canonical feature mapping of $\fK$ is Lipschitz continuous on $\cB_{k'_X}(B_{k'})$.
\end{lemma}
\begin{proof}
The argument is along the same lines as \citet{steinwart08}, Lemma 4.34.
Observe that, since $h_{u,v,e}(\lambda+\lambda',\mu+\mu') = h_{u+\lambda e,v +\mu e,e}(\lambda',\mu')$,
the function $h_{u,v,e}$ actually admits a uniformly bounded mixed partial derivative in any point
$(\lambda,\mu)\in\mbr^2$ such that $(u+\lambda e,v+\mu e)\in \cB_{k'_X}(B_{k'})$\,.
Let us denote $\Delta_1 h_{u,v,e}(\lambda,\mu) := h_{u,v,e}(\lambda,\mu) - h_{u,v,e}(0,\mu)$\,.
For any $u , v \in \cB_{k'_X}(B_{k'})$\,, $u \neq v$\,, let us set $\lambda:=\norm{v-u}$ and
the unit vector $e:=\lambda^{-1}(v-u)$; we have
\begin{align*}
\norm{\Phi_{\fK}(u) - \Phi_{\fK}(v)}^2 & =
\fK(u,u) + \fK(u+\lambda e,u+\lambda e) -\fK(u,u+\lambda e) - \fK(u+\lambda e,u)\\
& = \Delta_1 h_{u,v,e}(\lambda,\lambda) - \Delta_1 h_{u,v,e}(\lambda,0)\\
& = \lambda \partial_2 \Delta_1 h_{u,v,e}(\lambda,\lambda')\,,
\end{align*}
where we have used the mean value theorem, yielding existence of
$\lambda' \in [0,\lambda]$ such that the last equality holds.
Furthermore,
\begin{align*}
\partial_2 \Delta_1 h_{u,v,e}(\lambda,\lambda') & =
\partial_2 h_{u,v,e}(\lambda,\lambda') - \partial_2 h_{u,v,e}(0,\lambda')\\
& = \lambda \partial_1 \partial_2 h_{u,v,e}(\lambda'',\lambda')\,,
\end{align*}
using again the mean value theorem, yielding existence of
$\lambda'' \in [0,\lambda]$ in the last equality.
Finally, we get
%\begin{align*}
\[
\norm{\Phi_{\fK}(u) - \Phi_{\fK}(v)}^2 = \lambda^2 \partial_{1} \partial_2  h_{u,v,e}(\lambda',\lambda'')
\leq C_{\fK}^2 \norm{v-u}^2\,.
\]
%\end{align*}
\end{proof}

\begin{lemma}
Assume that the kernel $\fK$ takes the form of either (a) $\fK(u,v) = g(\norm{u-v}^2)$ or
(b) $\fK(u,v) = g(\inner{u,v})$\,, where $g$ is a twice differentiable real function of a
real variable defined on $[0,4B_{k'}^2]$ in case (a), and on
$[-B_{k'}^2, B_{k'}^2]$ in case (b). Assume
$\norm{g'}_\infty \leq C_1$ and $\norm{g''}_\infty \leq C_2$. Then
$\fK$ satisfies the assumption of Lemma \ref{le:Lipk} with $C_\fK:=2C_1 + 16 C_2 B_{k'}^2$
in case (a), and $C_\fK := C_1 + C_2 B_{k'}^2$ for case (b).
\end{lemma}
\begin{proof}
In case (a), we have $h_{u,v,e}(\lambda,\mu)=g(\norm{u-v + (\lambda-\mu)e}^2)$. It follows
\begin{align*}
\abs{\partial_{1} \partial_2 h_{u,v,e}(0,0)} & = \abs{-2 g'(\norm{u-v}^2)\norm{e}^2
- 4g''(\norm{u-v}^2)
\inner{u-v,e}^2}\\
& \leq 2 C_1 + 16 C_2 B^2_{k'} \,.
\end{align*}
In case (b), we have $h_{u,v,e}(\lambda,\mu)=g(\inner{u + \lambda e, v + \mu e})$. It follows
\begin{align*}
\abs{\partial_{1} \partial_2 h_{u,v,e}(0,0)} & = \abs{g'(\inner{u,v})\norm{e}^2 +
g''(\inner{u,v}) \inner{u,e}\inner{v,e}
}\\
& \leq C_1 + C_2 B^2_{k'} \,.
\end{align*}
\end{proof}

\subsection{Proof of Lemma \ref{lem:produniv}}

\begin{proof}
Let $\cH,\cH'$ the RKHS associated to $k,k'$ with the associated feature
mappings
$\Phi,\Phi'$. Then it can be checked that $(x,x')\in \cX \times \cX'
\mapsto \Phi(x)\otimes\Phi'(x')$
is a feature mapping for $\ol{k}$ into the Hilbert space $\cH\otimes\cH'$.
Using \citep{steinwart08}, Th. 4.21, we deduce that the RKHS $\ol{H}$ of
$\ol{k}$ contains precisely all
functions of the form $(x,x') \in \cX\times\cX' \mapsto F_w(x,x') =
\inner{w,\Phi(x)\otimes\Phi(x')}$,
where $w$ ranges over $\cH\otimes\cH'$. Taking $w$ of the form $w =
g\otimes g'$,
$g\in \cH,g\in \cH'$, we deduce that $\ol{H}$ contains in particular all
functions of the
form $f(x,x')=g(x)g(x')$, and further
\[
\wt{\cH} := \mathrm{span}\set{(x,x') \in \cX\times\cX' \mapsto g(x)g(x') ;
g \in \cH, g' \in \cH'} \subset \ol{H}.
\]
Denote $\cC(\cX),\cC(\cX'),\cC(\cX\times \cX')$ the set of real-valued
continuous functions on the respective spaces. Let
\[
\cC(\cX)\otimes \cC(\cX') := \mathrm{span}\set{ (x,x') \in \cX \times \cX'
\mapsto  f(x)f'(x') ; f \in \cC(\cX), f' \in \cC(\cX')}.
\]
Let $G(x,x')$ be an arbitrary element of $\cC(\cX)\otimes \cC(\cX')$,
$G(x,x') = \sum_{i=1}^k \lambda_i g_i(x)g'_i(x')$
with $g_i \in \cC(\cX),g_i'\in \cC(\cX')$ for $i=1,\ldots,k$. For
$\eps>0$,
by universality of $k$ and $k'$, there exist
$f_i \in \cH,f_i'\in \cH'$ so that $\norm{f_i-g_i}_\infty \leq \eps$,
$\norm{f'_i-g'_i}_\infty \leq \eps$
for $i=1,\ldots,k$. Let $F(x,x') := \sum_{i=1}^k \lambda_i f_i(x)f'_i(x')
\in \wt{\cH}$. We have
\begin{align*}
\norm{F(x,x') - G(x,x')}_\infty & \leq \norm{\sum_{i=1}^k \lambda_i
(g_i(x)g'_i(x) - f_i(x) f'_i(x))}_\infty \\
& = \Bigg\|\sum_{i=1}^k \lambda_i \Big[(f_i(x)-g_i(x))(g'_i(x')-f'_i(x'))
\\
& \qquad \qquad + g_i(x)(g'_i(x) - f'_i(x')) + (g_i(x) -
f_i(x)) g'_i(x')\Big] \Bigg\|_\infty \\
& \leq \eps \sum_{i=1}^k \abs{\lambda_i}(\eps + \norm{g_i}_\infty +
\norm{g'_i}_\infty)\,.
\end{align*}
This establishes that $\wt{\cH}$ is dense in $\cC(\cX)\otimes \cC(\cX')$
for the supremum norm.
It can be easily checked that $\cC(\cX)\otimes \cC(\cX')$ is an algebra of
functions which does not vanish
and separates points on $\cX\times \cX'$. By the Stone-Weierstrass theorem,
it is therefore dense in
$\cC(\cX\times \cX')$ for the supremum norm. We deduce that $\wt{\cH}$
(and thus also $\ol{\cH}$)
is dense in $\cC(\cX\times \cX')$, so that $\ol{k}$ is universal.
\end{proof}

\subsection{Approximate Feature Mapping for Scalable Implementation}

We first treat random Fourier features and then the Nystr\"om method.

\subsubsection{Random Fourier Features}
The approximation of \citet{rahimi2007random} is based on Bochner's theorem, which
characterizes shift invariant kernels.
\begin{theorem}
    \label{thm:bochner}
    A continuous kernel $k(x,y) = k(x-y)$ on $\mathbb{R}^d$ is positive definite iff $k(x-y)$ is the Fourier transform of a
    finite positive measure $ p(w) $, i.e.,
\end{theorem}
 \begin{equation}
        \label{eq:bochner_eq}
       k(x-y) = \int_{\mathbb{R}^d} p(w)e^{jw^T(x-y)}dw\,.
 \end{equation}

%(TODO: check ``finite'')
% Done: Finite is correct, changed non-negative to positive.

If a shift invariant kernel $k(x-y)$ is properly scaled then
Theorem~\ref{thm:bochner} guarantees that $p(w)$ in \eqref{eq:bochner_eq}
is a proper probability distribution.

Random Fourier features (RFFs) approximate the integral in \eqref{eq:bochner_eq}
using samples drawn from $p(w)$. If $w_1,
w_2,..., w_L$ are i.i.d. draws from $p(w)$,
\begin{align}
    \label{eq:rff_approx}
    \nonumber k(x-y)  &= \int_{\mathbb{R}^d} p(w)e^{jw^T(x-y)}dw  \\  \nonumber
    &=   \int_{\mathbb{R}^d} p(w)\cos(w^Tx-w^Ty)dw   \\ \nonumber
   & \approx  \frac {1}{L} \sum_{i=1}^{L} \cos(w_i^Tx-w_i^Ty)\\ \nonumber
 	&= \frac {1}{L} \sum_{i=1}^{L} \cos(w_i^Tx) \cos(w_i^Ty) + \sin(w_i^Tx) \sin(w_i^Ty)\\ \nonumber
	&= \frac {1}{L} \sum_{i=1}^{L} [\cos(w_i^Tx) , \sin(w_i^Tx)]^T [\cos(w_i^Ty)  , \sin(w_i^Ty)] \\
	&= z_w(x)^Tz_w(y)  \,,
\end{align}
where $z_w(x) = \frac {1}{\sqrt L} [\cos(w_1^Tx) , \sin(w_1^Tx), ... ,
  \cos(w_L^Tx) , \sin(w_L^Tx)] \in \mathbb{R}^{2L}$ is an approximate nonlinear
feature mapping of dimensionality $2L$. In the
following, we extend the RFF methodology to the kernel
$\bar{k}$ on the extended feature space $\fP_{\cX}\times \cX $.  Let $
X_1,\ldots,X_{n_1}$ and $ X^{\prime}_1,\ldots,X^{\prime}_{n_2}$ be i.i.d.
realizations of $ P_X$ and $ P_X^{\prime}$ respectively, and let $
\widehat{P}_X$ and $ \widehat{P}^{\prime}_X$ denote the corresponding
empirical distributions. Given $ x, x^{\prime} \in {\cX} $, denote $
\tilde{x} = ( \wh{P}_X, x) $ and $ \tilde{x}^{\prime} = (\wh{P}_X^{\prime},
x^{\prime})$. The goal is to find an approximate feature mapping
$\bar{z}(\tilde{x})$
such that $ \bar{k}(\tilde{x},\tilde{x}^{\prime}) \approx
\bar{z}(\tilde{x})^T\bar{z}(\tilde{x}^{\prime}) $.  Recall that
\begin{equation*}
\label{eq:kernel_ext_space}
\bar{k}(\tilde{x},\tilde{x}^{\prime}) = k_P(\widehat{P}_X,\widehat{P}_X^{\prime})k_X(x,x^{\prime});
\end{equation*}
specifically, we consider $k_X$ and $k_X'$ to be Gaussian kernels and
the kernel on distributions $k_P$ to have the Gaussian-like form
\begin{equation*}
\label{eq:kernel_distribution}
  k_P(\widehat{P}_X,\widehat{P}_X^{\prime}) = \exp  \left\{
\frac{1}{2\sigma_P^2} \| \Psi(\widehat{P}_X) - \Psi(\widehat{P}_X^{\prime})\|^2_{H_{k_X^{\prime}} } \right\}.
\end{equation*}
As noted earlier in this section, the calculation of
$ k_P(\widehat{P}_X,\widehat{P}_X^{\prime}) $ reduces to the computation
of
\begin{equation}
    \label{eq:kernel_distribution_calc1}
    \langle  \Psi(\widehat{P}_X), \Psi(\widehat{P}_X^{\prime}) \rangle =
    \frac{1}{n_1n_2}  \sum_{i=1}^{n_1}  \sum_{j=1}^{n_2}  k_X^{\prime}(X_{i},X_{j}^{\prime}) .
\end{equation}

We use Theorem~\ref{thm:bochner} to approximate $ k_X^{\prime}$ and thus $
\langle  \Psi(\widehat{P}_X),
\Psi(\widehat{P}_X^{\prime}) \rangle$. Let $w_1, w_2,..., w_L$ be i.i.d. draws
from $ p^{\prime}(w) $, the inverse Fourier transform of
$k_X^{\prime}$.
Then we have:
\begin{multline*}
  \lefteqn{ \langle  \Psi(\widehat{P}_X), \Psi(\widehat{P}_X^{\prime}) \rangle =   \frac{1}{n_1n_2}  \sum_{i=1}^{n_1}  \sum_{j=1}^{n_2}  k_X^{\prime}(X_{i},X_{j}^{\prime}) } \\
\begin{aligned}
& \approx    \frac{1}{Ln_1n_2}\sum_{l=1}^L\sum_{i=1}^{n_1} \sum_{j=1}^{n_2} \cos(w_l^TX_{i} - w_l^TX_{j}^{\prime}) \\
& = \frac{1}{Ln_1n_2}\sum_{l=1}^L\sum_{i=1}^{n_1} \sum_{j=1}^{n_2} [\cos(w_l^TX_{i})\cos(w_l^TX_{j}^{\prime}) + \sin(w_l^TX_{i})\sin(w_l^TX_{j}^{\prime})] \\
& = \frac{1}{Ln_1n_2}\sum_{l=1}^L\{\sum_{i=1}^{n_1} [\cos(w_l^TX_{i}) , \sin(w_l^TX_{i})]^T \sum_{j=1}^{n_2}[\cos(w_l^TX_{j}^{\prime}) , \sin(w_l^TX_{j}^{\prime})]\}  \\
& = Z_P(\widehat{P}_X)^TZ_P(\widehat{P}_X^{\prime}), \\
\end{aligned}
\end{multline*}
where
\begin{eqnarray}
\label{eq:kernel_distribution_approx_first}
Z_P(\widehat{P}_X)=  \frac{1}{n_1\sqrt{L}} \sum_{i=1}^{n_1} \Big[ \cos
(w_1^TX_{i}), \sin (w_1^TX_{i}), ... ,\cos (w_L^TX_{i}), \sin
(w_L^TX_{i}) \Big],
\end{eqnarray}
and $Z_P(\widehat{P}_X^{\prime}) $ is defined analogously with $ n_1$
replaced by $ n_2$. For the proof of Theorem
\ref{thm:error_bound_KTL_approx}, let $z_X^{\prime}$ denote the approximate feature map
corresponding to $k_X^{\prime}$, which
satisfies $Z_P(\widehat{P}_X) = \frac1{n_1} \sum_{i=1}^{n_1}
z_X^{\prime}(X_i)$.

%Therefore,
%\begin{eqnarray}
%\label{eq:kernel_distribution_calc3}
%\langle \Psi(\widehat{P}_X), \Psi(\widehat{P}_X^{\prime}) \rangle &
%\approx & Z_P(\widehat{P}_X)^TZ_P(\widehat{P}_X^{\prime}).
%\end{eqnarray}

Note that the lengths of the vectors $Z_P(\widehat{P}_X)$ and
$Z_P(\widehat{P}_X^{\prime})$ are $2L$. To approximate $ \bar{k}$ we may
write
%Eqn \eqref{eq:product_kernel_approx1} represents the approximated
%product kernel. Here $X \in \mathbb{R}^d$ and $X^{\prime} \in
%\mathbb{R}^d$ can be any sample from $X_{1},...,X_{n_1}$ and
%$X_{1}^{\prime},...,X_{n_2}^{\prime}$ respectively. In this case, we have
\begin{align}
\label{eq:product_kernel_approx1}
\bar{k}( \tilde{x},\tilde{x}^{\prime}) &\approx  \exp{\frac{-\|Z_P(\widehat{P}_X)-Z_P(\widehat{P}_X^{\prime})\|_{\mathbb{R}^{2L}}^2}{2\sigma_P^2}} \cdot \exp{\frac{-\|x-x^{\prime}\|_{\mathbb{R}^d}^2}{2\sigma_X^2}} \\ \nonumber
& = \exp{\frac{-(\sigma_X^2\|Z_P(\widehat{P}_X)-Z_P(\widehat{P}_X^{\prime})\|_{\mathbb{R}^{2L}}^2+\sigma_P^2\|x-x^{\prime}\|_{\mathbb{R}^d}^2)}{2\sigma_P^2\sigma_X^2}}\\ \nonumber
& = \exp{\frac{-(\|\sigma_XZ_P(\widehat{P}_X)-\sigma_XZ_P(\widehat{P}_X^{\prime})\|_{\mathbb{R}^{2L}}^2+\|\sigma_Px-\sigma_Px^{\prime}\|_{\mathbb{R}^d}^2)}{2\sigma_P^2\sigma_X^2}} \\ \nonumber
&= \exp{\frac{-\|(\sigma_XZ_P(\widehat{P}_X),\sigma_Px)-(\sigma_XZ_P(\widehat{P}_X^{\prime}),\sigma_Px^{\prime})\|_{\mathbb{R}^{2L+d}}^2}{2\sigma_P^2\sigma_X^2}}. \nonumber
\end{align}

This is also a Gaussian kernel, now on $ \mathbb{R}^{2L + d}$.  Again by
applying Theorem~\ref{thm:bochner}, we have
\begin{eqnarray*}
\bar{k}( \widehat{P}_X ,X),(\widehat{P}_X^{\prime},X^{\prime})) \approx \int_{\mathbb{R}^{2L+d}} p(v)e^{jv^T((\sigma_XZ_P(P_X),\sigma_PX)-(\sigma_XZ_P(P_X^{\prime}),\sigma_PX^{\prime}))}dv.
\end{eqnarray*}
Let $v_1,v_2,...,v_q$ be drawn i.i.d. from $p(v)$, the inverse Fourier transform
of the Gaussian kernel with bandwidth $\sigma_P \sigma_X $. Let $u
=(\sigma_XZ_P(\widehat{P}_X),\sigma_Px) $ and $u^{\prime} =
(\sigma_XZ_P(\widehat{P}_X^{\prime}),\sigma_Px^{\prime}) $. Then
\begin{eqnarray*}
\label{eq:final_product_approx}
\bar{k}( \tilde{x},\tilde{x}^{\prime})  &\approx& \frac{1}{Q}\sum_{q=1}^Q \cos(v_q^T(u-u^{\prime})) \\
& = &  \bar{z}(\tilde{x})^T\bar{z}(\tilde{x}^{\prime}),
\end{eqnarray*}
where
\begin{eqnarray}
\label{eq:final_product_approx_feature}
\bar{z}(\tilde{x}) = \frac {1}{\sqrt Q} [\cos(v_1^Tu) , \sin(v_1^Tu), ...
, \cos(v_Q^T u) , \sin(v_Q^Tu)] \in \mathbb{R}^{2Q}
\end{eqnarray}
and $\bar{z}(\tilde{x}^\prime)$ is defined similarly.

This completes the construction of the approximate feature map. The following result, which uses Hoeffding's
inequality and generalizes a result of \citet{rahimi2007random}, says that the
approximation achieves any desired approximation error with very high probability as $L, Q \to \infty$. 

%\begin{algorithm}
%\caption{Random Fourier Features - RFF}
%\label{alg:RFF_algo}
%\begin{algorithmic}[1]
%\State Input: A positive definite shift invariant kernel $k(x,y) = k(x-y)$
%\State Compute: Fourier transform $p$ of the kernel $k$: $p(w) = \frac{1}{2\pi} \int_{R^d} e^{-jw^T\delta}k(\delta)d\Delta$
%\State Draw L iid $w_1, w_2,..., w_L \in R^d$ samples from $p$
%\State Return: $z(x) =  \frac {1}{\sqrt L} [cos(w_i^Tx) , sin(w_i^Tx)]_{i=1}^L \in R^{2L}$
%\end{algorithmic}
%\end{algorithm}
%\begin{algorithm}
%    \SetKwInOut{Input}{input}
%    \SetKwInOut{Output}{output}
%    \SetKwInOut{Compute}{compute}
%    \Input{Bandwidths - $ \sigma_X, \sigma_X^{\prime}, \sigma_P$, test points - $ \tilde{x} = (\widehat{P}_X, x)$.}
%       \Compute{Fourier transform $p(w)$ of the kernel $k_X$. \newline Draw $L$ iid $w_1, w_2,..., w_L \in R^d$ samples from $p^{\prime}(w)$. }
%       \Compute{$Z_P(P_X)$ as in \eqref{eq:kernel_distribution_approx_first}.}
%       \Compute{Fourier transform $p(v)$ of the approximated Gaussian kernel of $\bar{k}$. \newline Draw $ Q$ iid $v_1, v_2,..., v_Q$ samples from $p(v)$.}
%       \Compute{$\bar{z}(\tilde{x})$ as in \eqref{eq:final_product_approx_feature}.}
%	\Output{$\bar{z}(\tilde{x})$ }
%	\label{alg:RFF_Approx}
%    \caption{Kernel approximation using random Fourier features}
%\end{algorithm}

\begin{theorem}
    \label{thm:error_bound_KTL_approx}
Let $L$ be the number of random features to approximate the kernel on distributions and Q
be the number of features to approximate the final product kernel. For any $\epsilon_l > 0$,
$\epsilon_q > 0$, $\tilde{x} =
(\widehat{P}_X,x)$, $ \tilde{x}^{\prime} = (\widehat{P}_X^{\prime},x^{\prime})$,
\begin{equation}
P( |\bar{k}(\tilde{x}, \tilde{x}^{\prime}) - \bar{z}(\tilde{x})^T\bar{z}(\tilde{x}^{\prime})| \geq \epsilon_l + \epsilon_q ) \leq 2\exp\Big(-\frac{Q\epsilon_q^2}{2}\Big) + 6n_1n_2\exp\Big(-\frac{L\epsilon^2}{2}\Big),
\end{equation}
where $ \epsilon = \frac{\sigma_P^2}{2}\log (1 + \epsilon_l)$, $\sigma_P$ is the
bandwidth parameter of the Gaussian-like kernel $k_P$, and $n_1$ and $n_2$ are the
sizes of the empirical distributions $\widehat{P}_X$ and $\widehat{P}_X^{\prime} $,
respectively.
\end{theorem}

\begin{proof}
Observe:
\begin{equation*}
\bar{k} (\tilde{x},\tilde{x}^{\prime})= \exp\left\{\frac{-1}{2\sigma_P^2}\| \Psi(\widehat{P}_X) - \Psi(\widehat{P}_X^{\prime})\|^2\right\}\exp\left\{\frac{-1}{2\sigma_X^2}\|x - x^{\prime}\|^2\right\},
\end{equation*}
and denote:
\begin{equation*}
\tilde{k}(\tilde{x},\tilde{x}^{\prime}) = \exp\left\{\frac{-1}{2\sigma_P^2}\|Z_P(\widehat{P}_X) - Z_P(\widehat{P}_X^{\prime})\|^2\right\}\exp\left\{\frac{-1}{2\sigma_X^2}\|x - x^{\prime}\|^2\right\},
\end{equation*}
We omit the arguments of $\bar{k}, \tilde{k} $ for brevity. Let $k_q$ be the final approximation ($ k_q =
\bar{z}(\tilde{x})^T\bar{z}(\tilde{x}^{\prime}) $) and then we have

\begin{equation}
\label{eq:split_error}
|\bar{k} - k_q | = |\bar{k} - \tilde{k} + \tilde{k} - k_q | \leq  |\bar{k} - \tilde{k}| + |\tilde{k} - k_q |.
\end{equation}

%Our goal is to find the bound on $P( |\tilde{k} - k_q | \geq \epsilon_l + \epsilon_q ) $.

%When random Fourier features are used for approximating Gaussian kernel ( $k(x,x^{\prime})$), we have following error bound
%using Hoeffding's inequality \citep{rahimi2007random}.$ P( |z_w(x)^Tz_w(x^{\prime}) - k(x,x^{\prime}) | \geq \epsilon ) \leq
%2\exp(-\frac{L\epsilon^2}{2})  $, for the approximation of a Gaussian kernel using RFF (as in Eqn. \eqref{eq:rff_approx}),
%where $L$ is number of Fourier features used.

From Eqn.~\eqref{eq:split_error} it follows that,
\begin{equation}
\label{eq:split_error_prob}
P(|\bar{k} - k_q | \geq \epsilon_l + \epsilon_q ) \leq P(|\bar{k} - \tilde{k}| \geq \epsilon_l  ) + P(|\tilde{k} - k_q |
\geq \epsilon_q).
\end{equation}
%Using the random Fourier features error bound, we can bound second probability:
By a direct application of Hoeffding's inequality,
\begin{equation}
\label{eq:second_bound}
P( |\tilde{k} - k_q | \geq \epsilon_q ) \leq 2\exp(-\frac{Q\epsilon_q^2}{2}).
\end{equation}
Recall that $\langle \Psi(\widehat{P}_X), \Psi(\widehat{P}_X^{\prime}) \rangle =  \frac{1}{n_1n_2}  \sum_{i=1}^{n_1}
\sum_{j=1}^{n_2}  k_X^{\prime}(X_{i},X_{j}^{\prime})  $.
For a pair $X_i,X_j^{\prime}$, we have again by Hoeffding
\begin{equation*}
\label{eq:first_bound_single_pair}
P(|z_X^{\prime}(X_{i})^T z_X^{\prime}(X_{j}^{\prime}) - k_X^{\prime}(X_{i},X_{j}^{\prime}) | \geq  \epsilon ) \leq
2\exp(-\frac{L\epsilon^2}{2}).
\end{equation*}
Let $\Omega_{ij} $ be the event $ |z_X^{\prime}(X_{i})^T z_X^{\prime}(X_{j}^{\prime}) - k_X^{\prime}(X_{i},X_{j}^{\prime}) |
\geq  \epsilon $, for particular $i,j $. Using the union bound we have
\begin{equation*}
\label{eq:first_bound_all_pairs1}
P(\Omega_{11} \cup \Omega_{12} \cup \ldots \cup \Omega_{n_1n_2} ) \leq 2n_1n_2\exp(-\frac{L\epsilon^2}{2})
\end{equation*}
This implies %$\Rightarrow $
\begin{equation}
\label{eq:first_bound_all_pairs2}
P( | Z_P(\widehat{P}_X)^TZ_P(\widehat{P}_X^{\prime}) - \langle \Psi(\widehat{P}_X), \Psi(\widehat{P}_X^{\prime}) \rangle|
\geq \epsilon ) \leq 2n_1n_2\exp(-\frac{L\epsilon^2}{2}).
\end{equation}
Therefore,
\begin{flalign*}
    \label{eq:first_bound_bound}
  & \left| \bar{k} - \tilde{k}  \right|  = \Bigg| \exp\left\{ \frac{-1}{2\sigma_X^2}\|x - x^{\prime}\|^2 \right\} \Bigg[ \exp\left\{\frac{-1}{2\sigma_P^2}\| \Psi(\widehat{P}_X) - \Psi(\widehat{P}_X^{\prime})\|^2\right\} &\\
 & \indent     - \exp\left\{\frac{-1}{2\sigma_P^2}\|Z_P(\widehat{P}_X) - Z_P(\widehat{P}_X^{\prime})\|^2\right\} \Bigg] \Bigg| &\\
    &\leq \left| \Bigg[ \exp\left\{\frac{-1}{2\sigma_P^2}\| \Psi(\widehat{P}_X) - \Psi(\widehat{P}_X^{\prime})\|^2\right\} - \exp\left\{\frac{-1}{2\sigma_P^2}\|Z_P(\widehat{P}_X) - Z_P(\widehat{P}_X^{\prime})\|^2\right\} \Bigg] \right| &
\end{flalign*}
\begin{flalign*}
    & = \Bigg| \exp\left\{\frac{-1}{2\sigma_P^2}\| \Psi(\widehat{P}_X) - \Psi(\widehat{P}_X^{\prime})\|^2\right\}  \Bigg[ 1 - \exp\Big\{\frac{-1}{2\sigma_P^2}\Big( \|Z_P(\widehat{P}_X)- Z_P(\widehat{P}_X^{\prime})\|^2 &\\
    &\indent  \   - \| \Psi(\widehat{P}_X) - \Psi(\widehat{P}_X^{\prime})\|^2 \Big) \Big\} \Bigg] \Bigg| &\\
    &\leq \left| \Bigg[ 1 - \exp\left\{\frac{-1}{2\sigma_P^2}\Big( \| Z_P(\widehat{P}_X) - Z_P(\widehat{P}_X^{\prime})\|^2 - \| \Psi(\widehat{P}_X) - \Psi(\widehat{P}_X^{\prime})\|^2 \Big) \right\} \Bigg] \right| &\\
    &= \Bigg| 1 - \exp\Bigg\{\frac{-1}{2\sigma_P^2}\Big( Z_P(\widehat{P}_X)^TZ_P(\widehat{P}_X)- \langle \Psi(\widehat{P}_X), \Psi(\widehat{P}_X) \rangle  + Z_P(\widehat{P}_X^{\prime})^TZ_P(\widehat{P}_X^{\prime}) &\\
    &\indent    - \langle \Psi(\widehat{P}_X^{\prime}), \Psi(\widehat{P}_X^{\prime}) \rangle - 2\big( Z_P(\widehat{P}_X)^TZ_P(\widehat{P}_X^{\prime}) - \langle \Psi(\widehat{P}_X), \Psi(\widehat{P}_X^{\prime}) \rangle \big) \Big) \Bigg\} \Bigg| &\\
      &\leq \Bigg| 1 - \exp\Bigg\{\frac{1}{2\sigma_P^2}\Big( |Z_P(\widehat{P}_X)^TZ_P(\widehat{P}_X) - \langle \Psi(\widehat{P}_X), \Psi(\widehat{P}_X) \rangle|  + |Z_P(\widehat{P}_X^{\prime})^TZ_P(\widehat{P}_X^{\prime}) &  \\
   & \indent  - \langle \Psi(\widehat{P}_X^{\prime}), \Psi(\widehat{P}_X^{\prime}) \rangle|  + 2|\big( Z_P(\widehat{P}_X)^TZ_P(\widehat{P}_X^{\prime}) - \langle \Psi(\widehat{P}_X),
\Psi(\widehat{P}_X^{\prime}) \rangle \big)| \Big) \Bigg\} \Bigg| &
\end{flalign*}
The result now follows by applying the bound of Eqn. \eqref{eq:first_bound_all_pairs2} to each of the three terms in the
exponent of the
preceding expression, together with the stated formula for $\epsilon$ in terms of $\epsilon_\ell$.

%\begin{eqnarray}
%P \Big(  \sup_{x,x^{\prime} \in \mathcal{M}} |\bar{k}(\tilde{x},\tilde{x}^{\prime}) - \bar{z}(\tilde{x})^T\bar{z}(\tilde{x}^{\prime})| \geq \epsilon_l + \epsilon_q \Big) \\ \nonumber
%\leq  2^8 {\big(\frac{\sigma r}{\epsilon_q}\big)}^2 \exp(\frac{-Q\epsilon_q^2}{2(d+2)} ) +  2^86n_1n_2 {\big(\frac{\sigma r}{\epsilon_l}\big)}^2 \exp(\frac{-L\epsilon_l^2}{2(d+2)} )
%\end{eqnarray}

\end{proof}

The above results holds for fixed $\tilde{x}$ and $\tilde{x}^\prime$. Following again \citet{rahimi2007random},
one can use an $\epsilon$-net argument to prove a stronger statement for every pair of
points in the input space simultaneously. They show

\begin{lemma}
\label{lem:rff_strong}
Let $\mathcal{M} $ be a compact subset of $\mathbb{R}^d $ with diameter $r = \mathrm{diam}(\mathcal{M} )$ and let $D $ be the number of random Fourier features used. Then for
the mapping defined in \eqref{eq:rff_approx}, we have
\[ P\Big( \sup_{x,y \in \mathcal{M}}| z_w(x)^Tz_w(y)  - k(x-y) | \geq \epsilon \Big) \leq
2^8 {\Big(\frac{\sigma r}{\epsilon}\Big)}^2 \exp\Big(\frac{-D\epsilon^2}{2(d+2)} \Big),\]
where $\sigma = \mathbb{E}[w^Tw]$ is the second moment of the Fourier transform of $k$.
\end{lemma}

Our RFF approximation of $\bar{k}$ is grounded on Gaussian RFF approximations on Euclidean spaces, and thus,
the following result holds by invoking Lemma \ref{lem:rff_strong}, and otherwise following the argument of
Theorem \ref{thm:error_bound_KTL_approx}.

\begin{theorem}
    \label{thm:error_strong_bound_KTL_approx}
    Using the same notations as in Theorem~\ref{thm:error_bound_KTL_approx} and Lemma~\ref{lem:rff_strong},
\begin{multline}
P \Big(  \sup_{x,x^{\prime} \in \mathcal{M}} |\bar{k}(\tilde{x},\tilde{x}^{\prime}) - \bar{z}(\tilde{x})^T\bar{z}(\tilde{x}^{\prime})| \geq \epsilon_l + \epsilon_q \Big) \\
 \leq  2^8 {\Big(\frac{\sigma_X^{\prime} r}{\epsilon_q}\Big)}^2 \exp\Big(\frac{-Q\epsilon_q^2}{2(d+2)} \Big) +  2^93n_1n_2 {\Big(\frac{\sigma_P \sigma_X r}{\epsilon_l}\Big)}^2 \exp\Big(\frac{-L\epsilon_l^2}{2(d+2)} \Big)
\end{multline}
where $\sigma_{X}^{\prime}$ is the width of kernel $ k_{X}^{\prime}$ in Eqn. (\ref{eq:kernel_distribution_calc1})
 and $ \sigma_P$ and $\sigma_X $  are the widths of kernels $k_P$ and $k_X $ respectively.
\end{theorem}

\begin{proof}
	The proof is very similar to the proof of Theorem \ref{thm:error_bound_KTL_approx}. We use
Lemma \ref{lem:rff_strong} to replace
bound (\ref{eq:second_bound}) with:
	\begin{equation}
	\label{eq:second_bound_strong}
	P\bigg(  \sup_{x,x^{\prime} \in \mathcal{M}} |\tilde{k} - k_q | \geq \epsilon_q \bigg) \leq 2^8 {\bigg(\frac{\sigma_{X}^{\prime} r}{\epsilon_q}\bigg)}^2 \exp\bigg(\frac{-Q\epsilon_q^2}{2(d+2)} \bigg).
	\end{equation}
	Similarly, Eqn. (\ref{eq:first_bound_all_pairs2}) is replaced by	
	\begin{multline}
	\label{eq:first_bound_all_pairs2_strong}
	P\bigg(  \sup_{x,x^{\prime} \in \mathcal{M}} \big| Z_P(\widehat{P}_X)^TZ_P(\widehat{P}_X^{\prime}) - \big\langle \Psi(\widehat{P}_X), \Psi(\widehat{P}_X^{\prime}) \big\rangle\big|
	\geq \epsilon \bigg)  \\
	\leq  2^9n_1n_2 {\bigg(\frac{\sigma_P\sigma_X r}{\epsilon_l}\bigg)}^2 \exp\bigg(\frac{-L\epsilon_l^2}{2(d+2)}\bigg).
	\end{multline}
	
	The remainder of the proof now proceeds as in the previous proof.
%	Everything else follows similarly as in Eqn. \ref{eq:first_bound_bound}.

\end{proof}

There are recent developments that give faster rates for approximation quality of random Fourier features and could potentially be combined with our analysis \citep{sriperumbudur2015optimal,sutherland2015error}. For example, approximation quality for the kernel mean map is discussed in \citet{sutherland2015error}, and these ideas could be extended to Theorem \ref{thm:error_strong_bound_KTL_approx} by combining with the two-stage approach presented in this paper. We also note that our analysis of random Fourier features is separate from our analysis of the kernel learning algorithm. We have not presented a generalization error bound for the learning algorithm using random Fourier features \citep{rudi2017generalization}.
%Note that we have a separate generalization analysis (Theorem \ref{th:main}) and then separate analysis for the bound on kernel approximation (Theorem \ref{thm:error_strong_bound_KTL_approx}).

\subsubsection{Nystr{\"o}m Approximation}

Like random Fourier features, the Nystr{\"o}m approximation is a technique to approximate kernel matrices \citep{williams2001using, drineas2005nystrom}. Unlike random Fourier features, %which uses data independent features,
for the Nystr{\"o}m approximation, the feature maps are data-dependent. Also, in the last subsection,
all
kernels were assumed to be shift invariant. With the Nystr{\"o}m approximation there is no
such assumption.

For a general kernel $k$, the goal is to find a feature mapping $z : \mathbb{R}^d
\rightarrow \mathbb{R}^L$, where
$L >
d$, such that $ k(x,x^{\prime}) \approx z(x)^Tz(x^{\prime})$. Let $r$ be
the target rank of the final approximated kernel matrix, and $m$ be the number of selected
columns
of the original kernel matrix. In general $r \leq m \ll n$.
% The Nystr{\"o}m method
%approximates the kernel matrix $[K(x_i,x_j)]_{i,j}$ using a small subset of its columns
%($m$).

Given data points $x_1, \ldots, x_n$, the Nystr{\"o}m method approximates
the kernel matrix by first sampling $m$ data points ${x}_1^{\prime}
,{x}_2^{\prime},...,{x}_m^{\prime}$ without replacement from the original sample, and then
constructing a low rank matrix by
$\widehat{K}_r = K_b\widehat{K}^{-1}K_b^T$, where $K_b = [k(x_i,{x}_j^{\prime})]_{n \times
m},$ and $\widehat{K} =[k({x}_i^{\prime},{x}_j^{\prime})]_{m \times m} $. Hence, the final
approximate feature mapping is
\begin{equation}
\label{eq:nystro_final}
z_n(x) =
\widehat{D}^{-\frac{1}{2}}\widehat{V}^T[k(x,{x}_1^{\prime}),...,k(x,{x}_m^{\prime})],
\end{equation}
where $\widehat{D}$ is the eigenvalue matrix of $ \widehat{K}$ and $ \widehat{V}$ is the
corresponding eigenvector matrix.

\subsection{Results in Tabular Format}

\begin{table}[H]
% title of Table
\centering % used for centering table
\begin{tabular}{c| c| c| c| c|} % centered columns (4 columns)
	\multicolumn{1}{c}{} & \multicolumn{4}{c}{Tasks} \\
\cline{2-5}
 \multirow{4}*{\rotatebox{90}{Examples per Task}} &  & 16 & 64 & 256  \\ [1ex] % inserts table
\cline{2-5}% inserts single horizontal line
& 8 & 36.01 & 33.08 & 31.69 \\[1ex] % [1ex] adds vertical space
\cline{2-5}
& 16 & 31.55 & 31.03 & 30.96 \\[1ex]
\cline{2-5}
& 32 & 30.44 & 29.31  & 23.87\\[1ex]
\cline{2-5}
& 256 & 23.78 & 7.22 & 1.27\\[1ex]
\cline{2-5} %inserts single line
\end{tabular}
  \caption{Average Classification Error of Marginal Transfer Learning on Synthetic Data set
  \label{table:synthetic_marginal} % is used to refer this table in the text
} \end{table}

\begin{table}[H]
% title of Table
\centering % used for centering table
\begin{tabular}{c| c| c| c| c|} % centered columns (4 columns)
\multicolumn{1}{c}{} & \multicolumn{4}{c}{Tasks} \\
\cline{2-5}
\multirow{4}*{\rotatebox{90}{Examples per Task}} & & 16 & 64 & 256  \\ [1ex] % inserts table
%heading
\cline{2-5}
& 8 & 49.14 & 49.11 & 50.04 \\[1ex] % [1ex] adds vertical space
\cline{2-5}
& 16 & 49.89 & 50.04 & 49.68 \\[1ex]
\cline{2-5}
& 32 & 50.32 & 50.21 & 49.61\\[1ex]
\cline{2-5}
& 256 & 50.01 & 50.43 & 49.93\\[1ex]
\cline{2-5}
\end{tabular}
  \caption{Average Classification Error of Pooling on Synthetic Data set
  \label{table:synthetic_pooling} % is used to refer this table in the text
} 
\end{table}

\begin{table}[H]
\centering % used for centering table
\begin{tabular}{c| c| c| c| c| c| c|c|} % centered columns (4 columns)
\multicolumn{1}{c}{} & \multicolumn{7}{c}{Tasks} \\
\cline{2-8}
\multirow{16}*{\rotatebox{90}{Examples per Task}} &  & 10 &  15  & 20 &  25 & 30 & 35 \\ [1ex] \cline{2-8}
& 20 & 13.78 & 12.37 & 11.93 &  10.74 & 10.08 & 11.17 \\[1ex] % [1ex] adds vertical space
\cline{2-8}
& 24 & 14.18 & 11.89 & 11.51 & 10.90 & 10.55 & 10.18\\[1ex] % [1ex] adds vertical
 \cline{2-8}
& 28 & 14.95 & 13.29 & 12.00 & 10.21 & 10.59& 9.52\\[1ex] % [1ex] adds vertical
\cline{2-8}
& 34 & 13.27 & 11.66 & 11.79 & 9.16 & 9.34 & 10.50 \\[1ex] % [1ex] adds vertical
\cline{2-8}
& 41 & 12.89& 11.27 & 11.17 & 9.91 & 9.10 & 10.05\\[1ex] % [1ex] adds vertical
\cline{2-8}
& 49 & 13.15 & 11.70 & 13.81 &  10.12 & 9.01 & 8.69\\[1ex] % [1ex] adds vertical
\cline{2-8}
& 58 & 12.16 & 9.59 & 9.85 & 9.28 & 8.44 & 7.62\\[1ex] % [1ex] adds vertical
\cline{2-8}
& 70 & 13.03 & 9.16 & 8.80 & 9.03 & 8.16 & 7.88\\[1ex] % [1ex] adds vertical
\cline{2-8}
& 84 & 11.98 & 9.18 & 9.74 &  9.03 & 7.30 & 7.01\\[1ex] % [1ex] adds vertical
\cline{2-8}
& 100 & 12.69 & 8.48 & 9.52 & 8.01 &  7.14 & 7.5 \\[1ex] % [1ex] adds vertical
\cline{2-8}
\end{tabular}
  \caption{RMSE of Marginal Transfer Learning on Parkinson's Disease Data set
  \label{table:parkinson_marginal} % is used to refer this table in the text
} % title of Table
\end{table}

\begin{table}[H]
\centering % used for centering table
\begin{tabular}{c| c| c| c| c| c| c|c|} % centered columns (4 columns)
	\multicolumn{1}{c}{} & \multicolumn{7}{c}{Tasks}\\
	\cline{2-8}
\multirow{16}*{\rotatebox{90}{Examples per Task}} &   & 10 &  15  & 20 &  25 & 30 & 35 \\ [1ex]
\cline{2-8}
 &  20 & 13.64 & 11.93 & 11.95 &  11.06 & 11.91 & 12.08 \\[1ex] % [1ex] adds vertical space
\cline{2-8}
& 24 & 13.80 & 11.83 & 11.70 & 11.98 & 11.68 & 11.48\\[1ex] % [1ex] adds vertical
\cline{2-8}
& 28 & 13.78 & 11.70 & 11.72 & 11.18 & 11.58 & 11.73 \\[1ex] % [1ex] adds vertical
\cline{2-8}
& 34 & 13.71 & 12.20 & 12.04 & 11.17 & 11.67 & 11.92 \\[1ex] % [1ex] adds vertical
\cline{2-8}
& 41 & 13.69 & 11.73 & 12.08 & 11.28 & 11.55 & 12.59 \\[1ex] % [1ex] adds vertical
\cline{2-8}
& 49 & 13.75 & 11.85 & 11.79 &  11.17  & 11.34 & 11.82\\[1ex] % [1ex] adds vertical
\cline{2-8}
& 58 & 13.70 & 11.89 & 12.06 & 11.06 & 11.82 & 11.65\\[1ex] % [1ex] adds vertical
\cline{2-8}
& 70 & 13.54 & 11.86 & 12.14 & 11.21 & 11.40 & 11.96\\[1ex] % [1ex] adds vertical
\cline{2-8}
& 84 & 13.55 & 11.98 & 12.03 & 11.25 & 11.54 & 12.22\\[1ex] % [1ex] adds vertical
\cline{2-8}
& 100 & 13.53 & 11.85 & 11.92 & 11.12 & 11.96 & 11.84 \\[1ex] % [1ex] adds vertical
\cline{2-8}
\end{tabular}
  \caption{RMSE of Pooling on Parkinson's Disease Data set
  \label{table:parkinson_pooling} % is used to refer this table in the text
} % title of Table
\end{table}

\begin{table}[H]
\centering % used for centering table
\begin{tabular}{c| c| c| c| c| c|} % centered columns (4 columns)
\multicolumn{1}{c}{} & \multicolumn{5}{c}{Tasks} \\
\cline{2-6}
\multirow{8}*{\rotatebox{90}{Examples per Task}} & & 10  & 20 & 30 & 40 \\ [1ex]
\cline{2-6}
& 5 & 8.62 & 7.61 & 8.25 & 7.17\\[1ex] % [1ex] adds vertical space
\cline{2-6}
& 15 & 6.21 & 5.90 & 5.85 & 5.43\\[1ex]
\cline{2-6}
& 30 & 6.61 & 5.33 & 5.37 & 5.35\\[1ex]
\cline{2-6}
& 45 & 5.61 & 5.19 & 4.71 & 4.70\\[1ex]
\cline{2-6}
& all training data & 5.36 & 4.91 & 3.86 & 4.08\\[1ex]
\cline{2-6}
\end{tabular}
  \caption{Average Classification Error of Marginal Transfer Learning on Satellite Data set
  \label{table:satellite_marginal} % is used to refer this table in the text
} % title of Table
\end{table}

\begin{table}[H]
\centering % used for centering table
\begin{tabular}{c| c| c| c| c| c|} % centered columns (4 columns)
\multicolumn{1}{c}{} & \multicolumn{5}{c}{Tasks} \\
\cline{2-6}
\multirow{8}*{\rotatebox{90}{Examples per Task}} & & 10  & 20 & 30 & 40 \\ [1ex] % inserts table
%heading
\cline{2-6}
&5 & 8.13 & 7.54 & 7.94 & 6.96\\[1ex] % [1ex] adds vertical space
\cline{2-6}
&15 & 6.55 & 5.81 & 5.79 & 5.57\\[1ex]
\cline{2-6}
&30 & 6.06 & 5.36 & 5.56 & 5.31\\[1ex]
\cline{2-6}
&45 & 5.58 & 5.12 & 5.30 & 4.99\\[1ex]
\cline{2-6}
&all training data & 5.37 & 4.98 & 5.32 & 5.14\\[1ex]
\cline{2-6}
\end{tabular}
  \caption{Average Classification Error of Pooling on Satellite Data set
  \label{table:satellite_pooling} % is used to refer this table in the text
} % title of Table
\end{table}

\begin{table}[H]
\centering % used for centering table
\begin{tabular}{c| c| c| c| c| c| } % centered columns (4 columns)
\multicolumn{1}{c}{} & \multicolumn{5}{c}{Tasks} \\
\cline{2-6}
\multirow{8}*{\rotatebox{90}{Examples per Task}} & & 5 & 10 & 15 & 20 \\ [1ex] % inserts table
%heading
\cline{2-6}
& 1024 & 9 & 9.03 & 9.03 & 8.70\\[1ex] % [1ex] adds vertical space
\cline{2-6}
&2048 & 9.12 & 9.56 & 9.07 & 8.62\\[1ex]
\cline{2-6}
&4096 & 8.96 & 8.91 & 9.01 & 8.66\\[1ex]
\cline{2-6}
&8192 & 9.18 & 9.20 & 9.04 & 8.74\\[1ex]
\cline{2-6}
&16384 & 9.05 & 9.08 & 9.04 & 8.63\\[1ex]
\cline{2-6}
\end{tabular}
  \caption{Average Classification Error of Marginal Transfer Learning on Flow Cytometry Data set
  \label{table:flow_marginal} % is used to refer this table in the text
} % title of Table
\end{table}

\begin{table}[H]
\centering % used for centering table
\begin{tabular}{c| c| c| c| c| c| } % centered columns (4 columns)
	\multicolumn{1}{c}{} & \multicolumn{5}{c}{Tasks} \\
	\cline{2-6}
	\multirow{8}*{\rotatebox{90}{Examples per Task}} & & 5 & 10 & 15 & 20 \\ [1ex] % inserts table
%heading
\cline{2-6}
&1024 & 9.41 & 9.48 & 9.32 & 9.52\\[1ex] % [1ex] adds vertical space
\cline{2-6}
&2048 & 9.92 & 9.57 & 9.45 & 9.54\\[1ex]
\cline{2-6}
&4096 & 9.72 & 9.56 & 9.36 & 9.40\\[1ex]
\cline{2-6}
&8192 & 9.43 & 9.53 & 9.38 & 9.50\\[1ex]
\cline{2-6}
&16384 & 9.42 & 9.56 & 9.40 & 9.33\\[1ex]
\cline{2-6}
\end{tabular}
  \caption{Average Classification Error of Pooling on Flow Cytometry Data set
  \label{table:flow_pooling} % is used to refer this table in the text
} % title of Table
\end{table}

% Note: in this sample, the section number is hard-coded in. Following
% proper LaTeX conventions, it should properly be coded as a reference:

%In this appendix we prove the following theorem from
%Section~\ref{sec:textree-generalization}:

%\bibliographystyle{../IEEEbib}
\vskip 0.2in
\bibliography{refs}

\begin{thebibliography}{102}
\providecommand{\natexlab}[1]{#1}
\providecommand{\url}[1]{\texttt{#1}}
\expandafter\ifx\csname urlstyle\endcsname\relax
  \providecommand{\doi}[1]{doi: #1}\else
  \providecommand{\doi}{doi: \begingroup \urlstyle{rm}\Url}\fi

\bibitem[Aghaeepour et~al.(2013)Aghaeepour, Finak, Consortium, Consortium,
  Hoos, Mosmann, Brinkman, Gottardo, and Scheuermann]{aghaeepour2013critical}
Nima Aghaeepour, Greg Finak, The~FlowCAP Consortium, The~DREAM Consortium,
  Holger Hoos, Tim~R. Mosmann, Ryan Brinkman, Raphael Gottardo, and Richard~H.
  Scheuermann.
\newblock Critical assessment of automated flow cytometry data analysis
  techniques.
\newblock \emph{Nature Methods}, 10\penalty0 (3):\penalty0 228--238, 2013.

\bibitem[Akuzawa et~al.(2019)Akuzawa, Iwasawa, and
  Matsuo]{Akuzawa2019AdversarialIF}
Kei Akuzawa, Yusuke Iwasawa, and Yutaka Matsuo.
\newblock Adversarial invariant feature learning with accuracy constraint for
  domain generalization.
\newblock In \emph{European Conference on Machine Learning and Principles and
  Practice of Knowledge Discovery in Databases}, 2019.

\bibitem[Azizzadenesheli et~al.(2019)Azizzadenesheli, Liu, Yang, and
  Anandkumar]{azizzadenesheli2018regularized}
Kamyar Azizzadenesheli, Anqi Liu, Fanny Yang, and Animashree Anandkumar.
\newblock Regularized learning for domain adaptation under label shifts.
\newblock In \emph{International Conference on Learning Representations}, 2019.
\newblock URL \url{https://openreview.net/forum?id=rJl0r3R9KX}.

\bibitem[Bak{\i}r et~al.(2007)Bak{\i}r, Hofmann, Sch{\"o}lkopf, Smola, and
  Taskar]{bakir2007predicting}
G{\"o}khan Bak{\i}r, Thomas Hofmann, Bernhard Sch{\"o}lkopf, Alexander~J Smola,
  and Ben Taskar.
\newblock \emph{Predicting Structured Data}.
\newblock MIT Press, 2007.

\bibitem[Balaji et~al.(2018)Balaji, Sankaranarayanan, and
  Chellappa]{balaji18neurips}
Yogesh Balaji, Swami Sankaranarayanan, and Rama Chellappa.
\newblock Meta{R}eg: Towards domain generalization using meta-regularization.
\newblock In S.~Bengio, H.~Wallach, H.~Larochelle, K.~Grauman, N.~Cesa-Bianchi,
  and R.~Garnett, editors, \emph{Advances in Neural Information Processing
  Systems 31}, pages 998--1008. Curran Associates, Inc., 2018.

\bibitem[Bartlett and Mendelson(2002)]{BarMen02}
Peter Bartlett and Shahar Mendelson.
\newblock {R}ademacher and {G}aussian complexities: Risk bounds and structural
  results.
\newblock \emph{Journal of Machine Learning Research}, 3:\penalty0 463--482,
  2002.

\bibitem[Bartlett et~al.(2006)Bartlett, Jordan, and McAuliffe]{bartlett06}
Peter Bartlett, Michael Jordan, and Jon McAuliffe.
\newblock Convexity, classification, and risk bounds.
\newblock \emph{Journal of the American Statistical Association}, 101\penalty0
  (473):\penalty0 138--156, 2006.

\bibitem[Baxter(2000)]{baxter:2000:jair}
Jonathan Baxter.
\newblock A model of inductive bias learning.
\newblock \emph{Journal of Artificial Intelligence Research}, 12:\penalty0
  149--198, 2000.

\bibitem[Ben-David and Urner(2012)]{bendavid12alt}
Shai Ben-David and Ruth Urner.
\newblock On the hardness of domain adaptation and the utility of unlabeled
  target samples.
\newblock In Nader~H. Bshouty, Gilles Stoltz, Nicolas Vayatis, and Thomas
  Zeugmann, editors, \emph{Algorithmic Learning Theory}, pages 139--153, 2012.

\bibitem[Ben-David et~al.(2007)Ben-David, Blitzer, Crammer, and
  Pereira]{bendavid07nips}
Shai Ben-David, John Blitzer, Koby Crammer, and Fernando Pereira.
\newblock Analysis of representations for domain adaptation.
\newblock In B.~Sch\"{o}lkopf, J.~C. Platt, and T.~Hoffman, editors,
  \emph{Advances in Neural Information Processing Systems 19}, pages 137--144.
  2007.

\bibitem[Ben-David et~al.(2010)Ben-David, Blitzer, Crammer, Kulesza, Pereira,
  and Wortman~Vaughan]{bendavid10ml}
Shai Ben-David, John Blitzer, Koby Crammer, Alex Kulesza, Fernando Pereira, and
  Jennifer Wortman~Vaughan.
\newblock A theory of learning from different domains.
\newblock \emph{Machine Learning}, 79:\penalty0 151--175, 2010.

\bibitem[Bickel et~al.(2009)Bickel, Br\"{u}ckner, and Scheffer]{bickel09}
Steffen Bickel, Michael Br\"{u}ckner, and Tobias Scheffer.
\newblock Discriminative learning under covariate shift.
\newblock \emph{Journal of Machine Learning Research}, 10:\penalty0 2137--2155,
  2009.

\bibitem[Blanchard et~al.(2011)Blanchard, Lee, and Scott]{blanchard:11:nips}
Gilles Blanchard, Gyemin Lee, and Clayton Scott.
\newblock Generalizing from several related classification tasks to a new
  unlabeled sample.
\newblock In John Shawe-Taylor, Richard~S. Zemel, Peter~L. Bartlett, Fernando
  Pereira, and Kilian~Q. Weinberger, editors, \emph{Advances in Neural
  Information Processing Systems 24}, pages 2178--2186. 2011.

\bibitem[Blanchard et~al.(2016)Blanchard, Flaska, Handy, Pozzi, and
  Scott]{blanchard16ejs}
Gilles Blanchard, Marek Flaska, Gregory Handy, Sara Pozzi, and Clayton Scott.
\newblock Classification with asymmetric label noise: Consistency and maximal
  denoising.
\newblock \emph{Electronic Journal of Statistics}, 10:\penalty0 2780--2824,
  2016.

\bibitem[Blitzer et~al.(2008)Blitzer, Crammer, Kulesza, Pereira, and
  Wortman]{blitzer08nips}
John Blitzer, Koby Crammer, Alex Kulesza, Fernando Pereira, and Jennifer
  Wortman.
\newblock Learning bounds for domain adaptation.
\newblock In J.~C. Platt, D.~Koller, Y.~Singer, and S.~T. Roweis, editors,
  \emph{Advances in Neural Information Processing Systems 20}, pages 129--136.
  2008.

\bibitem[Cannings et~al.(2018)Cannings, Fan, and Samworth]{cannings18}
Timothy~I. Cannings, Yingying Fan, and Richard~J. Samworth.
\newblock Classification with imperfect training labels.
\newblock Technical Report arXiv:1805.11505, 2018.

\bibitem[Carlucci et~al.(2019)Carlucci, D'Innocente, Bucci, Caputo, and
  Tommasi]{Carlucci2019DomainGB}
Fabio~Maria Carlucci, Antonio D'Innocente, Silvia Bucci, Barbara Caputo, and
  Tatiana Tommasi.
\newblock Domain generalization by solving jigsaw puzzles.
\newblock \emph{2019 IEEE Conference on Computer Vision and Pattern Recognition
  (CVPR)}, pages 2224--2233, 2019.

\bibitem[Caruana(1997)]{caruana:97:ml}
Rich Caruana.
\newblock Multitask learning.
\newblock \emph{Machine Learning}, 28:\penalty0 41--75, 1997.

\bibitem[Chang and Lin(2011)]{chang2011libsvm}
Chih-Chung Chang and Chih-Jen Lin.
\newblock {LIBSVM}: A library for support vector machines.
\newblock \emph{ACM Transactions on Intelligent Systems and Technology},
  2\penalty0 (3):\penalty0 27, 2011.

\bibitem[Christmann and Steinwart(2010)]{CriSte10}
Andreas Christmann and Ingo Steinwart.
\newblock Universal kernels on non-standard input spaces.
\newblock In J.~Lafferty, C.~K.~I. Williams, J.~Shawe-Taylor, R.~Zemel, and
  A.~Culotta, editors, \emph{Advances in Neural Information Processing Systems
  23}, pages 406--414, 2010.

\bibitem[Cortes et~al.(2008)Cortes, Mohri, Riley, and Rostamizadeh]{cortes08}
Corinna Cortes, Mehryar Mohri, Michael Riley, and Afshin Rostamizadeh.
\newblock Sample selection bias correction theory.
\newblock In \emph{Algorithmic Learning Theory}, pages 38--53, 2008.

\bibitem[Cortes et~al.(2015)Cortes, Mohri, and Mu\~{n}oz Medina]{cortes15kdd}
Corinna Cortes, Mehryar Mohri, and Andr{\'e}s Mu\~{n}oz Medina.
\newblock Adaptation algorithm and theory based on generalized discrepancy.
\newblock In \emph{Proceedings of the 21th ACM SIGKDD International Conference
  on Knowledge Discovery and Data Mining}, KDD '15, pages 169--178, 2015.

\bibitem[Daley and Vere-Jones(2003)]{daley2003introductionvol1}
Daryl~J. Daley and David Vere-Jones.
\newblock \emph{An Introduction to the Theory of Point Processes, volume I:
  Elementary Theory and Methods}.
\newblock Springer, 2003.

\bibitem[Daley and Vere-Jones(2008)]{daley2008introductionvol2}
Daryl~J. Daley and David Vere-Jones.
\newblock \emph{An Introduction to the Theory of Point Processes, volume II:
  General Theory and Structure}.
\newblock Springer, 2008.

\bibitem[Denevi et~al.(2018{\natexlab{a}})Denevi, Ciliberto, Stamos, and
  Pontil]{denevi18}
Giulia Denevi, Carlo Ciliberto, Dimitris Stamos, and Massimiliano Pontil.
\newblock Learning to learn around a common mean.
\newblock In S.~Bengio, H.~Wallach, H.~Larochelle, K.~Grauman, N.~Cesa-Bianchi,
  and R.~Garnett, editors, \emph{Advances in Neural Information Processing
  Systems 31}, pages 10169--10179. 2018{\natexlab{a}}.

\bibitem[Denevi et~al.(2018{\natexlab{b}})Denevi, Ciliberto, Stamos, and
  Pontil]{pontil18}
Giulia Denevi, Carlo Ciliberto, Dimitros Stamos, and Massimiliano Pontil.
\newblock Incremental learning-to-learn with statistical guarantees.
\newblock In \emph{Proc. Uncertainty in Artificial Intelligence},
  2018{\natexlab{b}}.

\bibitem[Ding and Fu(2018)]{Ding2018DeepDG}
Zhengming Ding and Yun Fu.
\newblock Deep domain generalization with structured low-rank constraint.
\newblock \emph{IEEE Transactions on Image Processing}, 27:\penalty0 304--313,
  2018.

\bibitem[Dou et~al.(2019)Dou, Coelho~de Castro, Kamnitsas, and
  Glocker]{dou19neurips}
Qi~Dou, Daniel Coelho~de Castro, Konstantinos Kamnitsas, and Ben Glocker.
\newblock Domain generalization via model-agnostic learning of semantic
  features.
\newblock In H.~Wallach, H.~Larochelle, A.~Beygelzimer, F.~d'~Alch\'{e}-Buc,
  E.~Fox, and R.~Garnett, editors, \emph{Advances in Neural Information
  Processing Systems 32}, pages 6450--6461. 2019.

\bibitem[Drineas and Mahoney(2005)]{drineas2005nystrom}
Petros Drineas and Michael~W. Mahoney.
\newblock On the {N}ystr{\"o}m method for approximating a gram matrix for
  improved kernel-based learning.
\newblock \emph{The Journal of Machine Learning Research}, 6:\penalty0
  2153--2175, 2005.

\bibitem[Du~Plessis and Sugiyama(2012)]{plessis12}
Marthinus~Christoffel Du~Plessis and Masashi Sugiyama.
\newblock Semi-supervised learning of class balance under class-prior change by
  distribution matching.
\newblock In J.~Langford and J.~Pineau, editors, \emph{Proc. 29th Int. Conf. on
  Machine Learning}, pages 823--830, 2012.

\bibitem[Evgeniou et~al.(2005)Evgeniou, Michelli, and Pontil]{evgeniou05}
Theodoros Evgeniou, Charles~A. Michelli, and Massimiliano Pontil.
\newblock Learning multiple tasks with kernel methods.
\newblock \emph{Journal of Machine Learning Research}, 6:\penalty0 615--637,
  2005.

\bibitem[Fan et~al.(2008)Fan, Chang, Hsieh, Wang, and Lin]{fan2008liblinear}
Rong-En Fan, Kai-Wei Chang, Cho-Jui Hsieh, Xiang-Rui Wang, and Chih-Jen Lin.
\newblock {LIBLINEAR}: A library for large linear classification.
\newblock \emph{The Journal of Machine Learning Research}, 9:\penalty0
  1871--1874, 2008.

\bibitem[Finn et~al.(2017)Finn, Abbeel, and Levine]{finn17}
Chelsea Finn, Pieter Abbeel, and Sergey Levine.
\newblock Model-agnostic meta-learning for fast adaptation of deep networks.
\newblock In Doina Precup and Yee~Whye Teh, editors, \emph{International
  Conference on Machine Learning}, volume~70 of \emph{Proceedings of Machine
  Learning Research}, pages 1126--1135, 2017.

\bibitem[Gan et~al.(2016)Gan, Yang, and Gong]{Gan_2016_CVPR}
Chuang Gan, Tianbao Yang, and Boqing Gong.
\newblock Learning attributes equals multi-source domain generalization.
\newblock In \emph{The IEEE Conference on Computer Vision and Pattern
  Recognition (CVPR)}, June 2016.

\bibitem[Germain et~al.(2016)Germain, Habrard, Laviolette, and
  Morvant]{germain16icml}
Pascal Germain, Amaury Habrard, Fran{\c{c}}ois Laviolette, and Emilie Morvant.
\newblock A new pac-bayesian perspective on domain adaptation.
\newblock In \emph{Internation Conference on Machine Learning}, volume~48 of
  \emph{{JMLR} Workshop and Conference Proceedings}, pages 859--868, 2016.

\bibitem[Ghifary et~al.(2015)Ghifary, Kleijn, Zhang, and Balduzzi]{ghifary15}
Muhammad Ghifary, W.~Bastiaan Kleijn, Mengjie Zhang, and David Balduzzi.
\newblock Domain generalization for object recognition with multi-task
  autoencoders.
\newblock In \emph{IEEE International Conference on Computer Vision}, page
  2551–2559, 2015.

\bibitem[Ghifary et~al.(2017)Ghifary, Balduzzi, Kleijn, and
  Zhang]{ghifary:2017:pami}
Muhammad Ghifary, David Balduzzi, W.~Bastiaan Kleijn, and Mengjie Zhang.
\newblock Scatter component analysis: A unified framework for domain adaptation
  and domain generalization.
\newblock \emph{IEEE Transactions on Pattern Analysis and Machine
  Intelligence}, 39\penalty0 (7):\penalty0 1411--1430, 2017.

\bibitem[Gong et~al.(2016)Gong, Zhang, Liu, Tao, Glymour, and
  Sch{\"o}lkopf]{gong2016domain}
Mingming Gong, Kun Zhang, Tongliang Liu, Dacheng Tao, Clark Glymour, and
  Bernhard Sch{\"o}lkopf.
\newblock Domain adaptation with conditional transferable components.
\newblock In \emph{International conference on machine learning}, pages
  2839--2848, 2016.

\bibitem[Gretton et~al.(2007{\natexlab{a}})Gretton, Borgwardt, Rasch,
  Sch\"{o}lkopf, and Smola]{Greetal07a}
Arthur Gretton, Karsten Borgwardt, Malte Rasch, Bernhard Sch\"{o}lkopf, and
  Alexander Smola.
\newblock A kernel approach to comparing distributions.
\newblock In R.~Holte and A.~Howe, editors, \emph{22nd AAAI Conference on
  Artificial Intelligence}, pages 1637--1641, 2007{\natexlab{a}}.

\bibitem[Gretton et~al.(2007{\natexlab{b}})Gretton, Borgwardt, Rasch,
  Sch\"{o}lkopf, and Smola]{Greetal07b}
Arthur Gretton, Karsten Borgwardt, Malte Rasch, Bernhard Sch\"{o}lkopf, and
  Alexander Smola.
\newblock A kernel method for the two-sample-problem.
\newblock In B.~Sch\"{o}lkopf, J.~Platt, and T.~Hoffman, editors,
  \emph{Advances in Neural Information Processing Systems 19}, pages 513--520,
  2007{\natexlab{b}}.

\bibitem[Grubinger et~al.(2015)Grubinger, Birlutiu, Sch\"{o}ner,
  Natschl\"{a}ger, and Heskes]{grubinger:2015:iwann}
Thomas Grubinger, Adriana Birlutiu, Holger Sch\"{o}ner, Thomas Natschl\"{a}ger,
  and Tom Heskes.
\newblock Domain generalization based on transfer component analysis.
\newblock In I.~Rojas, G.~Joya, and A.~Catala, editors, \emph{Advances in
  Computational Intelligence, International Work-Conference on Artificial
  Neural Networks}, volume 9094 of \emph{Lecture Notes in Computer Science},
  pages 325--334. Springer International Publishing, 2015.

\bibitem[Hall(1981)]{hall81}
Peter Hall.
\newblock On the non-parametric estimation of mixture proportions.
\newblock \emph{Journal of the Royal Statistical Society}, 43\penalty0
  (2):\penalty0 147--156, 1981.

\bibitem[Hsieh et~al.(2008)Hsieh, Chang, Lin, Keerthi, and
  Sundararajan]{hsieh2008dual}
Cho-Jui Hsieh, Kai-Wei Chang, Chih-Jen Lin, S.~Sathiya Keerthi, and
  S.~Sundararajan.
\newblock A dual coordinate descent method for large-scale linear {SVM}.
\newblock In \emph{International Conference on Machine Learning}, pages
  408--415. ACM, 2008.

\bibitem[Hu et~al.(2019)Hu, Zhang, Chen, and Chan]{hu19uai}
Shoubo Hu, Kun Zhang, Zhitang Chen, and Laiwan Chan.
\newblock Domain generalization via multidomain discriminant analysis.
\newblock In Amir Globerson and Ricardo Silva, editors, \emph{Uncertainty in
  Artificial Intelligence}, 2019.

\bibitem[Huang et~al.(2007)Huang, Smola, Gretton, Borgwardt, and
  Scholkopf]{huang07}
Jiayuan Huang, Alexander~J. Smola, Arthur Gretton, Karsten~M. Borgwardt, and
  Bernhard Scholkopf.
\newblock Correcting sample selection bias by unlabeled data.
\newblock In \emph{Advances in Neural Information Processing Systems}, pages
  601--608, 2007.

\bibitem[Jitkrittum et~al.(2015)Jitkrittum, Gretton, Heess, Eslami,
  Lakshminarayanan, Sejdinovic, and Szab{\'o}]{jitkrittum2015kernel}
Wittawat Jitkrittum, Arthur Gretton, Nicolas Heess, SM~Eslami, Balaji
  Lakshminarayanan, Dino Sejdinovic, and Zolt{\'a}n Szab{\'o}.
\newblock Kernel-based just-in-time learning for passing expectation
  propagation messages.
\newblock In \emph{Proceedings of the Thirty-First Conference on Uncertainty in
  Artificial Intelligence}, pages 405--414. AUAI Press, 2015.

\bibitem[Joachims(1999)]{joachims:99:svmlight}
Thorsten Joachims.
\newblock Making large-scale {SVM} learning practical.
\newblock In B.~Sch{\"o}lkopf, C.~Burges, and A.~Smola, editors, \emph{Advances
  in Kernel Methods - Support Vector Learning}, chapter~11, pages 169--184. MIT
  Press, Cambridge, MA, 1999.

\bibitem[Kallenberg(2002)]{Kal02}
Olav Kallenberg.
\newblock \emph{Foundations of Modern Probability}.
\newblock Springer, 2002.

\bibitem[Kanamori et~al.(2009)Kanamori, Hido, and Sugiyama]{kanamori09}
Takafumi Kanamori, Shohei Hido, and Masashi Sugiyama.
\newblock A least-squares approach to direct importance estimation.
\newblock \emph{Journal of Machine Learning Research}, 10:\penalty0 1391--1445,
  2009.

\bibitem[Khosla et~al.(2012)Khosla, Zhou, Malisiewicz, Efros, and
  Torralba]{khosla12}
Aditya Khosla, Tinghui Zhou, Tomasz Malisiewicz, Alexei~A. Efros, and Antonio
  Torralba.
\newblock Undoing the damage of dataset bias.
\newblock In \emph{12th European Conference on Computer Vision - Volume Part
  I}, page 158–171, 2012.

\bibitem[Kohavi(1995)]{kohavi1995study}
Ron Kohavi.
\newblock A study of cross-validation and bootstrap for accuracy estimation and
  model selection.
\newblock In \emph{International Joint Conference on Artificial Intelligence},
  volume~14, pages 1137--1145, 1995.

\bibitem[Koltchinskii(2001)]{kolt01}
Vladimir Koltchinskii.
\newblock Rademacher penalties and structural risk minimization.
\newblock \emph{{IEEE} Transactions on Information Theory}, 47\penalty0
  (5):\penalty0 1902 -- 1914, 2001.

\bibitem[Latinne et~al.(2001)Latinne, Saerens, and Decaestecker]{saerens01}
Patrice Latinne, Marco Saerens, and Christine Decaestecker.
\newblock Adjusting the outputs of a classifier to new a priori probabilities
  may significantly improve classification accuracy: Evidence from a
  multi-class problem in remote sensing.
\newblock In C.~Sammut and A.~H. Hoffmann, editors, \emph{International
  Conference on Machine Learning}, pages 298--305, 2001.

\bibitem[Le et~al.(2013)Le, Sarl{\'o}s, and Smola]{le2014fastfood}
Quoc Le, Tam{\'a}s Sarl{\'o}s, and Alex Smola.
\newblock Fastfood: approximating kernel expansions in loglinear time.
\newblock In \emph{International Conference on International Conference on
  Machine Learning-Volume 28}, pages III--244, 2013.

\bibitem[Li et~al.(2017)Li, Yang, Song, and Hospedales]{li2017deeper}
Da~Li, Yongxin Yang, Yi-Zhe Song, and Timothy~M Hospedales.
\newblock Deeper, broader and artier domain generalization.
\newblock In \emph{Proceedings of the IEEE International Conference on Computer
  Vision}, pages 5542--5550, 2017.

\bibitem[Li et~al.(2018{\natexlab{a}})Li, Yang, Song, and
  Hospedales]{li2018learning}
Da~Li, Yongxin Yang, Yi-Zhe Song, and Timothy~M Hospedales.
\newblock Learning to generalize: Meta-learning for domain generalization.
\newblock In \emph{AAAI Conference on Artificial Intelligence},
  2018{\natexlab{a}}.

\bibitem[Li et~al.(2018{\natexlab{b}})Li, Pan, Wang, and Kot]{li18adversarial}
Haoliang Li, Sinno~Jialin Pan, Shiqi Wang, and Alex~C. Kot.
\newblock Domain generalization with adversarial feature learning.
\newblock In \emph{IEEE Conference on Computer Vision and Pattern Recognition},
  pages 5400--5409, 2018{\natexlab{b}}.

\bibitem[Li et~al.(2018{\natexlab{c}})Li, Gong, Tian, Liu, and Tao]{li18aaai}
Ya~Li, Mingming Gong, Xinmei Tian, Tongliang Liu, and Dacheng Tao.
\newblock Domain generalization via conditional invariant representations.
\newblock In \emph{AAAI Conference on Artificial Intelligence},
  2018{\natexlab{c}}.

\bibitem[Li et~al.(2018{\natexlab{d}})Li, Tian, Gong, Liu, Liu, Zhang, and
  Tao]{li2018deep}
Ya~Li, Xinmei Tian, Mingming Gong, Yajing Liu, Tongliang Liu, Kun Zhang, and
  Dacheng Tao.
\newblock Deep domain generalization via conditional invariant adversarial
  networks.
\newblock In \emph{Proceedings of the European Conference on Computer Vision
  (ECCV)}, pages 624--639, 2018{\natexlab{d}}.

\bibitem[Mansour et~al.(2009{\natexlab{a}})Mansour, Mohri, and
  Rostamizadeh]{mansour09colt}
Yishay Mansour, Mehryar Mohri, and Afshin Rostamizadeh.
\newblock Domain adaptation: Learning bounds and algorithms.
\newblock In \emph{Conference on Learning Theory}, 2009{\natexlab{a}}.

\bibitem[Mansour et~al.(2009{\natexlab{b}})Mansour, Mohri, and
  Rostamizadeh]{mansour2009domain}
Yishay Mansour, Mehryar Mohri, and Afshin Rostamizadeh.
\newblock Domain adaptation with multiple sources.
\newblock In \emph{Advances in Neural Information Processing Systems}, pages
  1041--1048, 2009{\natexlab{b}}.

\bibitem[Maurer(2009)]{maurer:2009:ml}
Andreas Maurer.
\newblock Transfer bounds for linear feature learning.
\newblock \emph{Machine Learning}, 75\penalty0 (3):\penalty0 327--350, 2009.

\bibitem[Maurer et~al.(2013)Maurer, Pontil, and
  Romera-Paredes]{maurer:2013:icml}
Andreas Maurer, Massimiliano Pontil, and Bernardino Romera-Paredes.
\newblock Sparse coding for multitask and transfer learning.
\newblock In Sanjoy Dasgupta and David McAllester, editors, \emph{International
  Conference on Machine Learning}, volume~28 of \emph{Proceedings of Machine
  Learning Research}, pages 343--351, 2013.

\bibitem[Maurer et~al.(2016)Maurer, Pontil, and
  Romera-Paredes]{Maurer:2016:jmlr}
Andreas Maurer, Massimiliano Pontil, and Bernardino Romera-Paredes.
\newblock The benefit of multitask representation learning.
\newblock \emph{Journal of Machine Learning Research}, 17\penalty0
  (1):\penalty0 2853--2884, 2016.

\bibitem[Menon et~al.(2018)Menon, van Rooyen, and Natarajan]{menon18}
Aditya~Krishna Menon, Brendan van Rooyen, and Nagarajan Natarajan.
\newblock Learning from binary labels with instance-dependent noise.
\newblock \emph{Machine Learning}, 107:\penalty0 1561--1595, 2018.

\bibitem[Motiian et~al.(2017)Motiian, Piccirilli, Adjeroh, and
  Doretto]{motiian2017unified}
Saeid Motiian, Marco Piccirilli, Donald~A. Adjeroh, and Gianfranco Doretto.
\newblock Unified deep supervised domain adaptation and generalization.
\newblock In \emph{IEEE International Conference on Computer Vision}, pages
  5715--5725, 2017.

\bibitem[Muandet et~al.(2013)Muandet, Balduzzi, and
  Sch\"{o}lkopf]{muandet:2013:icml}
Krikamol Muandet, David Balduzzi, and Bernhard Sch\"{o}lkopf.
\newblock Domain generalization via invariant feature representation.
\newblock In \emph{International Conference on Machine Learning}, volume~28 of
  \emph{Proceedings of Machine Learning Research}, pages I--10--I--18, 2013.

\bibitem[Natarajan et~al.(2018)Natarajan, Dhillon, Ravikumar, and
  Tewari]{natarajan18jmlr}
Nagarajan Natarajan, Inderjit~S. Dhillon, Pradeep Ravikumar, and Ambuj Tewari.
\newblock Cost-sensitive learning with noisy labels.
\newblock \emph{Journal of Machine Learning Research}, 18\penalty0
  (155):\penalty0 1--33, 2018.
\newblock URL \url{http://jmlr.org/papers/v18/15-226.html}.

\bibitem[Parthasarathy(1967)]{Par67}
Kalyanapuram~Rangachari Parthasarathy.
\newblock \emph{Probability Measures on Metric Spaces}.
\newblock Academic Press, 1967.

\bibitem[Pentina and Lampert(2014)]{lampert:2014:icml}
Anastasia Pentina and Christoph Lampert.
\newblock A {PAC}-{B}ayesian bound for lifelong learning.
\newblock In Eric~P. Xing and Tony Jebara, editors, \emph{International
  Conference on Machine Learning}, volume~32 of \emph{Proceedings of Machine
  Learning Research}, pages 991--999, 2014.

\bibitem[Pinelis and Sakhanenko(1985)]{PinSak85}
Iosif~F. Pinelis and Aleksandr~Ivanovich Sakhanenko.
\newblock Remarks on inequalities for probabilities of large deviations.
\newblock \emph{Theory Probab. Appl.}, 30\penalty0 (1):\penalty0 143--148,
  1985.

\bibitem[Quionero-Candela et~al.(2009)Quionero-Candela, Sugiyama, Schwaighofer,
  and Lawrence]{Quionero-Candela:2009:DSM:1462129}
Joaquin Quionero-Candela, Masashi Sugiyama, Anton Schwaighofer, and Neil
  Lawrence.
\newblock \emph{Dataset Shift in Machine Learning}.
\newblock The MIT Press, 2009.

\bibitem[Rahimi and Recht(2007)]{rahimi2007random}
Ali Rahimi and Ben Recht.
\newblock Random features for large-scale kernel machines.
\newblock In \emph{Advances in Neural Information Processing Systems}, pages
  1177--1184, 2007.

\bibitem[Rudi and Rosasco(2017)]{rudi2017generalization}
Alessandro Rudi and Lorenzo Rosasco.
\newblock Generalization properties of learning with random features.
\newblock In \emph{Advances in Neural Information Processing Systems}, pages
  3215--3225, 2017.

\bibitem[Sanderson and Scott(2014)]{sanderson14}
Tyler Sanderson and Clayton Scott.
\newblock Class proportion estimation with application to multiclass anomaly
  rejection.
\newblock In \emph{Conferencen on Artificial Intelligence and Statistics},
  2014.

\bibitem[Scott(2019)]{scott19alt}
Clayton Scott.
\newblock A generalized {N}eyman-{P}earson criterion for optimal domain
  adaptation.
\newblock In Aur\'elien Garivier and Satyen Kale, editors, \emph{Algorithmic
  Learning Theory}, volume~98 of \emph{Proceedings of Machine Learning
  Research}, pages 738--761, 2019.

\bibitem[Shankar et~al.(2018)Shankar, Piratla, Chakrabarti, Chaudhuri, Jyothi,
  and Sarawagi]{shankar2018generalizing}
Shiv Shankar, Vihari Piratla, Soumen Chakrabarti, Siddhartha Chaudhuri, Preethi
  Jyothi, and Sunita Sarawagi.
\newblock Generalizing across domains via cross-gradient training.
\newblock In \emph{International Conference on Learning Representations}, 2018.
\newblock URL \url{https://openreview.net/forum?id=r1Dx7fbCW}.

\bibitem[Sharma and Cutler(2015)]{sharma2015robust}
Srinagesh Sharma and James~W. Cutler.
\newblock Robust orbit determination and classification: A learning theoretic
  approach.
\newblock \emph{Interplanetary Network Progress Report}, 203:\penalty0 1, 2015.

\bibitem[Sriperumbudur and Szab{\'o}(2015)]{sriperumbudur2015optimal}
Bharath Sriperumbudur and Zolt{\'a}n Szab{\'o}.
\newblock Optimal rates for random {F}ourier features.
\newblock In \emph{Advances in Neural Information Processing Systems}, pages
  1144--1152, 2015.

\bibitem[Sriperumbudur et~al.(2010)Sriperumbudur, Gretton, Fukumizu,
  Sch\"{o}lkopf, and Lanckriet]{Srietal10}
Bharath Sriperumbudur, Arthur Gretton, Kenji Fukumizu, Bernhard Sch\"{o}lkopf,
  and Gert Lanckriet.
\newblock Hilbert space embeddings and metrics on probability measures.
\newblock \emph{Journal of Machine Learning Research}, 11:\penalty0 1517--1561,
  2010.

\bibitem[Steinwart and Christmann(2008)]{steinwart08}
Ingo Steinwart and Andreas Christmann.
\newblock \emph{Support Vector Machines}.
\newblock Springer, 2008.

\bibitem[Storkey(2009)]{storkey09}
Amos~J Storkey.
\newblock When training and test sets are different: characterising learning
  transfer.
\newblock In \emph{Dataset Shift in Machine Learning}, pages 3--28. MIT Press,
  2009.

\bibitem[Sugiyama et~al.(2008)Sugiyama, Suzuki, Nakajima, Kashima, von
  B\"{u}nau, and Kawanabe]{sugiyama08kliep}
Masashi Sugiyama, Taiji Suzuki, Shinichi Nakajima, Hisashi Kashima, Paul von
  B\"{u}nau, and Motoaki Kawanabe.
\newblock Direct importance estimation for covariate shift adaptation.
\newblock \emph{Annals of the Institute of Statistical Mathematics},
  60:\penalty0 699--746, 2008.

\bibitem[Sutherland and Schneider(2015)]{sutherland2015error}
Danica~J. Sutherland and Jeff Schneider.
\newblock On the error of random {F}ourier features.
\newblock In \emph{Uncertainty in Artificial Intelligence}, pages 862--871,
  2015.

\bibitem[Szab{\'o} et~al.(2016)Szab{\'o}, Sriperumbudur, P{\'o}czos, and
  Gretton]{szabo2016learning}
Zolt{\'a}n Szab{\'o}, Bharath~K Sriperumbudur, Barnab{\'a}s P{\'o}czos, and
  Arthur Gretton.
\newblock Learning theory for distribution regression.
\newblock \emph{The Journal of Machine Learning Research}, 17\penalty0
  (1):\penalty0 5272--5311, 2016.

\bibitem[Tasche(2017)]{tasche17jmlr}
Dirk Tasche.
\newblock Fisher consistency for prior probability shift.
\newblock \emph{Journal of Machine Learning Research}, 18:\penalty0 1--32,
  2017.

\bibitem[Thrun(1996)]{thrun:96:nips}
Sebastian Thrun.
\newblock Is learning the n-th thing any easier than learning the first?
\newblock \emph{Advances in Neural Information Processing Systems}, pages
  640--646, 1996.

\bibitem[Titterington(1983)]{titterington83}
D.~Michael Titterington.
\newblock Minimum distance non-parametric estimation of mixture proportions.
\newblock \emph{Journal of the Royal Statistical Society}, 45\penalty0
  (1):\penalty0 37--46, 1983.

\bibitem[Toedling et~al.(2006)Toedling, Rhein, Ratei, Karawajew, and
  Spang]{toedling06bioinfo}
Joern Toedling, Peter Rhein, Richard Ratei, Leonid Karawajew, and Rainer Spang.
\newblock Automated in-silico detection of cell populations in flow cytometry
  readouts and its application to leukemia disease monitoring.
\newblock \emph{BMC Bioinformatics}, 7:\penalty0 282, 2006.

\bibitem[Tsanas et~al.(2010)Tsanas, Little, McSharry, and
  Ramig]{tsanas2010accurate}
Athanasios Tsanas, Max~A. Little, Patrick~E. McSharry, and Lorraine~O. Ramig.
\newblock Accurate telemonitoring of {P}arkinson's disease progression by
  noninvasive speech tests.
\newblock \emph{IEEE Transactions on Biomedical Engineering}, 57\penalty0
  (4):\penalty0 884--893, 2010.

\bibitem[Tsochantaridis et~al.(2005)Tsochantaridis, Joachims, Hofmann, and
  Altun]{tsochantaridis2005large}
Ioannis Tsochantaridis, Thorsten Joachims, Thomas Hofmann, and Yasemin Altun.
\newblock Large margin methods for structured and interdependent output
  variables.
\newblock \emph{Journal of machine learning research}, 6\penalty0
  (Sep):\penalty0 1453--1484, 2005.

\bibitem[van Rooyen and Williamson(2018)]{rooyen18jmlr}
Brendan van Rooyen and Robert~C. Williamson.
\newblock A theory of learning with corrupted labels.
\newblock \emph{Journal of Machine Learning Research}, 18\penalty0
  (228):\penalty0 1--50, 2018.

\bibitem[Wang et~al.(2019)Wang, He, Lipton, and Xing]{wang2019learning}
Haohan Wang, Zexue He, Zachary~C. Lipton, and Eric~P. Xing.
\newblock Learning robust representations by projecting superficial statistics
  out.
\newblock In \emph{International Conference on Learning Representations}, 2019.
\newblock URL \url{https://openreview.net/forum?id=rJEjjoR9K7}.

\bibitem[Wiens(2010)]{wiens10}
Jenna Wiens.
\newblock \emph{Machine Learning for Patient-Adaptive Ectopic Beat
  Classication}.
\newblock Masters Thesis, Department of Electrical Engineering and Computer
  Science, Massachusetts Institute of Technology, 2010.

\bibitem[Williams and Seeger(2001)]{williams2001using}
Christopher Williams and Matthias Seeger.
\newblock Using the {N}ystr{\"o}m method to speed up kernel machines.
\newblock In \emph{Advances in Neural Information Processing Systems}, pages
  682--688, 2001.

\bibitem[Xu et~al.(2014)Xu, Li, Niu, and Xu]{xu2014exploiting}
Zheng Xu, Wen Li, Li~Niu, and Dong Xu.
\newblock Exploiting low-rank structure from latent domains for domain
  generalization.
\newblock In \emph{European Conference on Computer Vision}, pages 628--643.
  Springer, 2014.

\bibitem[Yang et~al.(2013)Yang, Hanneke, and Carbonell]{carbonell:2013:ml}
Liu Yang, Steve Hanneke, and Jamie Carbonell.
\newblock A theory of transfer learning with applications to active learning.
\newblock \emph{Machine Learning}, 90\penalty0 (2):\penalty0 161--189, 2013.

\bibitem[Yang et~al.(2009)Yang, Kim, and Xing]{yang2009heterogeneous}
Xiaolin Yang, Seyoung Kim, and Eric~P. Xing.
\newblock Heterogeneous multitask learning with joint sparsity constraints.
\newblock In \emph{Advances in Neural Information Processing Systems}, pages
  2151--2159, 2009.

\bibitem[Yu and Szepesvari(2012)]{yu12}
Yao-Liang Yu and Csaba Szepesvari.
\newblock Analysis of kernel mean matching under covariate shift.
\newblock In \emph{International Conference on Machine Learning}, pages
  607--614, 2012.

\bibitem[Zadrozny(2004)]{zadrozny04}
Bianca Zadrozny.
\newblock Learning and evaluating classifiers under sample selection bias.
\newblock In \emph{International Conference on Machine Learning}, 2004.

\bibitem[Zhang et~al.(2013)Zhang, Sch{\"o}lkopf, Muandet, and
  Wang]{zhang2013domain}
Kun Zhang, Bernhard Sch{\"o}lkopf, Krikamol Muandet, and Zhikun Wang.
\newblock Domain adaptation under target and conditional shift.
\newblock In \emph{International Conference on Machine Learning}, pages
  819--827, 2013.

\bibitem[Zhang et~al.(2015)Zhang, Gong, and Scholkopf]{zhang2015multi}
Kun Zhang, Mingming Gong, and Bernhard Scholkopf.
\newblock Multi-source domain adaptation: A causal view.
\newblock In \emph{AAAI Conference on Artificial Intelligence}, pages
  3150--3157. AAAI Press, 2015.

\end{thebibliography}

\end{document}